\documentclass[10pt,journal,compsoc]{IEEEtran}
\usepackage[T1]{fontenc}
\usepackage{mathtools}
\usepackage{balance}
\usepackage{amssymb}
\usepackage{lineno}
\usepackage{times}
\usepackage{soul}
\usepackage{url}
\usepackage[hidelinks]{hyperref}
\usepackage[utf8]{inputenc}
\usepackage{graphicx, wrapfig, subcaption, setspace, booktabs}
\usepackage{amsmath}
\usepackage{booktabs}
\usepackage{soul}
\usepackage{array}
\usepackage{microtype}
\usepackage{amsfonts}
\usepackage{amssymb}
\usepackage{xcolor}
\usepackage{enumitem}
\usepackage{algorithm,algpseudocode}
\usepackage{amsthm}
\usepackage[symbol]{footmisc}
\usepackage[normalem]{ulem}
\usepackage{amsmath,mathtools,amssymb}

\newtheorem{proposition}{Proposition}
\newtheorem{theorem}{Theorem}
\theoremstyle{definition}
\newtheorem{definition}{Definition}
\newtheorem{lemma}{Lemma}
\newtheorem*{assumption*}{\assumptionnumber}

\providecommand{\assumptionnumber}{}
\makeatletter

\newenvironment{assumption}[1]
 {%
  \renewcommand{\assumptionnumber}{Assumption #1}%
  \begin{assumption*}%
  \protected@edef\@currentlabel{#1}%
 }
 {%
  \end{assumption*}
 }

\DeclarePairedDelimiter{\norm}{\lVert}{\rVert}
\DeclarePairedDelimiter{\abs}{\lvert}{\rvert}%
\DeclarePairedDelimiterX{\inp}[2]{\langle}{\rangle}{#1, #2}

\makeatletter
\let\oldabs\abs
\def\abs{\@ifstar{\oldabs}{\oldabs*}}

\ifCLASSOPTIONcompsoc
  \usepackage[nocompress]{cite}
\else
  \usepackage{cite}
\fi

\ifCLASSINFOpdf
\else
\fi

\begin{document}

\title{Rethinking Graph Auto-Encoder Models\\ for Attributed Graph Clustering}

\author{Nairouz~Mrabah,~
        Mohamed~Bouguessa,~
        Mohamed~Fawzi~Touati,~
        Riadh~Ksantini
\IEEEcompsocitemizethanks{\IEEEcompsocthanksitem N. Mrabah, M. Bouguessa, M. F. Touati are with the Department
of Computer Science, University of Quebec at Montreal, Montreal, QC, Canada.

\IEEEcompsocthanksitem R. Ksantini is with the Department of Computer Science, College of IT, University of Bahrain, Kingdom of Bahrain.\protect\\ 
}
}

\markboth{Manuscript under review}%
{Shell \MakeLowercase{\textit{et al.}}: Bare Demo of IEEEtran.cls for Computer Society Journals}

\IEEEtitleabstractindextext{
\begin{abstract}
Most recent graph clustering methods have resorted to Graph Auto-Encoders (GAEs) to perform joint clustering and embedding learning. However, two critical issues have been overlooked. First, the accumulative error, inflicted by learning with noisy clustering assignments, degrades the effectiveness and robustness of the clustering model. This problem is called Feature Randomness. Second, reconstructing the adjacency matrix sets the model to learn irrelevant similarities for the clustering task. This problem is called Feature Drift. Furthermore, the theoretical relation between the aforementioned problems has not yet been investigated. We study these issues from two aspects: (1) there is a trade-off between Feature Randomness and Feature Drift when clustering and reconstruction are performed at the same level, and (2) the problem of Feature Drift is more pronounced for GAE models, compared with vanilla auto-encoder models, due to the graph convolutional operation and the graph decoding design. Motivated by these findings, we reformulate the GAE-based clustering methodology. Our solution is two-fold. First, we propose a sampling operator $\Xi$ that triggers a protection mechanism against the noisy clustering assignments. Second, we propose an operator $\Upsilon$ that triggers a correction mechanism against Feature Drift by gradually transforming the reconstructed graph into a clustering-oriented one. As principal advantages, our solution grants a considerable improvement in clustering effectiveness and can be easily tailored to existing GAE models.  
\end{abstract}

}

\maketitle

\IEEEdisplaynontitleabstractindextext

\IEEEpeerreviewmaketitle

\IEEEraisesectionheading{\section{Introduction}\label{S:1}}

Most recent attributed graph clustering methods leverage graph embedding \cite{paper91}. This strategy consists of projecting the graph structure and the node content in a low-dimensional compact space to harness the complementary modalities of attributed graphs. Graph embedding usually achieves exploitable representations for the clustering task. A significant part of the graph embedding literature revolves around edge modeling \cite{paper30}, matrix factorization \cite{paper115}, and random walks \cite{paper12}. Yet, these methods fall short of the expressive power of deep learning.

The last years witnessed the emergence of a promising graph embedding strategy, referred to as Graph Neural Networks (GNNs) \cite{paper92, paper114}. GNNs extend the deep learning framework to graph-structured data. Among the prominent categories of GNNs, we find Graph Convolutional Networks (GCNs) \cite{paper9}, which generalize the convolution operation to graph data. Specifically, the intuition of the graph convolutional operation is to exploit the graph structure in smoothing the content features of each node over its neighborhood. Motivated by GCNs, Graph Auto-Encoders (GAEs) \cite{paper11} and Variational Graph Auto-Encoders (VGAEs) \cite{paper11} have shown notable achievements in several attributed graph clustering applications \cite{paper25, paper26, paper27}. Typical GAE-based clustering methods project the input data in a low dimensional space using graph convolutional layers and then reconstruct the adjacency matrix. Minimizing the reconstruction objective for the clustering task rules out the situation where the encoder is \emph{only} trained based on noisy clustering assignments. The accumulated error makes the trained model capture non-representative features \cite{paper17}, which in turn corrupt the latent structure of the data. In our analysis, we adopt the terminology of Feature Randomness (FR) from our previous work \cite{paper14} for investigating this problem in the context of GAEs.

As mentioned before, adding the decoder component is key to optimizing the reconstruction objective, which is a handy way to lower FR's effect. However, the nature of the reconstructed graph is generally problematic to the clustering task. First, real-world graphs carry noisy and clustering-irrelevant links that can mislead the model into grouping together nodes from different clusters. This aspect can cause an under-segmentation problem. Second, it is also common for real-world graphs to come in a highly sparse structure. As a result, poor connectivity within the same cluster gives rise to an over-segmentation problem. Besides, the controversial relationship between clustering and reconstruction makes it hard to identify a static balance between them, during the training process. This problem, which is referred to as Feature Drift (FD) in our previous work \cite{paper14}, remains unexplored for GNNs.

To address the aforementioned issues, we reformulate the GAE-based clustering methodology from an FR and FD perspective. We start by organizing the existing approaches into two groups, and we provide abstract formulations for each one. Next, we leverage the abstract description to examine the limitations of existing methods. Then, we provide formal characterizations to problems associated with the analyzed formulations on the authority of recent insights. After that, we propose a new conceptual design, which can mitigate the impact of FR and FD.

To put our conceptual design into action, we propose two operators, which can be easily integrated into GAE-based clustering methods. Possible options for addressing FR are: (1) operationalizing a correction mechanism that can reverse the randomness effect, (2) supplying the model with a protection mechanism that can exclude the sources of randomness as much as possible. Recently, the authors of \cite{paper64, paper65} have observed that pretraining a network with random labels then fine-tuning with clean ones leads to considerably lower test accuracy, compared with a network trained with clean labels from scratch. From this standpoint, we advocate accounting for FR using a protection strategy. Specifically, we design a sampling operator, prioritizing correctness by considering the difference between the first high-confidence and second high-confidence clustering assignment scores. 

Additionally, we conceive a second operator that can control the effect of FD. Our design capitalizes on converting a general-purpose objective function into a task-specific one. Unlike previous GAE-based approaches, which optimize static objective functions during the whole clustering process, we \emph{gradually} eliminate the graph reconstruction cost in favor of a clustering-oriented graph construction objective. Furthermore, our second operator contributes to preventing the over-segmentation and under-segmentation problems. More specifically, we gradually update the self-supervision graph by adding clustering-friendly edges and dropping clustering-irrelevant links. 

The algorithmic intuitions behind our conceptual design and operators are supported by theoretical and empirical results. Theoretically, we demonstrate the existence of a trade-off between FR and FD for GAE-based clustering. Under mild assumptions, we prove that the graph convolutional operation and performing clustering and reconstruction at the same level aggravate the FD problem. Experimentally, we show that our operators can significantly improve the clustering effectiveness of existing GAE models without causing run-time overheads. Moreover, we show that our operators can mitigate the impact of FR and FD, and we provide empirical evidence that the improvement is imputed to the capacity of our operators in handling the trade-off between FR and FD. 
The significance of this work can be summarized as follows:

\begin{itemize}
  \item Analysis: We organize GAE-based clustering approaches into two groups, and we provide abstract formulations for each one. Accordingly, we analyze and formalize the problems associated with the examined formulations. Then, we present a new conceptual design that can favorably control the trade-off between FR and FD. From a theoretical standpoint, we prove the existence of this trade-off, and we study two important aspects that differentiate GAE models from vanilla auto-encoder methods. Specifically, we investigate the impact of performing clustering and reconstruction at different layers on FR and FD. Moreover, we inspect the influence of the graph convolutional operation on FD.
  \item Methods: First, we propose a sampling operator $\Xi$ that triggers a protection mechanism against FR. Second, we propose an operator $\Upsilon$ that triggers a correction mechanism against FD by gradually transforming the reconstructed graph into a clustering-oriented one.
  \item Experiments: We conduct extensive experiments to investigate the behavior and profit from using our operators. Our empirical results provide strong evidence that the proposed operators improve the clustering performance of GAE models in effectiveness, by mitigating the effect of FR and FD.
\end{itemize}

\section{A new vision for GAE-based clustering}

This section advocates a new vision for building GAE-based clustering models beyond the classical perception of designing better clustering objectives. We begin by describing existing GAE methods, which we organize into two groups. While the first group contains models that separate clustering from embedding learning, the second group only considers the methods that perform joint clustering and embedding learning. For each group, we devise abstract formulations, and we study their associated limitations. Finally, we propose a new conceptual design to mitigate the examined problems. We consider our work to be the first initiative to analyze GAE-based clustering models from FR and FD perspectives.

We consider a non-directed attributed graph $\mathcal{G} = (\mathcal{V}, \mathcal{E}, X)$, where $\mathcal{V} = \left \{v_{1}, v_{2}, ..., v_{N} \right \}$ is a set of nodes with $\left | \mathcal{V} \right | = N$, and $N$ is the number of nodes. $\mathcal{E} = \left \{ e_{ij} \right \}$ represents the set of edges. The topological structure of the graph $\mathcal{G}$ is denoted by the adjacency matrix $A = \left ( a_{ij} \right ) \in \mathbb{R}^{N \times N}$, where $a_{ij} = 1$ if $(v_{i}, v_{j}) \in \mathcal{E}$ and $a_{ij} = 0$ otherwise. $X = \left \{ x_{1}, ...,  x_{N} \right \}$ represents the matrix of features, where $x_{i} \in \mathbb{R}^{J}$ is the feature vector associated with the node $v_{i}$, and $J$ is the dimensionality of the input space. We consider that the graph $\mathcal{G}$ can be clustered into $K$ clusters $\left \{ C_{k}^{clus} \right \}_{k=1}^{K}$.

Our study investigates the auto-encoding architecture for attributed graph clustering. Consequently, two functions should be specified. The first one is a non-linear encoder, which takes as inputs $X$ and $A$, and outputs low-dimensional latent representations denoted by the matrix $Z \in \mathbb{R}^{N \times d}$. $d$ denotes the dimension of the latent space. $\theta$ stands for the set of learnable weights. The second function is a decoder, which outputs a matrix $\hat{A} = \text{sigmoid}(ZZ^{T})$; $\hat{A}$ measures the pairwise similarities between the latent codes.

The idea of self-supervision involves solving a \textit{pretext task} that requires a high-level understanding of the data. Specifically, the reconstruction loss is among the standard self-supervision methods for pretraining GAE models. It is generally expressed as a binary cross-entropy $L_{bce}(\hat{A}(Z(\theta)), A^{self})$, where $A^{self}$ is a self-supervision graph, and it is set equal to $A$, and the order in $\hat{A}(Z(\theta))$ describes dependencies between $\hat{A}$, $Z$, and $\theta$. Let $\mathbb{A}_{C}$ be a clustering algorithm and $P\in \mathbb{R}^{N \times K}$ be the clustering assignment matrix obtained by applying $\mathbb{A}_{C}$ to the embedded representations. $L_{clus}(P(Z(\theta)))$ is the clustering loss associated with algorithm $\mathbb{A}_{C}$. Without a pretraining stage, the clustering algorithm would be applied to random latent representations. 

As mentioned at the beginning of this section, we organize existing approaches into two groups. For the first group, clustering is performed separately from embedding learning. Thus, we express the formulation associated with models from the first group as:

\begin{equation} 
 \begin{aligned}
    P^{*} = \text{arg}\min_{_{P}}L_{clus}(P(Z(\theta))),
  \end{aligned}
\label{eq:first_category}
\end{equation}

\noindent where $\theta$ is initialized by the pretraining task, and $P^{*}$ is a solution to Equation (\ref{eq:first_category}). Examples from the first group include MGAE (Marginalized Graph Auto-Encoder) \cite{paper80}, which improves the clustering performance by increasing robustness to small input disturbances. From the same group, ARGAE (Adversarially Regularized Graph Auto-Encoder) \cite{paper25} leverages an adversarial regularization technique that enforces the embeddings to match a prior distribution using a discriminator network. Nevertheless, methods from the first group, such as MGAE and ARGAE, lack the capacity to learn clustering-oriented features. 

In another perspective, Ansuini et al. \cite{paper47} have shown that the embedded representations of a deep network lie on \emph{highly} curved manifolds. This aspect implies that the Euclidean geometry is not suitable to assess the embedded similarities after the pretraining phase. To alleviate this problem, the second group of GAE-based clustering methods achieves \emph{joint} clustering and embedded learning. In this regard, we reformulate Equation (\ref{eq:first_category}) in a way that enforces the embedded representations to follow the clustering assumptions based on euclidean geometry. To ensure this quality, the formulation of the second group is articulated as:

\begin{equation} 
  \begin{aligned}
   \theta^{*}, \, P^{*} = \text{arg}\min_{_{\theta , P}}L_{clus}(P(Z(\theta))),
  \end{aligned}
\label{eq:second_category}
\end{equation}

\noindent where $\theta ^{*}$ and $P^{*}$ are solutions to Equation (\ref{eq:second_category}). Typical clustering losses aim at decreasing the intra-cluster variances and increasing the inter-cluster variances. By optimizing $\theta$, the embedded points move in a way that establishes a clustering-oriented distribution. Therefore, the choice of the clustering cost becomes less important. However, the formulation of Equation (\ref{eq:second_category}) is still problematic because the embedded points can move in a way that violates their semantic categories, while still decreasing the embedded clustering penalty. 

Let $Q$ be the matrix of true hard-clustering assignments. A supervised deep clustering problem can be described by Equation (\ref{eq:second_category_sup}). Compared with Equation (\ref{eq:second_category}), the supervised objective pushes the latent codes to be clustering-friendly according to the true clustering assignment matrix $Q$. Let $\mathbb{A}_{H}$ be the Hungarian algorithm \cite{paper77}, which finds the best linear mapping from a true clustering assignment matrix $Q$ to a predicted clustering assignment matrix $P$. The algorithm $\mathbb{A}_{H}$ outputs a matrix $Q' = \mathbb{A}_{H}(Q, P)$. By analogy with pseudo-supervision, $y(P) = \text{arg}\max_{_{j \in  \left \{ 0,...,K \right \}}}(P_{:,j})$ can be considered as pseudo-labels for solving Equation (\ref{eq:second_category}), and $y(Q) = \text{arg}\max_{_{j \in  \left \{ 0,...,K \right \}}}(Q_{:,j})$ can be considered as ground-truth labels for solving Equation (\ref{eq:second_category_sup}). The ultimate goal of deep clustering is to formulate an optimization problem, where a solution for the clustering assignment matrix is $P^{*}$, such that $y(P^{*})=y(Q')$.

\begin{equation} 
  \begin{aligned}
   \theta^{*} = \text{arg}\min_{_{\theta}}L_{clus}(Q(Z(\theta))).
  \end{aligned}
\label{eq:second_category_sup}
\end{equation}

Under the extreme condition of entirely random labels, Zhang et al. \cite{paper79} have shown that an over-parameterized neural network can perfectly fit the training set. This finding inspired the scientific community to investigate the difference between ``training with true labels" and ``training with random labels". Keskar et al. \cite{paper50} have proposed a metric for measuring the sharpness of a minimizer to assess generalization. In \cite{paper47, paper48}, the authors have investigated the intrinsic dimensionality of the embedded representations to understand the impact of random labels. In our previous work \cite{paper14}, we have proposed to measure the effect of randomness by computing the cosine similarity between the gradient of the supervised loss and the gradient of the pseudo-supervised loss. 
However, the impact of random labels on Graph Neural Networks for graph datasets remains unexplored. As previously explained, embedded graph clustering can be considered a pseudo-supervised task. Thus, we can exploit our previously proposed metric \cite{paper14} to assess the impact of FR as described by:

\begin{equation} 
    \Lambda_{FR} = cos \bigg( \frac{\partial L_{clus}(P(Z(\theta)))}{\partial \theta}, \frac{\partial L_{clus}(Q'(Z(\theta)))}{\partial \theta} \bigg).\\
    \label{eq:delta_FR}
\end{equation}

$\Lambda_{FR}$ lies within the range $\left [ -1, 1 \right ]$. Higher values are associated with less FR. Possible strategies for countering random projections are: (1) performing pseudo-supervision based on self-paced training, and (2) pretraining by self-supervision (pretext task), and finetuning by combining pseudo-supervision (main task) and self-supervision (pretext task). An example of the first strategy is AGE \cite{paper82}, which constructs a pseudo-supervised graph by linking high similarity pairs and disconnecting the low similarity ones. However, two limitations are associated with the first strategy. First, it does not involve pretraining using a pretext task, and the pseudo-labels are initially constructed from the input data. Hence, the first strategy is limited to datasets, where the node features have low-semantic similarities that can be extracted without neural networks. Second, the first strategy does not combine pseudo-supervision and self-supervision during the clustering phase. You et al. \cite{paper113} have shown that combining the main task with a self-supervised pretext brings more generalizability and robustness to GCNs. For the second strategy, adjacency reconstruction constitutes the standard self-supervised technique for GAE models. In this work, we focus on the relation between pseudo-supervision (i.e., main task: clustering) and self-supervision (i.e., pretext task: reconstruction), which is governed by FR and FD. Accordingly, we reformulate Equation (\ref{eq:second_category_sup}) to take into consideration the reconstruction loss:

\begin{equation} 
  \begin{aligned}
   \theta^{*}, \, P^{*} = \text{arg}\min_{_{\theta , P}} L_{clus}(P(Z(\theta))) + \gamma L_{bce}(\hat{A}(Z(\theta)), \,A),
  \end{aligned}
\label{eq:third_category}
\end{equation}

where $\theta^{*}$ and $P^{*}$ are solutions to Equation (\ref{eq:third_category}). $\gamma$ is a balancing hyper-parameter that controls the trade-off between clustering and reconstruction. Examples from this category include DAEGC (Deep Attentional Embedded Graph Clustering) \cite{paper27}, which employs an attention mechanism to adjust the influence of the neighboring nodes. Another example is GMM-VGAE (Variational Graph Auto-Encoder with Gaussian Mixture Models) \cite{paper26}, which harnesses Gaussian Mixture Models to capture variances between the different clusters. However, the strong competition between embedded clustering and reconstruction causes FD. On the one hand, clustering aims at decreasing intra-cluster variances and increasing inter-cluster variances. On the other hand, the reconstruction objective pushes the latent representations to maintain all variances (i.e., intra-cluster and inter-cluster variances). The features learned by embedded clustering can be destroyed by the reconstruction cost, which captures clustering-irrelevant similarities. 

FD is an artificially-created problem to counter random projections. Thus, it can be completely solved by excluding the self-supervised loss from the optimized objective. However, abrupt elimination of the reconstruction loss aggravates FR. Among two existing methods for alleviating FD is ADEC (Adversarial Deep Embedded Clustering) \cite{paper14}. Instead of minimizing vanilla reconstruction, ADEC optimizes an adversarially constrained reconstruction. More specifically, the vanilla reconstruction is circumvented by blocking the direct connection between the encoder and the decoder, using an additional discriminator network. Nevertheless, ADEC inherits the stability limitations of adversarial training such as mode collapse \cite{paper87}, failure to converge \cite{paper89}, and memorization. In another work, DynAE (Dynamic Auto-Encoder) \cite{paper13} leverages the decoder to gradually construct images of the latent centers, instead of reconstructing the input images. However, DynAE was designed to generate euclidean representations (i.e., images corresponding to the embedded centroids) and can not generate graph-structured data. Added to that, DynAE is considered an improved version of DeepCluster \cite{paper88}. Both models (i.e., DynAE and DeepCluster) perform hard clustering using K-means and do not consider covariances of the embedded clusters. Enforcing pseudo-labels obtained by a hard clustering algorithm may destroy relevant similarities and hence give rise to FR. Last but not least, none of these methods can take both topological structure and content information as input signals.

To overcome the limitations of previous methods, we propose a new conceptual design. Our solution fixes the deficiency of existing GAE models from the perspective of FR and FD. In Figure \ref{fig:fourth_category}, we illustrate the generic framework of this conceptual design. More precisely, our formulation depends on two operators and does not require adversarial training. First, we develop a sampling operator $\Xi$ to gradually spot the nodes with reliable clustering assignments, denoted by the set $\Omega$. We exploit the reliable nodes for optimizing the embedded clustering objective. Second, we propose a graph-specific operator $\Upsilon$ to gradually transform the general-purpose self-supervisory signal $A$ into a clustering-oriented self-supervisory signal $A_{clus}^{self}$. The formulation of our conceptual design is expressed by:

\begin{equation} 
\begin{split}
\theta^{*}, \, P^{*} &= \text{arg}\min_{_{\theta, \, P}} L_{clus}(P(\Xi(Z(\theta)))) \\
& + \; \gamma L_{bce}(\hat{A}(Z(\theta)), \, \Upsilon(A, P(\Xi(Z(\theta))), \Omega)).
\end{split}
\label{eq:fourth_category}
\end{equation}

\begin{figure}
  \centering
  \includegraphics[width=0.55\textwidth, height=4.7cm]{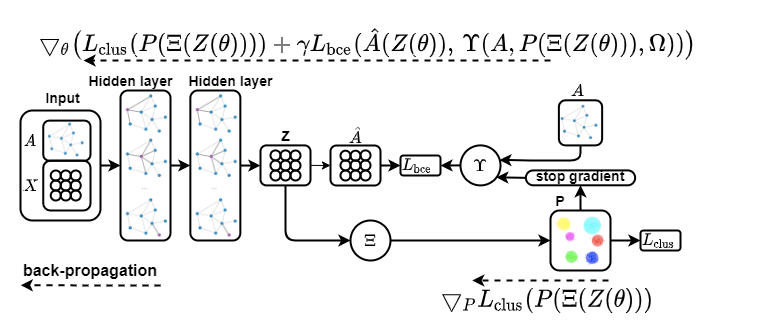}
  \caption{The proposed conceptual design for GAE-based clustering.}
  \label{fig:fourth_category}
  \vspace{-2.5mm}
\end{figure}

To estimate the impact of FD, \cite{paper14} measures the cosine similarity between the gradient of the self-supervised loss and the gradient of the clustering loss. This metric was only tested when clustering and reconstruction losses are computed based on Euclidean distances. However, in the graph case, we found that this metric does not reveal any interpretive pattern, probably because the structure of the output signal is non-euclidean. Thus, we propose a new metric that measures the cosine similarity between two gradients of the same loss but with different configurations, namely, the gradient of the self-supervised loss $L_{bce}(\hat{A}(Z(\theta)), \, \Upsilon(A, P(\Xi(Z(\theta))), \Omega))$, and the gradient of its supervised version $L_{bce}(\hat{A}(Z(\theta)), \, \Upsilon(A, Q'(Z(\theta)), \mathcal{V}))$. Our new metric is described by:

\begin{equation} 
    \begin{split}
    \Lambda_{FD} = cos \bigg( \frac{\partial L_{bce}(\hat{A}(Z(\theta)), \, \Upsilon(A, P(\Xi(Z(\theta))), \Omega))}{\partial \theta}, \\ \frac{\partial L_{bce}(\hat{A}(Z(\theta)), \, \Upsilon(A, Q'(Z(\theta)), \mathcal{V}))}{\partial \theta} \bigg ).
    \end{split}
   \label{eq:delta_FD}
\end{equation}

$\Lambda_{FD}$ lies in the range $\left [ -1, 1 \right ]$. Higher values are associated with less FD. $\Upsilon(A, P(\Xi(Z(\theta))), \Omega)$ makes a single-step modification to the input graph $A$. It only affects the sub-graph defined by the reliable nodes $\Omega$. As opposed to that, $\Upsilon(A, Q'(Z(\theta)), \mathcal{V})$ outputs the clustering-oriented graph that we want to obtain at the end of the training process. It is generated by transforming the whole graph using the supervisory signal. A small discrepancy between both graphs, in terms of gradient direction implies that the two signals have the same impact at the level of gradient computation. Thus, $\Lambda_{FD}$ assesses to what extent the generated self-supervision graph is clustering-oriented.

\section{Theoretical analysis}
Our conceptual design aims at reducing the impact of FD without causing excessive FR. An intuitive explanation of the trade-off between FR and FD is provided in the previous section. In this section, we discuss the problems of FR and FD for GAE models from a theoretical standpoint. Our formal analysis includes three points: (1) proving the existence of a trade-off between FR and FD for GAE models, (2) understanding the impact of performing clustering and reconstruction at different layers on FR and FD, and (3) understanding the impact of the graph convolutional operation, which is performed by all encoding layers, on FD. All mathematical proofs and derivations are provided in the Appendices. 

We start our theoretical analysis by showing that it is possible to write the loss function of a typical GAE model in a way that explicitly demonstrates the trade-off between FR and FD. More specifically, we found that solving a typical GAE-based clustering problem is equivalent to smoothing the embedded representations over a linear combination between the input graph and a clustering graph, using a graph Laplacian regularization loss \cite{paper96, paper97, paper98}. The two aforementioned graphs are different in nature and suffer from different problems. While the clustering graph has several random edges, the input graph is sparse and comes with clustering-irrelevant links. 

Virtually, the problem of clustering Euclidean data using an auto-encoder model appears to be similar to clustering graphs using a GAE model. However, there are two important differences between the two approaches. First, vanilla auto-encoders are not designed to deal with graph-structured data, whereas GAE models can capitalize the structural information thanks to the graph convolutional operation. Therefore, it is important to investigate the impact of this operation on FR and FD. Second, vanilla auto-encoder models have symmetric decoders. As opposed to that, the decoder of a GAE model is nothing more than the sigmoid of an inner product. We analyze the impact of these differences (i.e., graph decoding design and graph convolutional operation) on FR and FD from a theoretical standpoint. Under mild assumptions, our formal analysis demonstrates that the problem of FD is more pronounced for GAE models compared with the FD problem for vanilla auto-encoder models due to the graph convolutional operation and the graph decoding design. Furthermore, we provide sufficient conditions for comparing a typical GAE model against a GAE with a multi-layer decoder, in terms of FR and FD. 

\subsection{Trade-off between FR and FD}

Given a GAE-based clustering model and an attributed graph $\mathcal{G}$, we consider that the nodes of $\mathcal{G}$ are associated with $K$ ground-truth labels defining $K$ real clusters $\left \{ C_{k}^{sup} \right \}_{k=1}^{K}$, and that the embedded representations $Z$ can be clustered into $K$ clusters $\left \{ C_{k}^{clus} \right \}_{k=1}^{K}$ according to an algorithm $\mathbb{A}_{C}$. Let $L_{\mathcal{C}}$ be a generic loss, which takes as input an adjacency matrix $A' = (a'_{ij}) \in \mathbb{R}^{N \times N}$ and a feature matrix $Z' \in \mathbb{R}^{N \times d}$, and can be written in the form:

$$L_{\mathcal{C}}(Z', A') = \frac{1}{2}  \sum_{1 \leqslant i,j \leqslant N} a'_{ij} \: \norm{z'_{i} - z'_{j}}_{2}^{2}.$$

We define three graphs based on their adjacency matrices: a self-supervision graph $A^{self} = (a_{ij}^{self}) \in \mathbb{R}^{N \times N}$, a clustering graph $A^{clus} = (a_{ij}^{clus}) \in \mathbb{R}^{N \times N}$, and a supervision graph $A^{sup} = (a_{ij}^{sup}) \in \mathbb{R}^{N \times N}$. For the reconstruction loss, the self-supervision signal $A^{self}$ is equal to the input graph $A$. The clustering and supervision graphs are expressed as follows:

$$a_{ij}^{clus} = \left\{\begin{matrix}
\frac{1}{\left | C_{k}^{clus} \right |} \:\: \text{if} \:\: \exists\: k \text{ such that } i,j \in C_{k}^{clus}  \\ 
0 \:\:\:\:\:\:\:\:\:\:\:\:\:\:\:\:\:\:\:\:\:\:\:\:\:\:\:\: \text{otherwise,}
\end{matrix}\right.$$

$$a_{ij}^{sup} = \left\{\begin{matrix}
\frac{1}{\left | C_{k}^{sup} \right |} \:\: \text{if} \:\: \exists\: k \text{ such that } i,j \in C_{k}^{sup}  \\ 
0 \:\:\:\:\:\:\:\:\:\:\:\:\:\:\:\:\:\:\:\:\:\:\:\:\:\:\:\: \text{otherwise.}
\end{matrix}\right.$$

\begin{proposition} \label{proposition_1}
The reconstruction loss for a GAE model can be expressed as:
$$ L_{bce}(\hat{A}(Z(\theta)), A^{self}) = L_{\mathcal{C}}(Z(\theta), A^{self}) + L_{\mathcal{R}}(Z(\theta), A^{self}), $$
\begin{equation*}
\begin{split}
L_{\mathcal{R}}(Z(\theta), A^{self}) &=  \sum_{i,j}  \Big(log(1 + exp(z_{i}^{T} z_{j})) \\
&- \frac{1}{2}  a_{ij}^{self} (\norm{z_{i}}_{2}^{2} + \norm{z_{j}}_{2}^{2})\Big).
\end{split}
\end{equation*}
\end{proposition}

In Proposition \ref{proposition_1}, we write the reconstruction loss of a GAE model in the form of a linear combination between a graph Laplacian regularization term $L_{\mathcal{C}}(Z(\theta), A^{self})$ and another loss $L_{\mathcal{R}}(Z(\theta), A^{self})$. A trivial solution to minimize $L_{\mathcal{C}}(Z(\theta), A^{self})$ consists of mapping the features of all nodes to the same latent code. State-of-the-art self-supervised methods rely on negative pairs \cite{paper100, paper101} or a cross-model supplementary loss function \cite{paper102} to avoid degenerate solutions. In our case, the trivial solutions are ruled out by the function $L_{\mathcal{R}}(Z(\theta), A^{self})$. More precisely, minimizing $log(1 + exp(z_{i}^{T} z_{j}))$ implies an increase in the angle between the two vectors $z_{i}$ and $z_{j}$ and/or a decrease in their norms if their angle is lower than $90^{\circ}$ or an increase in their norms if their angle is greater than $90^{\circ}$. However, minimizing the second part of $L_{\mathcal{R}}$ (i.e., $- \frac{1}{2} a_{ij}^{self} \big(\norm{z_{i}}_{2}^{2} + \norm{z_{j}}_{2}^{2} \big)$) increases the norm of $z_{i}$ and $z_{j}$ when there is a link between the two nodes $i$ and $j$. Hence, we can conclude that $L_{\mathcal{R}}$ increases the angle between each couple of vectors $z_{i}$ and $z_{j}$ if there is a link between them. Otherwise, decreasing the norm of both vectors might be sufficient. 

\begin{proposition} \label{proposition_2}
The k-means clustering loss applied to the embedded representations can be expressed as:
\begin{equation*} 
\begin{split}
L_{clus}(Z(\theta)) = L_{\mathcal{C}}(Z(\theta), A^{clus}).
\end{split}
\end{equation*}
\end{proposition}

In Proposition \ref{proposition_2}, we write the embedded k-means loss in the form of a graph Laplacian regularization loss $L_{\mathcal{C}}(Z(\theta), A^{clus})$. As we can see, the graph required for embedded k-means is different from the graph required for the reconstruction loss. Furthermore, training the encoder to minimize embedded k-means without a reconstruction loss can easily lead to degenerate solutions.

\begin{theorem}\label{theorem_1}
The linear combination between reconstruction and embedded k-means for a GAE model can be expressed as:
\begin{equation*} 
\begin{split}
L_{clus}(Z(\theta)) & + \;  \; \gamma \; L_{bce}(\hat{A}(Z(\theta)), \,A^{self}) = \\
& L_{\mathcal{C}}(Z(\theta), A^{clus} + \gamma A^{self} ) + \; \gamma \;  L_{\mathcal{R}}(Z(\theta), A^{self}).
\end{split}
\end{equation*}
\end{theorem}

In Theorem \ref{theorem_1}, we have a typical GAE-based clustering model that optimizes a linear combination between embedded k-means and reconstruction. Based on Proposition \ref{proposition_1} and Proposition \ref{proposition_2}, we can write the loss function of this GAE model in the form of a linear combination between a graph-weighted loss $L_{\mathcal{C}}(Z(\theta), A^{clus} + \gamma A^{self})$ and a regularization term $L_{\mathcal{R}}(Z(\theta), A^{self})$. The regularization term $L_{\mathcal{R}}$ enables the training process to avoid degenerate solutions. The graph associated with $L_{\mathcal{C}}$ is a combination between the clustering graph and the self-supervision graph. Based on this result, we can clearly spot the trade-off between FR and FD, which is caused by combining two graphs of different nature. On the one hand, decreasing the balancing hyper-parameter $\gamma$ reinforces the impact of the clustering graph on the optimization process, which in turn gives rise to FR. On the other hand, increasing $\gamma$ leads to higher levels of FD due to the high-sparsity and clustering-irrelevant links within the self-supervision graph. It is important to highlight that our previous work \cite{paper14} has shown the trade-off between FR and FD only for the specific case, where the encoder and decoder are linear functions and the weight matrices are constrained to the Stiefel manifold. As opposed to that, Theorem \ref{theorem_1} holds for all GAE models.

\subsection{Impact of performing clustering and reconstruction at different layers on FR and FD}

To understand the impact of a GAE model on FR and FD compared with a vanilla auto-encoder model, we analyze $\Lambda_{FR}$ and $\Lambda_{FD}$ in a variety of contexts. To this end, we start by computing the gradient of the clustering and reconstruction losses w.r.t. the embedded representations. 
 
\begin{proposition} \label{proposition_3}
The gradient of the reconstruction loss $L_{bce}(\hat{A}(Z(\theta)), A^{self})$ w.r.t. the embedded representation $z_{i}$ can be expressed as:
\begin{equation*} 
\begin{split}
\frac{\partial L_{bce}(\hat{A}(Z(\theta)), A^{self})}{\partial z_{i}} = \sum_{1 \leqslant j \leqslant N} (\hat{a}_{ij}-a_{ij}^{self}) z_{j}.
\end{split}
\end{equation*}
\end{proposition}

\begin{proposition}\label{proposition_4}
The gradient of the clustering loss $L_{clus}(Z(\theta))$ w.r.t. the embedded representation $z_{i}$ can be expressed as:
\begin{equation*} 
\begin{split}
\frac{\partial L_{clus}(Z(\theta))}{\partial z_{i}} = \sum_{1 \leqslant j \leqslant N} a_{ij}^{clus} (z_{i} - z_{j}).
\end{split}
\end{equation*}
\end{proposition}

In Proposition \ref{proposition_3}, we compute the gradient of the reconstruction loss, and in Proposition \ref{proposition_4}, we compute the gradient of the embedded k-means loss. To facilitate the theoretical analysis of FR and FD, we perform three simplifications. Since the trade-off between FR and FD is \emph{only} related to the graph-weighted functions $L_{\mathcal{C}}$, we exclude the regularization term $L_{\mathcal{R}}$ from the gradient computation. Restraining our analysis to the $L_{\mathcal{C}}$ functions simplifies the analytical computation for evaluating FR and FD. In another simplification, we use the inner product for measuring the similarity between the gradient vectors instead of using the cosine function. Using the inner product overcomes the need to deal with the gradient norms. The final simplification consists of using normalized graphs. We denote the normalized self-supervised adjacency matrix by $\Tilde{A}^{self} = D^{-\frac{1}{2}}A^{self}D^{-\frac{1}{2}} = (\tilde{a}_{ij}^{self})_{ij}$, where $D = \text{diag}(d_{1}, ..., d_{n})$ is the degree matrix of $A^{self}$ such that $d_{i} = \sum_{j=1}^{n} A_{ij}^{self}$. Furthermore, $A^{clus}$ and $A^{sup}$ are normalized matrices by definition. Based on the aforementioned simplifications, we can obtain elementary metrics for assessing FR and FD as explained by Definition \ref{definition1} and Definition \ref{definition2} respectively.

\begin{definition} \label{definition1}
For a GAE model $\mathcal{Q}$, we define a metric $\Lambda'_{FR}(\mathcal{Q}, z_{i})$ to evaluate the impact of FR at the level of an embedded point $z_{i}$ as follows:
$$\Lambda'_{FR}(\mathcal{Q}, z_{i}) = \inp*{\frac{\partial L_{\mathcal{C}}(Z(\theta), A^{clus})}{\partial z_{i}}}{\frac{\partial L_{\mathcal{C}}(Z(\theta), A^{sup})}{\partial z_{i}}}.$$
\end{definition}

\begin{definition} \label{definition2}
For a GAE model $\mathcal{Q}$, we define a metric $\Lambda'_{FD}(\mathcal{Q}, z_{i})$ to evaluate the impact of FD at the level of an embedded point $z_{i}$ as follows:
$$\Lambda'_{FR}(\mathcal{Q}, z_{i}) = \inp*{\frac{\partial L_{\mathcal{C}}(Z(\theta), \Tilde{A}^{self})}{\partial z_{i}}}{\frac{\partial L_{\mathcal{C}}(Z(\theta), A^{sup})}{\partial z_{i}}}.$$
\end{definition}

\begin{figure}
  \centering
  \includegraphics[width=0.55\textwidth, height=4.7cm]{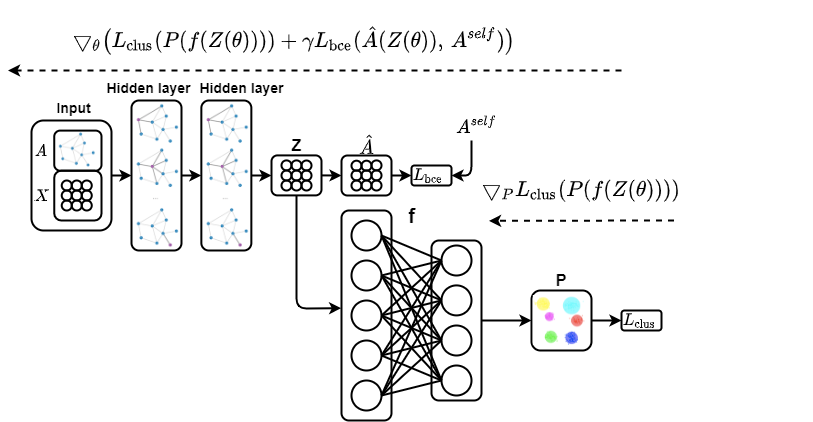}
  \caption{Adding fully-connected encoding layers on top of the last graph convolutional layer, and performing clustering at the level of the last encoding layer.}
  \label{fig:scenario_1}
  \vspace{-2.5mm}
\end{figure}

Modern neural networks are Lipschitz functions. The Lipschitz constant of a function informs how much the output can change in proportion to an input change. Constraining the Lipschitz constant of a neural network is connected to several interesting aspects. For instance, reducing this constant enhances adversarial robustness \cite{paper106}. For classification, reducing the Lipschitz constant induces better generalization bounds as shown by several previous works \cite{paper104, paper105, paper103}. In this section, we show the impact of constraining the Lipchitz constant on FR and FD for two specific situations. In the subsequent analysis, we make use of the following definition: 

\begin{definition} \label{definition3}
Given two metric spaces $(\mathcal{X}, d_{\mathcal{X}})$ and $(\mathcal{Y}, d_{\mathcal{Y}})$, where $d_{\mathcal{X}}$ is a metric on set $\mathcal{X}$ and $d_{\mathcal{Y}}$ is a metric on set $\mathcal{Y}$, a function $f: \mathcal{X} \rightarrow  \mathcal{Y}$ is called Lipschitz continuous if:
$$\exists \:\tau_{1} \geqslant 0, \:\:\:\:\: \forall \: x_{1}, \: x_{2}  \in \mathcal{X} \:\:\: \norm{f(x_{2}) - f(x_{1})}_{d_{\mathcal{Y}}} \leqslant \tau_{1} \norm{x_{2} - x_{1}}_{d_{\mathcal{X}}},$$
and the Lipschitz constant $\tau_{1}^{*}$ of $f$ is defined as:
$$\tau_{1}^{*} = \sup_{x_{1} \neq x_{2}}\bigg(  \frac{\norm{f(x_{2}) - f(x_{1})}_{d_{\mathcal{Y}}}}{\norm{x_{2} - x_{1}}_{d_{\mathcal{X}}}} \bigg).$$
If $f$ is a Lipschitz function and there exists $\tau_{2} \geqslant 0$ such that for all $x_{1}, \: x_{2}  \in \mathcal{X} \:\:\: \norm{f(x_{2}) - f(x_{1})}_{d_{\mathcal{Y}}} \geqslant \frac{1}{\tau_{2}} \norm{x_{2} - x_{1}}_{d_{\mathcal{X}}}$, then $f$ is bi-Lipschitz. We denote the Lipschitz constant of $f^{-1}:f(\mathcal{X}) \rightarrow \mathcal{X}$ as $\tau_{2}^{*}$.
\end{definition}

Unlike typical auto-encoder models for euclidean data clustering, GAE models perform clustering and reconstruction at the same level (i.e., same layer). We study the impact of performing clustering and reconstruction at different layers on FR and FD. To this end, we consider two possible scenarios. Let $\mathcal{NN}(d, d', L)$ be a family of fully-connected layers denoted by $f$: 

\begin{equation*} 
\begin{split}
 f \; : \; & \mathbb{R}^{d} \to \mathbb{R}^{d'} \\
   & z \mapsto ReLU(W_{l}...ReLU(W_{1}z + b_{1})...+b_{l}),  
\end{split}
\end{equation*}

\noindent such that $l \: = \: 1,...,L$ indexes the different layers of the network $f$, $W_{l} \in \mathbb{R}^{d^{(l)} \times d^{(l-1)}}$, $b_{l} \in \mathbb{R}^{d^{(l)}}$, $d=d^{(1)}$, and $d'=d^{(l)}$. The first scenario consists of adding fully-connected encoding layers on top of the last graph convolutional layer, and performing clustering at the level of the last encoding layer. This scenario is illustrated in Figure \ref{fig:scenario_1}. The second scenario consists of adding fully-connected decoding layers on top of the last graph convolutional layer, and performing reconstruction at the level of the last decoding layer. This scenario is illustrated in Figure \ref{fig:scenario_2}. Accordingly, we compare the behaviour of a typical GAE-based clustering model with the two versions described by Figure \ref{fig:scenario_1} and Figure \ref{fig:scenario_2}, in terms of FR and FD, at the level of the embedded representations.

\begin{figure}
  \centering
  \includegraphics[width=0.55\textwidth, height=4.7cm]{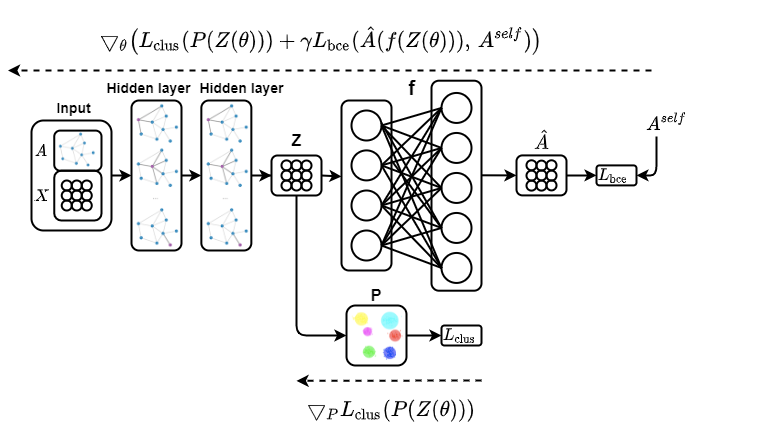}
  \caption{Adding fully-connected decoding layers on top of the last graph convolutional layer, and performing reconstruction at the level of the last decoding layer.}
  \label{fig:scenario_2}
  \vspace{-2.5mm}
\end{figure}


\begin{theorem} \label{theorem_2}
Given two GAE models $\mathcal{Q}_{1}$ and $\mathcal{Q}_{2}$, which have the same graph convolutional layers. $\mathcal{Q}_{1}$ optimizes the objective function in Equation (\ref{Q_1}) and $\mathcal{Q}_{2}$ minimizes the loss function in Equation (\ref{Q_2}), where $f \in \mathcal{NN}(d, d', L)$ and $d' \ll d$. Let $\tau_{1}^{*}$ be the Lipschitz constant of $f$, $\bar{Z}_{i} = (z_{jj'} - z_{ij'})_{j,j'}  \in \mathbb{R}^{N \times d}$, $\zeta_{i} = (\left\| z_{j} - z_{i} \right\|_{2})_{j} \in \mathbb{R}^{N}$, and $a_{i}$ is the $i^{th}$ row of A. 

\begin{equation}  \label{Q_1}
L_{\mathcal{Q}_{1}} = L_{clus}(Z(\theta)) + \gamma L_{bce}(\hat{A}(Z(\theta)), \,A^{self}),
\end{equation}
\begin{equation}  \label{Q_2}
L_{\mathcal{Q}_{2}} = L_{clus}(f(Z(\theta))) + \gamma L_{bce}(\hat{A}(Z(\theta)), \,A^{self}).
\end{equation}

\vspace{5mm}

\begin{itemize}
  \setlength\itemsep{1em}
  \item $ \Lambda'_{FD}(\mathcal{Q}_{2}, z_{i}) = \Lambda'_{FD}(\mathcal{Q}_{1}, z_{i}).$
  \item $ \text{If }$ $$\tau_{1}^{*} \leqslant \sqrt{\frac{(\bar{Z}_{i}^{T} a_{i}^{sup})^{T}  (\bar{Z}_{i}^{T} a_{i}^{clus})}{(\zeta_{i}^{T} \: a_{i}^{sup})(\zeta_{i}^{T}a_{i}^{clus})}},$$ 
$\text{then } \;\;$  $$\Lambda'_{FR}(\mathcal{Q}_{2}, z_{i}) \leqslant \Lambda'_{FR}(\mathcal{Q}_{1}, z_{i}).$$
\end{itemize}

\end{theorem}

In Theorem \ref{theorem_2}, we study the first scenario where a bunch of encoding layers is added on top of the last graph convolutional layer, and the clustering loss is applied at the level of the last encoding layer. We know that reducing the Lipchitz constant is linked to a better generalization capacity \cite{paper103}. Based on Theorem \ref{theorem_2}, we found that a constrained Lipchitz constant of the network $f$ leads to more FR compared with the initial GAE-based clustering model. Furthermore, we found that FD is not affected by the added encoding layers. Hence, we conclude that adding encoding layers independently from the decoding operation increases FR without affecting FD. An intuitive interpretation of this result comes from the fact that the gradient of the reconstruction loss does not back-propagate through the added encoding layers. Therefore, the clustering loss becomes more prone to random projections.

\begin{theorem} \label{theorem_3}
Given two GAE models $\mathcal{Q}_{1}$ and $\mathcal{Q}_{2}$, which have the same graph convolutional layers. $\mathcal{Q}_{1}$ optimizes the objective function in Equation (\ref{Q_3}) and $\mathcal{Q}_{2}$ minimizes the loss function in Equation (\ref{Q_4}), where $f \in \mathcal{NN}(d, d', L)$ an injective function and $d' \gg d$. Let $\tau_{2}^{*}$ be the Lipschitz constant of $f^{-1}:f(\mathbb{R}^{d}) \rightarrow \mathbb{R}^{d}$, $\bar{Z}_{i}^{'} = ((f(z_{j}))_{j'} - (f(z_{i}))_{j'})_{j,j'} \in \mathbb{R}^{N \times d'}$, $\zeta_{i}^{'} = (\left\| f(z_{j}) - f(z_{i}) \right\|_{2})_{j} \in \mathbb{R}^{N} $, and $a_{i}$ is the $i^{th}$ row of A.

\begin{equation}  \label{Q_3}
L_{\mathcal{Q}_{1}} = L_{clus}(Z(\theta)) + \gamma L_{bce}(\hat{A}(Z(\theta)), \,A^{self}),
\end{equation}
\begin{equation}  \label{Q_4}
L_{\mathcal{Q}_{2}} = L_{clus}(Z(\theta)) + \gamma L_{bce}(\hat{A}(f(Z(\theta))), \,A^{self}).
\end{equation}

\vspace{5mm}

\begin{itemize}
  \setlength\itemsep{1em}
  \item $ \Lambda'_{FR}(\mathcal{Q}_{2}, z_{i}) = \Lambda'_{FR}(\mathcal{Q}_{1}, z_{i}).$
  \item $ \text{If }$ $$\tau_{2}^{*} \leqslant  \sqrt{\frac{(\bar{Z}_{i}^{'T} a_{i}^{sup})^{T}  (\bar{Z}_{i}^{'T} \tilde{a}_{i}^{self})}{(\zeta_{i}^{'T} \: a_{i}^{sup})(\zeta_{i}^{'T}\tilde{a}_{i}^{self})}},$$ 
  $\text{then } \;\;$  $$\Lambda'_{FD}(\mathcal{Q}_{2}, z_{i}) \geqslant \Lambda'_{FD}(\mathcal{Q}_{1}, z_{i}).$$
\end{itemize}
\end{theorem}

In Theorem \ref{theorem_3}, we study the second scenario where a bunch of decoding layers is added on top of the last graph convolutional layer, and the reconstruction loss is applied at the level of the last decoding layer. This case is similar to the typical auto-encoder, where the decoder has several layers. Based on Theorem \ref{theorem_3}, we found that a constrained Lipchitz constant of $f^{-1}$ leads to less FD compared with the initial GAE-based clustering model. Intuitively, it is expected that the decoding layers attenuate the effect of FD when the gradient of the reconstruction loss has to back-propagate through several layers.

\subsection{Impact of the graph convolutional operation on FD}

The graph convolutional operation constitutes a principal difference between a typical auto-encoder model and a GAE model. For this reason, we study the impact of this operation on the clustering task from the perspective of FD. Features propagation for a single GCN layer is expressed by the rule $X^{(k+1)} = \phi(\Tilde{A}^{self}X^{(k)}W_{k})$, where $X^{(k)}$ represents the node features of the $k^{th}$ layer, $W_{k}$ is the matrix of trainable weights associated with this layer, and $\phi$ is an activation function. The multiplication of the graph filter $\Tilde{A}^{self}$ with the graph signal $X^{(k)}$ defines the graph convolutional operation. Let $h$ be an aggregation function such that $h^{sup}(x_{i}) = \sum_{j} \tilde{a}_{ij}^{sup} x_{j}$ is the center of the true cluster associated with $x_{i}$ (computed based on ground-truth assignments), and $h^{self}(x_{i}) = \sum_{j} \tilde{a}_{ij}^{self} x_{j}$ is the center of the immediate neighbors of $x_{i}$ according to $A^{self}$. In Equation (\ref{P(x)}), we define a function $\mathcal{P}$ to locally assess the impact of the graph filtering operation on the clustering task.

\begin{equation}  \label{P(x)}
\mathcal{P}(x_{i}) = \norm{x_{i} - h^{sup}(x_{i})}_{2} - \norm{h^{self}(x_{i}) - h^{sup}(x_{i})}_{2}.
\end{equation}

If $\mathcal{P}(x_{i}) \geqslant 0$, we say that the graph filtering operation has a positive impact on clustering the node $v_{i}$. To understand the impact of the filtering operation on FD, we consider two possible scenarios. 

\begin{assumption}{1}{} \label{assumption_1}
The self-supervision adjacency matrix $\Tilde{A}^{self}$ represents the immediate neighbors with a small error, that is,
$$\forall i, j \in \left [|1, N | \right ], \; \; \text{such that} \; \; \tilde{a}_{ij}^{self} \neq 0, \; \; \; x_{i} = x_{j} + \epsilon_{ij},$$ where $\epsilon_{ij} \in \mathbb{R}^{J}$ is a small error (i.e., $\epsilon_{ij} \text{ almost equal to zero}$).
\end{assumption}

\begin{assumption}{2}{} \label{assumption_2}
The immediate neighbors of a node $v_{i}$ are assumed to activate the same neurons for a well-trained ReLU-Affine layer with a training weight $W$, that is,
$$\forall i, j  \;\; \text{if} \;\; \tilde{a}_{ij}^{self} \neq 0 \;\; \text{then} \;\; \text{Sign}(W^{T}x_{i}) = \text{Sign}(W^{T}x_{j}).$$
\end{assumption}

\begin{theorem} \label{theorem_4}
Given two models $\mathcal{Q}_{1}$ and $\mathcal{Q}_{2}$, which optimize the same objective function as described by Equation (\ref{Q_5}). $\mathcal{Q}_{1}$ has a single fully-connected encoding layer characterized by the function $f_{1}(X)=ReLU(XW)$, where $W$ represents the learning weights of this layer. $\mathcal{Q}_{2}$ has a single graph convolutional layer characterized by the function $f_{2}(X)=ReLU(\Tilde{A}^{self}XW)$.

\begin{equation}  \label{Q_5}
L_{\mathcal{Q}_{1}} = L_{\mathcal{Q}_{2}} = L_{clus}(Z(\theta)) + \gamma \: L_{bce}(\hat{A}(Z(\theta)), \,A^{self}).
\end{equation}

\noindent Under Assumption \ref{assumption_1} and Assumption \ref{assumption_2}, we have:
$$\text{If} \;\; \mathcal{P}(f_{1}(x_{i})) \geqslant 0 \;\; \text{then} \;\; \Lambda'_{FD}(\mathcal{Q}_{2}, x_{i}) \leqslant \Lambda'_{FD}(\mathcal{Q}_{1}, x_{i}).$$
\end{theorem}

In Theorem \ref{theorem_4}, we study the first scenario, which consists of comparing a one-layer graph convolutional encoder against a one-layer fully-connected encoder. Our proof depends on two reasonable properties of $\Tilde{A}^{self}$. Specifically, we know by definition that $\Tilde{A}^{self}$ connects each node with few immediate neighbors as opposed to $A^{sup}$, which connects each node with all nodes from the same true cluster. Assumption \ref{assumption_1} states that the immediate neighbors of a node $v_{i}$ are represented with small errors. The second Assumption \ref{assumption_2} asserts that the immediate neighbors of a node $v_{i}$ activate the same neurons for a well-trained layer. Under these mild assumptions, Theorem \ref{theorem_4} indicates that performing a graph convolutional operation before a fully-connected layer increases the effect of FD on a node $v_{i}$, if the graph convolutional operation has a positive impact on clustering $v_{i}$. 
Intuitively, $\Tilde{A}^{self}$ only considers the immediate neighbors (due to the sparsity of $\Tilde{A}^{self}$) and maintains some clustering-irrelevant links. For every layer, we know that the graph convolutional operation is equivalent to minimizing the loss function $L_{\mathcal{C}}(X^{(k)}, \Tilde{A}^{self})$ \cite{paper9}, which implies an increase of FD at the level of the same layer.

\begin{theorem} \label{theorem_5}
Given two models $\mathcal{Q}_{1}$ and $\mathcal{Q}_{2}$, which optimize the same objective function as described by Equation (\ref{Q_6}). $\mathcal{Q}_{1}$ has a single graph convolutional layer characterized by the function $f_{1}(X)=ReLU(\Tilde{A}^{self}XW_{1})$, where $W_{1}$ represents the learning weights of this layer. $\mathcal{Q}_{2}$ has two graph convolutional layers characterized by the function $f_{2}(X)=ReLU(\Tilde{A}^{self} \; ReLU(\Tilde{A}^{self} X W_{1}) \; W_{2})$, where $W_{2}$ represents the learning weights of the second layer. We suppose that the Lipschitz constant $\tau_{1}^{*}$ of the second graph convolutional layer is less or equal to 1.

\begin{equation}  \label{Q_6}
L_{\mathcal{Q}_{1}} = L_{\mathcal{Q}_{2}} = L_{clus}(Z(\theta)) + \gamma L_{bce}(\hat{A}(Z(\theta)), \,A^{self}).
\end{equation}

\noindent Under Assumption \ref{assumption_1} and Assumption \ref{assumption_2}, we have:
$$\text{If} \;\; \mathcal{P}(f_{1}(x_{i})) \geqslant 0 \;\; \text{then} \;\; \Lambda'_{FD}(\mathcal{Q}_{2}, x_{i}) \leqslant \Lambda'_{FD}(\mathcal{Q}_{1}, x_{i}).$$
\end{theorem}

In Theorem \ref{theorem_5}, we study the second scenario, which consists of comparing a one-layer graph convolutional encoder against a two-layer graph convolutional encoder. Similar to Theorem \ref{theorem_4}, our proof relies on Assumption \ref{assumption_1} and Assumption \ref{assumption_2}. As a result, we found that adding a graph convolutional layer increases the effect of FD on a node $v_{i}$, if the graph convolutional operation has a positive impact on clustering $v_{i}$. Intuitively,  the smoothing effect of each layer propagates to the embedded representations $Z$, which in turn drift the clustering-oriented structures. For instance, an infinite-depth graph convolutional network produces the same embedded vector for each node \cite{paper107}. Mapping all nodes to the same embedded point renders the clustering irrelevant.

\section{Proposed operators}

Our theoretical analysis indicates the limitations of GAE models in tackling the FR and FD problems. Motivated by these limitations, we propose two operators that can be easily integrated into existing models. Most importantly, our operators gradually transform the general-purpose self-supervised graph into a clustering-oriented graph. 
Firstly, we design a sampling operator $\Xi$ that triggers a protection mechanism against FR. More precisely, $\Xi$ can delay FR from quickly taking place. Secondly, we propose an operator $\Upsilon$ that triggers a correction mechanism against FD. $\Upsilon$ revokes the impact of FD by gradually transforming the reconstructed graph into a clustering-oriented one. 

\subsection{A protection mechanism against FR}

Some supervised methods \cite{paper56, paper57, paper58} handle the impact of corrupted labels by an iterative selection protocol. Specifically, samples with clean labels are selected to train the model. Then, this latter is progressively used to select more samples with clean labels. In most cases, consistently high-confidence predictions, during training, are generally associated with uncorrupted samples. This strategy can be considered a correction mechanism as identifying the noisy samples requires training with them in advance. However, it is not clear to what extent the predictions of a noisy classifier (i.e., trained with random labels) are sufficient to recognize samples with corrupted labels. 

Compared with existing sampling techniques for supervised learning, our strategy is motivated by two additional insights. The first idea consists of using a protection mechanism against FR, instead of a correction one. In fact, it has been observed that fine-tuning a model by training it on ground-truth labels, once the pretraining phase is performed on random labels, can not reverse the impact of labels' randomness \cite{paper64}. Since a correction mechanism can not reverse the effect of labels' randomness, we opt for a protection mechanism that prioritizes the selection of samples with uncorrupted labels, before using them for training. Our sampling technique is initiated directly after the pretraining phase and exploits two strong criteria to collect a sufficient portion of nodes with reliable clustering assignments. Second, we argue that it is important to control the selection process according to the difference between the first high-confidence and second high-confidence clustering assignment scores. This aspect is quite useful when the labels can be flipped between two similar clusters.

We propose three guidelines to develop our sampling operator $\Xi$. The first guideline consists of transforming hard clustering assignments into soft assignments. To this end, we compute the matrix $(p'_{ij})_{i,j} \in\mathbb{R}^{N \times K}$. If $(p_{ij})_{i,j}$ is already a soft assignment matrix, then we set $p'_{ij} = p_{ij}$. If the matrix $(p_{ij})_{i,j}$ is a hard assignment matrix, then we measure the similarity between the embedded points and the clustering representatives according to:

\begin{equation} 
  \begin{aligned}
    p'_{ij} =  \frac{exp(- \frac{1}{2}  \left ( z_{i}-\mu_{j}\right )^{T} \Sigma_{j}^{-1}  \left ( z_{i}-\mu_{j}\right ))}{\sum_{j=1}^{K} exp(- \frac{1}{2}  \left ( z_{i}-\mu_{j}\right )^{T} \Sigma_{j}^{-1}  \left ( z_{i}-\mu_{j}\right ))},
  \end{aligned}
  \label{eq:p'}
\end{equation}

\noindent where $\mu_{j}$ stands for the center of cluster $C_{j}^{clus}$, and $\Sigma_{j}$ is a diagonal matrix representing the cluster variances. The second guideline consists of extracting the first and second high-confidence assignment scores from matrix $(p'_{ij})_{i,j}$ for each node. The first score associated with $z_{i}$ is denoted by $\lambda_{i}^{1}$:

\begin{equation}
  \begin{aligned}
   \lambda_{i}^{1} = \max_{_{j \in \left \{ 1,...,K \right \}}}(p'_{ij}).
  \end{aligned}
\label{eq:lambda_1}
\end{equation}

The second high-confidence assignment score for the embedded representation $z_{i}$ is denoted by $\lambda_{i}^{2}$:

\begin{equation}
  \begin{aligned}
   \lambda_{i}^{2} = \max_{_{j \in \left \{ 1,...,K \right \}}}(p'_{ij} \,\, | \,\, p'_{ij} < \lambda_{i}^{1}).
  \end{aligned}
\label{eq:lambda_2}
\end{equation}

The third guideline consists of constructing a set $\Omega(t)$ that contains nodes, whose clustering assignments at iteration $t$ are reliable enough to decide to which cluster they belong. 
Points from $\Omega$ are selected according to two criteria as described by Eqn. (\ref{eq:Omega1}). 
\begin{equation}\label{eq:Omega1}   
    \begin{aligned}
        \Omega=\left \{i \in \mathcal{V} | \; \lambda_{i}^{1} \geq   \alpha_{1} \; \textbf{and}  \;  (\lambda_{i}^{1}-\lambda_{i}^{2}) \geq  \alpha_{2}  \right \}.
    \end{aligned}
\end{equation}

First, a node from $\Omega$ is situated close to its closest cluster representative. Consequently, its first high-confidence assignment score is greater than a threshold $\alpha_{1}$, where $\alpha_{1}$ is a tunable hyper-parameter within the range $\left [ 0, \, 1 \right ]$. Second, a point from $\Omega$ is located far from the borderline between neighbor clusters as described by Equation (\ref{eq:Omega1}). Consequently, the difference between the first and second high-confidence assignment scores is greater than a threshold $\alpha_{2}$. We set $\alpha_{2} = \frac{\alpha_{1}}{2}$. Our sampling operator $\Xi$ is summarized in Algorithm \ref{algorithm1}. The computational complexity of Algorithm \ref{algorithm1} is $\mathcal{O}(NK^{2}d)$.

\begin{algorithm}[t]
\caption{Operator $\Xi$.}
\label{algorithm1}
\begin{algorithmic}[1]
\State {\bfseries Input:} Embedded data: $Z$, Number of clusters: $K$, First confidence threshold: $\alpha_{1}$, Second confidence threshold: $\alpha_{2}$.
\State {\bfseries Output:} Embedded representations of decidable nodes: $Z[\Omega]$.
\State Compute the matrix $(p'_{ij})_{i,j}\in\mathbb{R}^{N \times K }$ according to Eqn. (\ref{eq:p'}).
\For{$i=0$ {\bfseries to} $\left | X \right |$}
    \State Compute $\lambda_{i}^{1}$ according to Equation (\ref{eq:lambda_1}).
    \State Compute $\lambda_{i}^{2}$ according to Equation (\ref{eq:lambda_2}).
\EndFor
\State Construct $\Omega$ according to Equation (\ref{eq:Omega1}).
\State \textbf{Return} $Z[\Omega]$.
\end{algorithmic}
\end{algorithm}

\subsection{A correction mechanism against FD}

Real-world graphs carry edges that connect nodes from different clusters. Reconstructing the input graph structure is not suitable for learning clustering-oriented embeddings. To attenuate FD, we use the embeddings of reliable nodes $\Xi(Z(\theta))$ to gradually transform the reconstruction objective into a clustering-oriented cost. This can be done by gradually substituting the self-supervisory signal $A^{self}$ with a task-specific signal $\Upsilon(A, P(\Xi(Z(\theta))), \Omega)$. 

We propose two guidelines for developing the graph transforming operator $\Upsilon$. The first guideline consists of identifying a centroid node for each cluster. To this end, we compute $\tilde{\mu}_{j}$, which averages the embedded representations of reliable nodes from cluster $C_{j}^{clus}$. Then, for each $\tilde{\mu}_{j}$, we search for its nearest node, in the embedded space, among the set $\Omega$. The list of obtained nodes is denoted by $\Pi = \left [i \in \mathcal{V} | \; i= \text{1-NN}(\tilde{\mu}_{j}, \Omega) \text{ and } j \in \left \{ 1,...,K \right \} \right ]$, where 1-NN represents the nearest neighbor algorithm. 

The second guideline consists of constructing a new self-supervisory signal $A_{clus}^{self}$ based on the original graph structure $A$. To this end, we start by connecting each node from $\Omega$ with its associated centroid from $\Pi$. Then, we drop edges between nodes from $\Omega$, which are members of different clusters. As a result, the obtained graph $A_{clus}^{self}$ contains $K$ star-shaped sub-graphs representing the different clusters. Algorithm \ref{algorithm2} summarizes our proposed operator $\Upsilon$. The worst-case complexity of Algorithm \ref{algorithm2} is $\mathcal{O}(N(d+K)+|\mathcal{E}|(N+K))$. 

\begin{algorithm}[t]
\caption{Operator $\Upsilon$.}
\label{algorithm2}
\begin{algorithmic}[1]
\State {\bfseries Input:} Original sparse graph: $A$, Clustering assignment: $P$, Set of decidable nodes: $\Omega$.
\State {\bfseries Output:} Clustering-oriented self-supervision graph: $A_{clus}^{self}$.
\State $\Pi \leftarrow  \left [i \in \mathcal{V} | \; i= \text{1-NN}(\tilde{\mu}_{j}, \Omega) \text{ and } j \in \left \{ 1,...,K \right \} \right ]$.
\State $A_{clus}^{self} \leftarrow  A$
\For{$i$ in $\Omega$}
    \State $k_{1} \leftarrow  \text{arg}\max_{_{k}}(P[i, k])$
    \State $j \leftarrow  \Pi[k_{1}]$
    \State $k_{2} \leftarrow  \text{arg}\max_{_{k}}(P[j, k])$
    \If{($j \notin A[i]\text{.indices}$) and ($k_{1}=k_{2}$)} \Comment{$A[i]\text{.indices}$ indicates the list of nodes connected to node $i$.}  
        \State $A_{clus}^{self}[i, j] \leftarrow 1$
    \EndIf
    \For{$l$ in $A[i].indices$}
        \State $k_{2} \leftarrow \text{arg}\max_{_{k}}(P[l, k])$
        \If{($l \in \Omega$) and ($k_{1} \neq k_{2}$)}
            \State $A_{clus}^{self}[i, l] \leftarrow 0$
        \EndIf
    \EndFor
\EndFor 
\State \textbf{Return}  $A_{clus}^{self}$.
\end{algorithmic}
\end{algorithm}

A protection mechanism against FD can be established by transforming the self-supervisory signal $A$ into a clustering-oriented signal $\Upsilon(A, P(Z(\theta)), \mathcal{V})$, in a single step. This is done by applying $\Upsilon$ to the whole set of nodes $\mathcal{V}$, instead of $\Omega$. We argue that a correction mechanism, which allows FD to take place then gradually attenuates this problem, is a more advantageous
solution. 

\section{Experiments}

In order to validate the suitability of our conceptual design and our proposed operators, we conduct an extensive experimental protocol\footnote{We bring to the attention of the reader that our code can be found at: \url{https://github.com/nairouz/R-GAE}}. We show that it is possible to substantially improve the clustering performance of several GAE-based clustering models by integrating operators that can control FR and FD. We obtain promising results, which calls for further research in this direction. 

\subsection{Experimental settings}

Due to the limited number of second-group models, we propose a new approach entitled DGAE from this group. For the sake of reproducibility, we provide a technical description of this method in Appendix B. Our experimental protocol covers six models (GAE \cite{paper11}, VGAE \cite{paper11}, ARGAE \cite{paper25}, ARVGAE \cite{paper25}, GMM-VGAE \cite{paper26}, and DGAE). GAE, VGAE, ARGAE, and ARVGAE belong to the first GAE-based clustering group, which, as discussed in Section 2, establish clustering and embedding learning separately. DGAE and GMM-VGAE are members of the second group, which ensures joint clustering and embedding learning. For GAE, VGAE, ARGAE, and ARVGAE, we use the publicly available implementations. For GMM-VGAE, we reproduce their reported results by performing our implementation. We integrate our operators $\Xi$ and $\Upsilon$ into the aforementioned models. For the first group, we use $\Xi$ and $\Upsilon$ to gradually transform the reconstruction loss into a clustering-oriented objective, during the pretraining phase. We keep the original settings (optimizer, hyper-parameters, architecture) of each model for fairness of comparison. The obtained methods are abbreviated by (R-GAE, R-VGAE, R-ARGAE, R-ARVGAE, R-GMM-VGAE, R-DGAE). ``R-$\mathcal{D}$" stands for Rethinking the model $\mathcal{D}$ (i.e., GAE, VGAE, ARGAE, ARVGAE, GMM-VGAE, DGAE) from the perspective of FR and FD. To avoid training instability due to the consistent modification of the self-supervisory signal, we update  $\Omega$ and $A_{clus}^{self}$ every $M_{1}$ and $M_{2}$ iterations, respectively. We train the obtained models until meeting the convergence criterion $| \Omega | \geq  0.9  * |\mathcal{V}|$. Compared with the original approaches, that is, GAE, VGAE, ARGAE, ARVGAE, GMM-VGAE, and DGAE, three additional hyper-parameters, namely $M_{1}$, $M_{2}$ and $\alpha_{1}$, should be specified. The values of these parameters are provided in Appendix C. We assess the proposed operators on six benchmark datasets. Our evaluation includes three citation networks (Cora, Citeseer, and Pubmed \cite{paper84}) and three air-traffic networks (USA, Europe, and Brazil \cite{paper28}). Since the air-traffic networks go without node attributes, we leverage the one-hot encoding of node degrees to construct the feature matrix $X$ similar to \cite{paper85}. For all datasets, $X$ is (row-)normalized with the Euclidean norm. 




\subsection{Results}
We present the principal results of our experiments in this section. However, due to limited space, we provide further experiments and results in Appendix \ref{appendix_J}. 

\textbf{Effectiveness:} 
In Tables \ref{Table:best_acc_nmi_ari}, \ref{Table:mean_acc_nmi_ari}, \ref{Table:best_acc_nmi_ari_GMM-VGAE_R-GMM-VGAE}, and \ref{Table:mean_acc_nmi_ari_GMM-VGAE_R-GMM-VGAE}, we report the best and average clustering results among three trials on six datasets. For all tables, we mark the best methods in bold and the clustering performances in \%. For fairness of comparison, we ensure that each couple of methods $\mathcal{D}$ and R-$\mathcal{D}$ share the same pretraining weights before starting the clustering phase. Table \ref{Table:best_acc_nmi_ari} provides the best clustering performances on three citation networks. From this table, we observe that the second GAE group methods yield considerably better results than methods from the first group. These results confirm that performing joint clustering and embedding learning is advantageous to the clustering task. Among the first group, we can see that (R-GAE, R-VGAE, R-ARGAE, R-ARVGAE) generally have better ACC, NMI, and ARI compared with their counterparts (GAE, VGAE, ARGAE, ARVGAE). The embedded representations of (GAE, VGAE, ARGAE, ARVGAE) are optimized using the reconstruction objective. These methods do not suffer from FR and FD. By gradually transforming the graph reconstruction into a clustering-oriented loss, during the training process, (R-GAE, R-VGAE, R-ARGAE, R-ARVGAE) make the embedded representations more clustering-oriented. Among the second group, we observe that (R-GMM-VGAE, R-DGAE) outperform their counterparts (GMM-VGAE, DGAE) consistently by a significant margin. To confirm these results, we extend the performed experiments to three additional datasets as shown in Table \ref{Table:best_acc_nmi_ari_GMM-VGAE_R-GMM-VGAE}. Our results offer strong evidence that the proposed operators can improve the clustering effectiveness of GAE models in terms of ACC, NMI, and ARI. Since this manuscript aims at investigating the impact of FR and FD, we focus on (R-GMM-VGAE, R-DGAE) and their counterparts (GMM-VGAE, DGAE) in the subsequent experiments. Moreover, we provide a comprehensive comparison against several recent graph clustering methods in Appendix \ref{appendix_J}.

\begin{table*}[!h]
  \caption{Best clustering performance for the original and proposed GAE models on Cora, Citeseer and Pubmed.}
  \vspace{-0.5\baselineskip}
  \begin{center}
  \begin{small}
  \begin{tabular}{|p{2.3cm}|c|c|c|c|c|c|c|c|c|}
    \hline
    {Method} & \multicolumn{3}{c|}{Cora} & \multicolumn{3}{c|}{Citeseer} & \multicolumn{3}{c|}{Pubmed}  \\
    \cline{2-10}
    & ACC & NMI & ARI & ACC & NMI & ARI & ACC & NMI & ARI \\ \hline
    \hline
    GAE & 61.3 & 44.4 & 38.1 & 48.2 & 22.7 & 19.2 & 64.2 & 22.5 & 22.1  \\ \hline
    \textbf{R-GAE}  & \textbf{65.8} & \textbf{51.6} & \textbf{44.1} & \textbf{50.1} & \textbf{24.6} & \textbf{20.0} & \textbf{69.6} & \textbf{31.4} & \textbf{31.6}  \\ \hline
    \hline
    
    VGAE  & 64.7 & 43.4 & 37.5 & \textbf{51.9} & \textbf{24.9} & \textbf{23.8} & \textbf{69.6} & 28.6& \textbf{31.7} \\ \hline
    \textbf{R-VGAE}  & \textbf{71.3} & \textbf{49.8} & \textbf{48.0} & 44.9 & 19.9 & 12.5 & 69.2 & \textbf{30.3} & 30.9 \\ \hline
    \hline
    
    ARGAE & 64.0 & 44.9 & 35.2 & \textbf{57.3} & \textbf{35.0} & \textbf{34.1} & 68.1 & 27.6 & 29.1 \\ \hline
    \textbf{R-ARGAE}  & \textbf{72.0} & \textbf{51.5} & \textbf{49.5} & 49.3 & 28.4 & 17.4 & \textbf{70.2} & \textbf{31.4} & \textbf{32.6}\\ \hline
    \hline

    ARVGAE & 63.8 & 45.4	& 40.1 & 54.4 & 26.1 & 24.5 & 63.5 & 23.2 & 22.5 \\ \hline
    \textbf{R-ARVGAE} & \textbf{67.2} & \textbf{47.4} & \textbf{44.0} & \textbf{59.4} & \textbf{32.5} & \textbf{31.4} & \textbf{65.9} & \textbf{24.3} & \textbf{25.2}\\ \hline
    \hline
    
    DGAE & 70.2 & 50.7 & 47.2 & 67.7 & 40.9 & 42.5 & 68.4 & 29.0 & 29.1 \\ \hline
    \textbf{R-DGAE} & \textbf{73.7} & \textbf{56.0} & \textbf{54.1} & \textbf{70.5} & \textbf{45.0} & \textbf{47.1} & \textbf{71.4} & \textbf{34.4} & \textbf{34.6}  \\ \hline
    \hline
    
    GMM-VGAE & 71.9 & 53.3 & 48.2 & 67.5 & 40.7 & 42.4 & 71.1 & 29.9 & 33.0 \\ \hline
    \textbf{R-GMM-VGAE} & \textbf{76.7} & \textbf{57.3} & \textbf{57.9} & \textbf{68.9} & \textbf{42.0} & \textbf{43.9} & \textbf{74.0} & \textbf{33.4} & \textbf{37.9} \\ \hline
 
  \end{tabular}
  \end{small}
  \end{center}
  \label{Table:best_acc_nmi_ari}
\end{table*}

\begin{table*}[!h]
  \caption{Mean and standard deviation of evaluation metrics for the original and proposed GAE models on Cora, Citeseer and Pubmed.}
  \vspace{-0.5\baselineskip}
  \begin{center}
  \begin{small}
  \begin{tabular}{|p{2.3cm}|c|c|c|c|c|c|c|c|c|}
    \hline
    {Method} & \multicolumn{3}{c|}{Cora} & \multicolumn{3}{c|}{Citeseer} & \multicolumn{3}{c|}{Pubmed}  \\
    \cline{2-10}
    & ACC & NMI & ARI & ACC & NMI & ARI & ACC & NMI & ARI \\ \hline
    \hline

    GAE & 55.6 $\pm$ 4.9 & 41.2 $\pm$ 2.8 & 33.2 $\pm$ 4.5 & 42.5	$\pm$ 5.2 & 19.9 $\pm$ 2.6  & 13.7 $\pm$ 5.6 &  63.7 $\pm$ 0.5 & 23.3 $\pm$ 1.4 & 22.7	$\pm$ 1.6  \\ \hline
    \textbf{R-GAE} & \textbf{65.0 $\pm$ 1.0} & \textbf{50.2 $\pm$ 1.3} & \textbf{43.3 $\pm$ 0.7} & \textbf{49.8 $\pm$ 0.5} & \textbf{24.3 $\pm$ 0.3} & \textbf{19.6 $\pm$ 0.6} &  \textbf{68.0 $\pm$ 1.4} & \textbf{28.7 $\pm$ 2.4} &  \textbf{29.3 $\pm$ 2.1} \\ \hline
    \hline

    VGAE  & 58.6 $\pm$ 5.3 & 40.1	$\pm$ 2.9 & 34.2 $\pm$ 2.9 & \textbf{50.3 $\pm$ 1.6} & \textbf{23.6	$\pm$ 1.7} & \textbf{22.1 $\pm$ 2.4} & 68.9	$\pm$ 0.8 & 28.3 $\pm$ 1.1 & 30.6 $\pm$ 1.1 \\ \hline
    \textbf{R-VGAE} & \textbf{70.3 $\pm$ 1.2} & \textbf{48.8 $\pm$ 0.9} & \textbf{46.7 $\pm$ 1.2} & 42.6 $\pm$ 2.1 & 14.9 $\pm$ 4.3 & 12.3 $\pm$ 0.3 & \textbf{68.9 $\pm$ 0.3} & \textbf{29.9 $\pm$ 0.4} & \textbf{30.6 $\pm$ 0.3}\\ \hline
    \hline

    ARGAE & 59.3 $\pm$ 4.0 & 42.2 $\pm$ 2.5 & 31.6 $\pm$ 5.0 & 36.6 $\pm$ 8.4 & 28.4 $\pm$ 4.0 & 16.1 $\pm$ 7.5 &  68 $\pm$ 0.1 & 29.4 $\pm$ 1.6 & 29.3	$\pm$ 0.3 \\ \hline
    \textbf{R-ARGAE} & \textbf{71.2 $\pm$ 0.7} & \textbf{50.73 $\pm$ 0.8} & \textbf{47.1 $\pm$ 2.3} & \textbf{48.6 $\pm$ 0.7} & \textbf{28.5 $\pm$ 0.3} & \textbf{18.9 $\pm$ 1.3} & \textbf{69.2 $\pm$	0.9} & \textbf{30.0 $\pm$ 1.2} & \textbf{30.9 $\pm$	1.4} \\ \hline
    \hline
    
    ARVGAE & 63.4 $\pm$ 0.7 & 45.3 $\pm$ 0.3 & 39.17 $\pm$ 1.5 & 51.5 $\pm$ 2.9 & 26.3	$\pm$ 1.4 & 22.7 $\pm$ 1.8 & 63.4 $\pm$ 0.1 & 23.1 $\pm$ 0.1 & 22.4 $\pm$ 0.2 \\ \hline
    \textbf{R-ARVGAE} & \textbf{67.0 $\pm$	0.2} & \textbf{47.2 $\pm$ 0.1} & \textbf{43.8 $\pm$ 0.5} & \textbf{59.2 $\pm$ 0.3} & \textbf{31.6 $\pm$ 0.8} & \textbf{30.8 $\pm$ 0.6} & \textbf{65.73 $\pm$ 0.2} & \textbf{23.7 $\pm$ 0.5}  & \textbf{24.9 $\pm$ 0.3} \\ \hline
    \hline
    
    DGAE & 69.8 $\pm$ 0.5 & 49.9 $\pm$ 0.7 & 46.3 $\pm$ 0.9 & 66.5 $\pm$ 1.1 & 39.2 $\pm$ 1.5 & 40.3	$\pm$ 1.9 & 67.8 $\pm$ 0.6 & 28.0 $\pm$ 1.0 & 28.0 $\pm$ 1.0 \\ \hline
    \textbf{R-DGAE} & \textbf{73.1 $\pm$ 0.7} & \textbf{55.3 $\pm$ 0.7} & \textbf{53.0 $\pm$ 1.1} & \textbf{69.5 $\pm$ 0.8} & \textbf{43.7 $\pm$ 1.1} & \textbf{45.7 $\pm$ 1.2} & \textbf{71.0 $\pm$ 0.4} & \textbf{33.6 $\pm$ 0.9} &  \textbf{33.9 $\pm$	0.8} \\ \hline
    \hline
    
    GMM-VGAE & 71.7 $\pm$ 0.2 & 53.0 $\pm$ 0.3 & 47.9 $\pm$ 0.4 & 66.3 $\pm$ 0.5 & 39.5 $\pm$ 0.5 & 41.1 $\pm$ 0.6 &  70.6 $\pm$ 0.5 & 28.7 $\pm$ 1.1 & 32.0 $\pm$ 1.0 \\ \hline
    \textbf{R-GMM-VGAE} & \textbf{75.7 $\pm$ 0.9} & \textbf{55.8 $\pm$ 1.3} & \textbf{56.2 $\pm$ 1.5} & \textbf{68.4 $\pm$ 0.4} & \textbf{41.5 $\pm$ 0.4} & \textbf{43.6 $\pm$ 0.3} & \textbf{72.8 $\pm$ 1.5} & \textbf{32.2 $\pm$ 1.9} &  \textbf{35.7 $\pm$ 2.5} \\ \hline
    
  \end{tabular}
  \end{small}
  \end{center}
  \label{Table:mean_acc_nmi_ari}
\end{table*}

\begin{table*}[!h]
  \caption{Best clustering performance for the original and proposed GAE models on Air-Traffic datasets.}
  \vspace{-0.5\baselineskip}
  \begin{center}
  \begin{small}
  \begin{tabular}{|p{2.3cm}|c|c|c|c|c|c|c|c|c|}
    \hline
    {Method} & \multicolumn{3}{c|}{USA Air-Traffic} & \multicolumn{3}{c|}{Europe Air-Traffic} & \multicolumn{3}{c|}{Brazil Air-Traffic} \\
    \cline{2-10}
     & ACC & NMI & ARI & ACC & NMI & ARI & ACC & NMI & ARI \\ \hline
     \hline
    GMM-VGAE & 48.1 & 21.9 & 13.2 & 53.1 & 31.1	& 24.4 & 70.2 & \textbf{46.0} & 41.9 \\ \hline
    \textbf{R-GMM-VGAE} & \textbf{50.8} & \textbf{23.1} & \textbf{15.3} & \textbf{57.4} & \textbf{31.4} & \textbf{25.8} & \textbf{73.3} & 45.6 & \textbf{42.5} \\ \hline
    \hline
    DGAE & 46.4 & \textbf{28.0} & \textbf{18.4} & 53.6 & 33.3 & 23.3 & 71.0 & 48.0 & 41.2\\ \hline
    \textbf{R-DGAE} & \textbf{51.7}	& 24.7 & 16.5 & \textbf{57.1} & \textbf{34.5} & \textbf{25.2} & \textbf{74.0} & \textbf{51.3} & \textbf{45.4} \\ \hline
  \end{tabular}
  \end{small}
  \end{center}
  \label{Table:best_acc_nmi_ari_GMM-VGAE_R-GMM-VGAE}
\end{table*}

\begin{table*}[!h]
   \caption{Mean and standard deviation of the evaluation metrics for the original and proposed GAE models on Air-Traffic datasets.}
  \vspace{-0.5\baselineskip}
  \begin{center}
  \begin{small}
  \begin{tabular}{|p{2.3cm}|c|c|c|c|c|c|c|c|c|}
    \hline
        {Method} & \multicolumn{3}{c|}{USA Air-Traffic} & \multicolumn{3}{c|}{Europe Air-Traffic} & \multicolumn{3}{c|}{Brazil Air-Traffic} \\
    \cline{2-10}
     & ACC & NMI & ARI & ACC & NMI & ARI & ACC & NMI & ARI \\ \hline \hline

    GMM-VGAE & 47.2 $\pm$ 0.9 & 21 $\pm$ 0.8 & 12.7 $\pm$ 0.5 & 52.3 $\pm$ 1.0 & 29.2	$\pm$ 1.80 & 22.6 $\pm$ 1.5 & 69.0 $\pm$ 1.6 & 43.7 $\pm$ 2.6 &  38.8	$\pm$ 3.2 \\ \hline
    \textbf{R-GMM-VGAE} & \textbf{50.4 $\pm$ 0.59} & \textbf{22.6 $\pm$ 0.5} & \textbf{15.2 $\pm$ 0.6} & \textbf{56.4 $\pm$ 1.3} & \textbf{31.2 $\pm$ 0.78} &  \textbf{25.3 $\pm$ 0.8} & \textbf{71.8 $\pm$ 1.6} & \textbf{45.0 $\pm$ 2.7} & \textbf{41.6	$\pm$ 3.4} \\ \hline
    \hline
     
    DGAE & 45.8 $\pm$ 0.6 & \textbf{28.1 $\pm$ 0.2} & \textbf{18.2 $\pm$ 0.3} & 53.2 $\pm$ 0.5 & 33.1 $\pm$ 0.2 & 23.1 $\pm$ 0.2 & 70.7 $\pm$ 0.4 & 48.1 $\pm$ 1.0 & 39.9	$\pm$ 1.3\\ \hline
    \textbf{R-DGAE} & \textbf{51.3 $\pm$ 0.4} & 24.4 $\pm$ 0.4 & 16.2 $\pm$ 0.4 & \textbf{56.7 $\pm$ 0.7} & \textbf{33.2 $\pm$ 1.1} & \textbf{24.3 $\pm$ 0.8} & \textbf{74.1 $\pm$ 0.3} & \textbf{52.4 $\pm$ 1.3} & \textbf{45.7	$\pm$ 0.6} \\ \hline
    
  \end{tabular}
  \end{small}
  \end{center}
  \label{Table:mean_acc_nmi_ari_GMM-VGAE_R-GMM-VGAE}
\end{table*}

\textbf{Efficiency:} In Table \ref{Table:efficiency}, we compare (R-GMM-VGAE, R-DGAE) with their counterparts (GMM-VGAE, DGAE) in terms of run-time. We report the best, the mean, and the variance in execution time over ten trials. Although Pubmed has almost ten times more edges and features than Cora and Citeseer, we observe that the difference in execution time between (R-GMM-VGAE, R-DGAE) and their counterparts (GMM-VGAE, DGAE) remains considerably small on Pubmed. In accordance with the provided complexity analysis for Algorithm \ref{algorithm1} and Algorithm \ref{algorithm2}, our results confirm that the designed operators do not cause any significant overhead in execution time, compared with the original models.

\begin{table*}[!h]
  \caption{Execution time (in seconds) of the couples (GMM-VGAE, R-GMM-VGAE) and (DGAE, R-DGAE). }
  \vspace{-0.5\baselineskip}
  \begin{center}
  \begin{small}
  \begin{tabular}{|p{2.2cm}|c|c|c|c|c|c|c|c|c|}
    \hline
    {Method} & \multicolumn{3}{c|}{Cora} & \multicolumn{3}{c|}{Citeseer} & \multicolumn{3}{c|}{Pubmed}  \\
    \cline{2-10}
    & Best  & Mean & Variance & Best & Mean & Variance & Best & Mean & Variance\\ \hline \hline
    GMM-VGAE & 17.135 & 17.703 & 0.530 & 36.269 & 36.442 & 1.436 & 1341.190 & 1348.960 & 16.056\\ \hline
    \textbf{R-GMM-VGAE} & 21.928 & 24.509 & 2.589 & 40.084 & 41.910 & 2.884 & 1457.188 & 1477.405 & 155.492\\  \hline
    \hline
    DGAE & 19.298 & 20.179 & 0.644 & 38.074 & 38.226 & 0.012 & 1067.301 & 1076.431 & 33.446  \\ \hline
    \textbf{R-DGAE} & 28.981 & 31.053 & 1.464 & 51.363 & 52.976 & 1.850 & 1192.913 & 1215.241 & 361.036 \\ 
    \hline
  \end{tabular}
  \end{small}
  \end{center}
  \label{Table:efficiency}
\end{table*}

\textbf{Visualisation of $A_{clus}^{self}$:} In Figure \ref{fig:vis_graphs_R-GMM-VGAE}, we visualize the self-supervisory graph $A_{clus}^{self}$ constructed by $\Upsilon$, during the training of R-GMM-VGAE on Cora. As the training progresses, more nodes are connected with their associated centroids. Furthermore, we observe that several clustering-unfriendly edges are dropped. At epoch 120, $A_{clus}^{self}$ contains $7$ star-shaped sub-graphs representing the different clusters. These results confirm the ability of our operator $\Upsilon$ to gradually transform the reconstructed graph into a clustering-oriented graph.

\begin{figure*}[!h]
  \begin{subfigure}[b]{0.24\textwidth}
    \includegraphics[width=\linewidth]{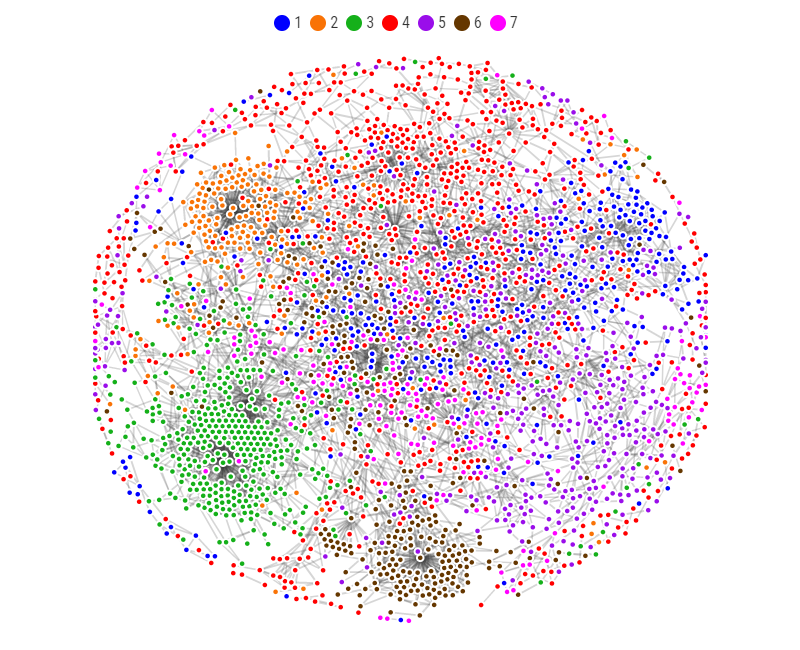}
    \caption{Epoch 0}
  \end{subfigure}
  \begin{subfigure}[b]{0.24\textwidth}
    \includegraphics[width=\linewidth]{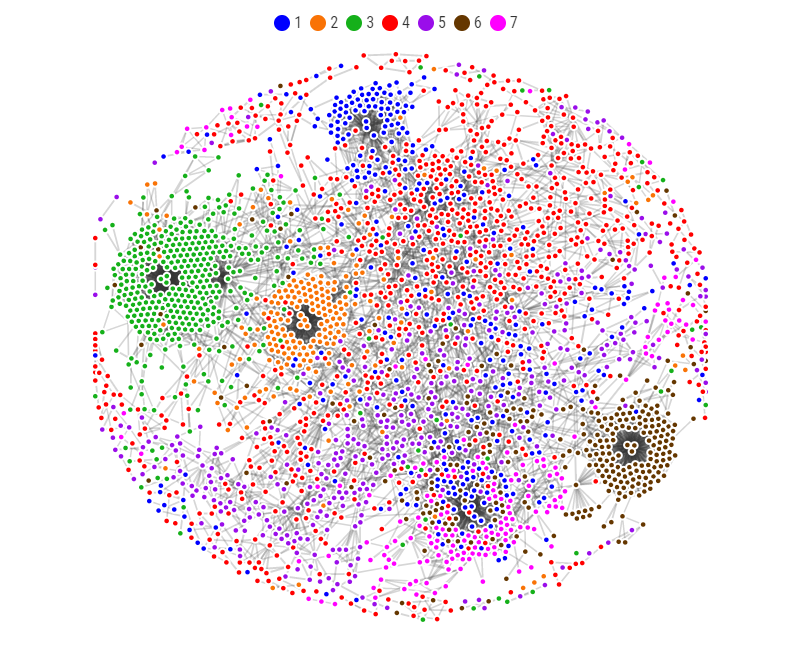}
    \caption{Epoch 40}
  \end{subfigure}
  \begin{subfigure}[b]{0.24\textwidth}
    \includegraphics[width=\linewidth]{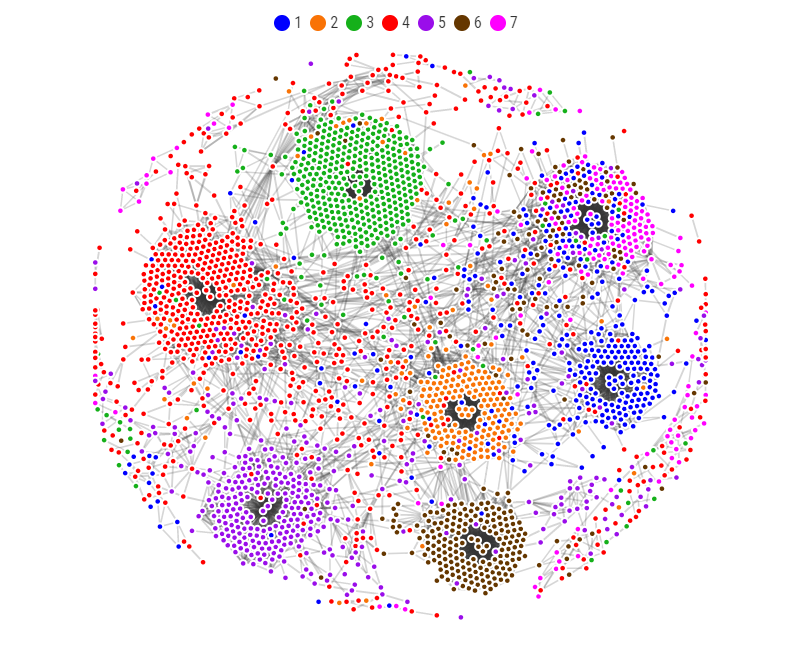}
    \caption{Epoch 80}
  \end{subfigure}
  \begin{subfigure}[b]{0.24\textwidth}
    \includegraphics[width=\linewidth]{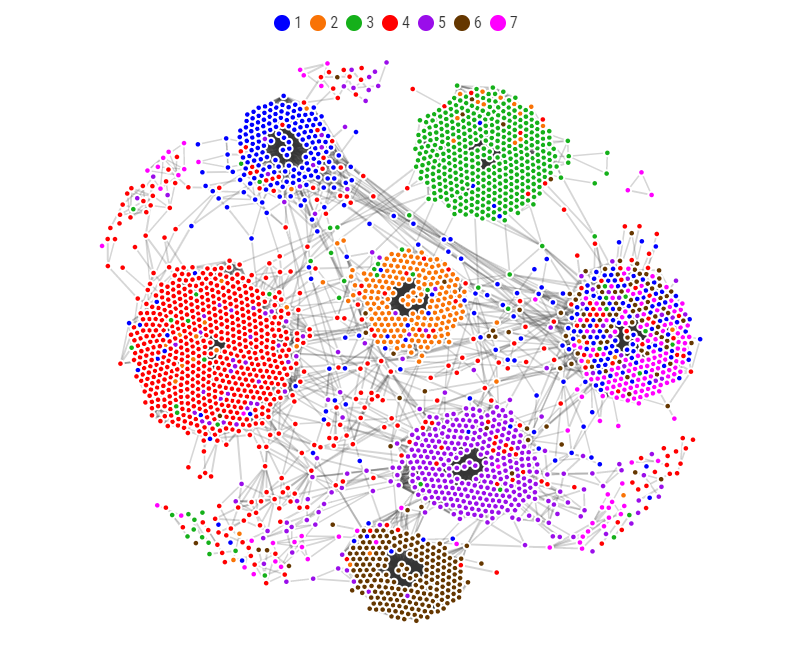}
    \caption{Epoch 120}
  \end{subfigure}
  \caption{Visualizing the self-supervisory graph $A_{clus}^{self}$, on Cora using R-GMM-VGAE.}
  \label{fig:vis_graphs_R-GMM-VGAE}
\end{figure*}

\textbf{Feature Randomness:} In this part, we discuss the evolution of $\Lambda_{FR}$ values for GMM-VGAE and R-GMM-VGAE on Cora. The cosine similarity between the gradient of $L_{clus}(Z(\theta), \,P)$ and the gradient of $L_{clus}(Z(\theta), \,Q')$ is denoted by $\Lambda_{FR} \text{(GMM-VGAE)}$, whereas $\Lambda_{FR}\text{(R-GMM-VGAE)}$ denotes the cosine similarity between the gradient of $L_{clus}(\Xi(Z(\theta)), \,P)$ and the gradient of $L_{clus}(Z(\theta), \,Q')$. We illustrate both metrics, during training of R-GMM-VAGE and GMM-VGAE, in Figures \ref{fig:FR_R-GMM-VGAE} (a) and (b), respectively. To facilitate our analysis, we also provide the normalized cumulative difference between $\Lambda_{FR}\text{(R-GMM-VGAE)}$ and $\Lambda_{FR} \text{(GMM-VGAE)}$, during the training of R-GMM-VAGE and GMM-VGAE, in Figures \ref{fig:FR_R-GMM-VGAE} (d) and (e), respectively. As a general observation from Figures \ref{fig:FR_R-GMM-VGAE} (a), (b), and (c), $\Lambda_{FR} \text{(GMM-VGAE)}$ and $\Lambda_{FR} \text{(R-GMM-VGAE)}$ start from very high values (close to one). This implies that the unsupervised gradient, at an early training stage, has the same direction as the supervised one. This result is congruent with recent findings, which suggest that training with ground-truth or random labels prioritizes learning simple patterns first at the level of the earlier layers \cite{paper52, paper46}. These simple patterns are not dependent on the target labels \cite{paper64}. 

For the first experiment (Figures \ref{fig:FR_R-GMM-VGAE} (a) and (d)), we train R-GMM-VGAE and we report $\Lambda_{FR} \text{(GMM-VGAE)}$, $\Lambda_{FR} \text{(R-GMM-VGAE)}$, and the normalized cumulative difference between both of them. We can see that there are two stages. The first stage ranges from iteration 0 to 60, and the second stage ranges from iteration 60 to 140. For the first stage, we observe that $\Lambda_{FR} \text{(R-GMM-VGAE)}$ is higher than $\Lambda_{FR} \text{(GMM-VGAE)}$. This result is confirmed by observing the cumulative difference between $\Lambda_{FR} \text{(R-GMM-VGAE)}$ and $\Lambda_{FR} \text{(GMM-VGAE)}$ in Figure \ref{fig:FR_R-GMM-VGAE} (d), which has a pronounced increasing tendency. These results demonstrate the ability of our operator $\Xi$ to reduce FR, during the first stage. For the second stage (from iteration 60 to 140 of Figure \ref{fig:FR_R-GMM-VGAE} (a)), the blue and green curves become closer to each other. This observation is confirmed by a lower slope for the curve of Figure \ref{fig:FR_R-GMM-VGAE} (d) compared with the slope of the same curve for the first stage (i.e., between iterations 0 and 60 of Figure \ref{fig:FR_R-GMM-VGAE} (d)). At this point, $\Omega$ gradually approaches $\mathcal{V}$. Therefore, $\Lambda_{FR} \text{(R-GMM-VGAE)}$ becomes approximately equal to $\Lambda_{FR} \text{(GMM-VGAE)}$.

For the second experiment (Figures \ref{fig:FR_R-GMM-VGAE} (b) and (e)), we train GMM-VGAE and we report $\Lambda_{FR} \text{(GMM-VGAE)}$, $\Lambda_{FR} \text{(R-GMM-VGAE)}$, and the normalized cumulative difference between both of them. We observe that $\Lambda_{FR} \text{(R-GMM-VGAE)}$ is consistently close to 1. From Figure \ref{fig:FR_R-GMM-VGAE} (e), we can see that the cumulative difference between $\Lambda_{FR} \text{(R-GMM-VGAE)}$ and $\Lambda_{FR} \text{(R-GMM-VGAE)}$ has almost a constant slope. These results suggest that $\Xi$ can \emph{consistently} select a sufficient amount of reliable nodes even after learning based on unreliable nodes. Thus, $\Xi$ is capable of playing the role of a protection mechanism against FR. 

For the third experiment (Figures \ref{fig:FR_R-GMM-VGAE} (c) and (f)), we train GMM-VGAE and report $\Lambda_{FR} \text{(GMM-VGAE)}$, we train R-GMM-VGAE and report $\Lambda_{FR} \text{(R-GMM-VGAE)}$, and we finally report the normalized cumulative difference between both of them. We can see that there are three stages. The first stage ranges from iteration 0 to 50, the second stage ranges from iteration 50 to 100, and the third stage ranges from iteration 100 to 140. For the first stage, we observe that R-GMM-VGAE outperforms GMM-VGAE in terms of $\Lambda_{FR}$ thanks to our operator $\Xi$. For the second stage, we observe that GMM-VGAE yields better results than R-GMM-VGAE in terms of $\Lambda_{FR}$. To reduce FD, R-GMM-VGAE transforms the reconstruction loss into a clustering-oriented loss. However, eliminating the reconstruction gives rise to FR. Unlike R-GMM-VGAE, GMM-VGAE maintains the reconstruction loss, during the second stage, which is considered an implicit mechanism against FR. Figure \ref{fig:FR_R-GMM-VGAE} (f) shows clearly the trade-off between FR and FD. Although both models have reduced the same amount of FR, delaying the effect of FR has a favorable impact on the clustering performance. For the third stage, both models tie together. This experiment shows that using a protection mechanism delays the effect of FR and does not prevent it from taking place. By delaying the effect of randomness using a protection mechanism, it is possible to improve the clustering performance considerably.

\begin{figure*}[!h]
  \centering
  \begin{subfigure}[b]{0.33\textwidth}
    \includegraphics[width=\linewidth]{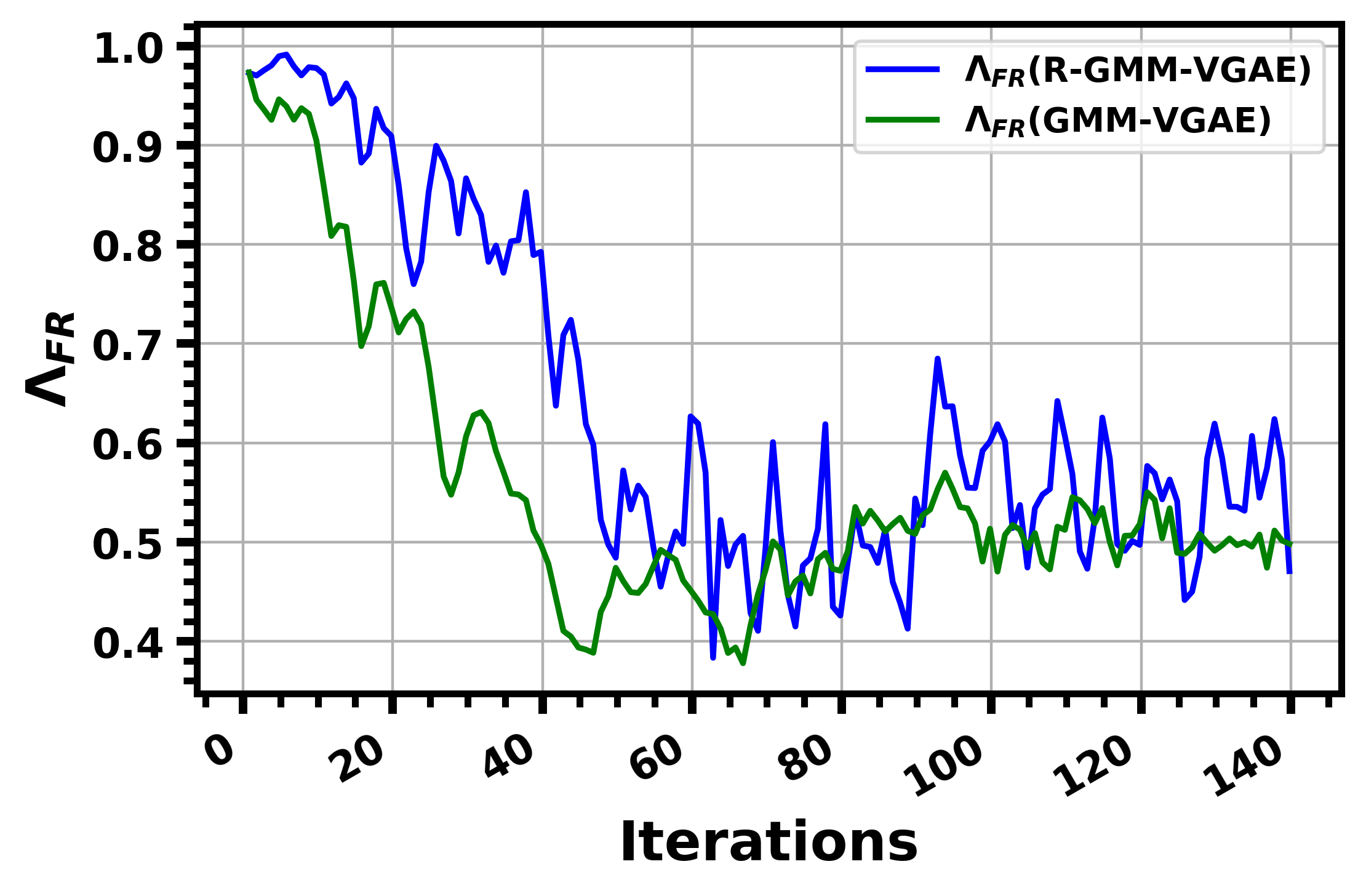}
    \caption{R-GMM-VGAE training}
  \end{subfigure} \hfil
  \begin{subfigure}[b]{0.33\textwidth}
     \includegraphics[width=\linewidth]{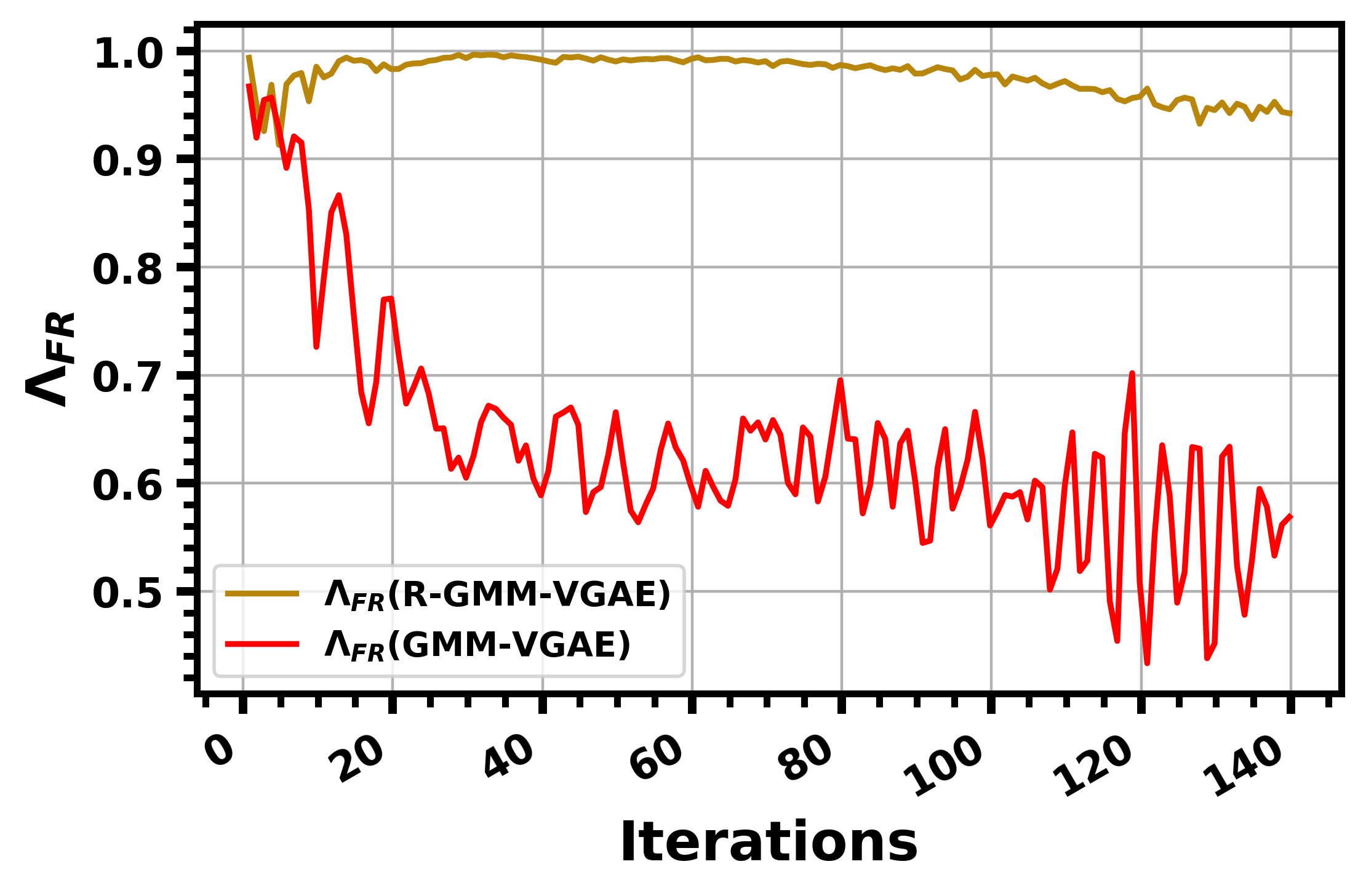}
     \caption{GMM-VGAE training}
  \end{subfigure} \hfil
  \begin{subfigure}[b]{0.33\textwidth}
     \includegraphics[width=\linewidth]{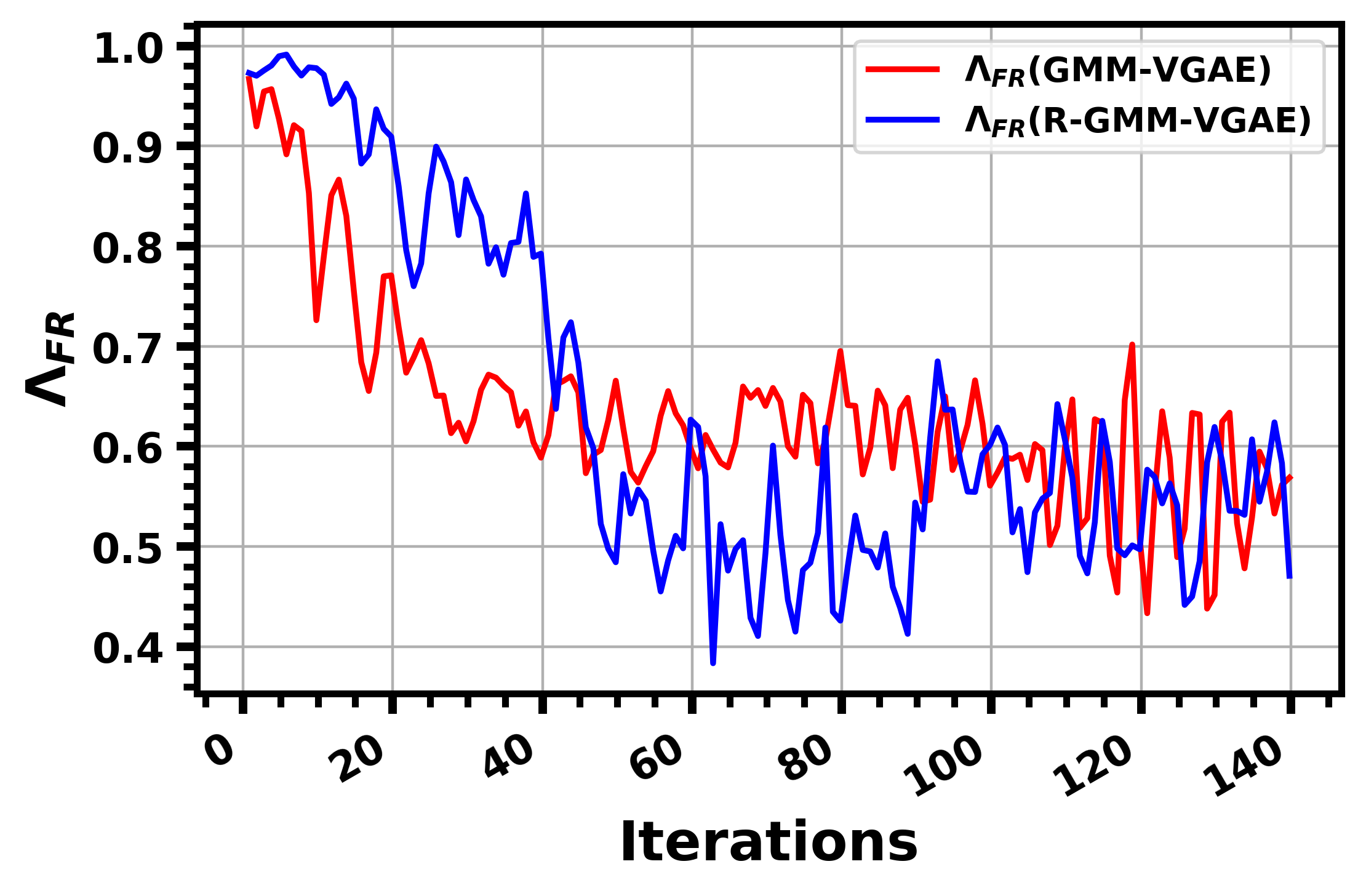}
     \caption{Independent training}
  \end{subfigure} \hfil
  
  \medskip
  \begin{subfigure}[b]{0.33\textwidth}
    \includegraphics[width=\linewidth]{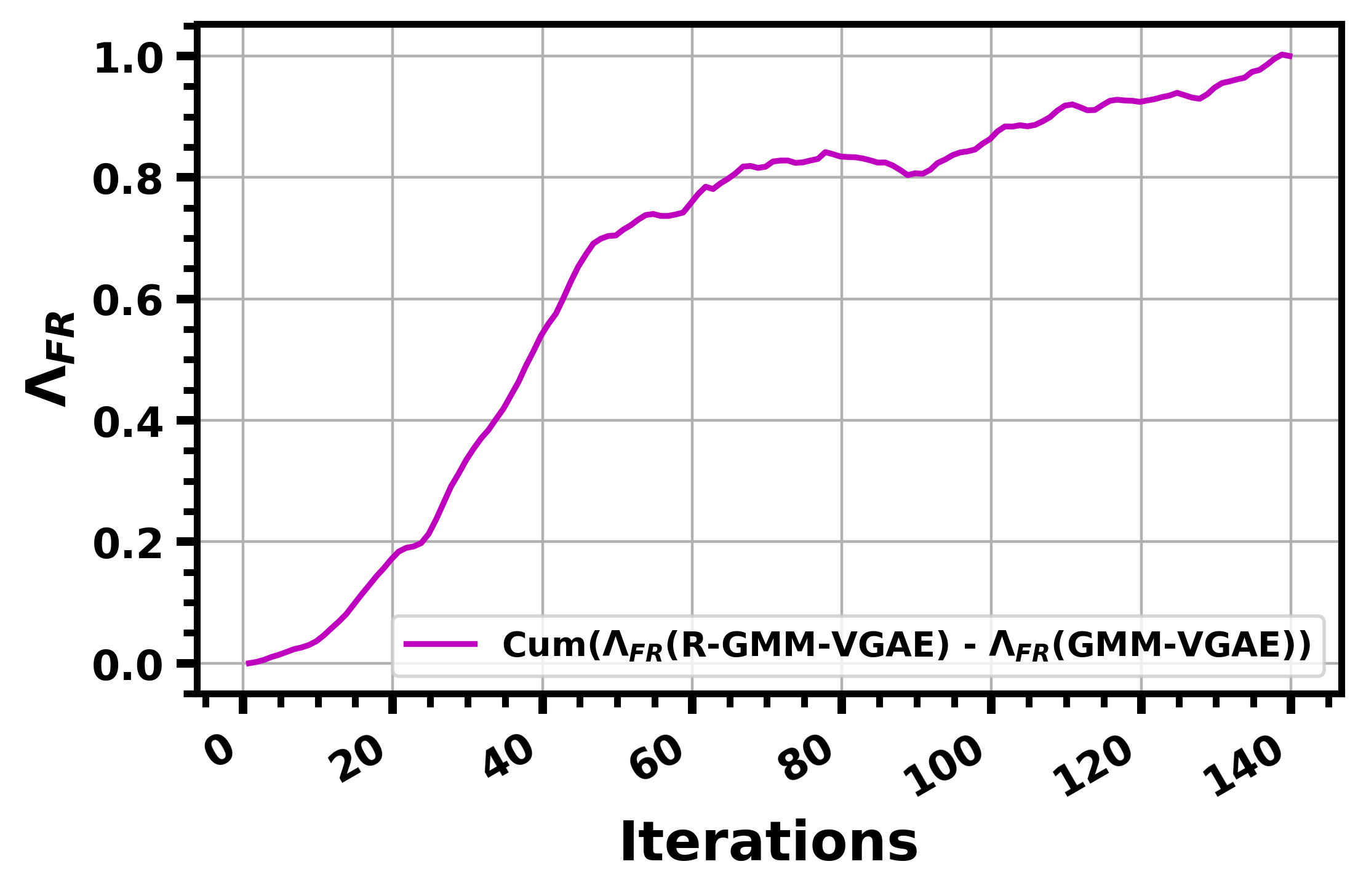}
    \caption{R-GMM-VGAE training}
  \end{subfigure} \hfil
  \begin{subfigure}[b]{0.33\textwidth}
     \includegraphics[width=\linewidth]{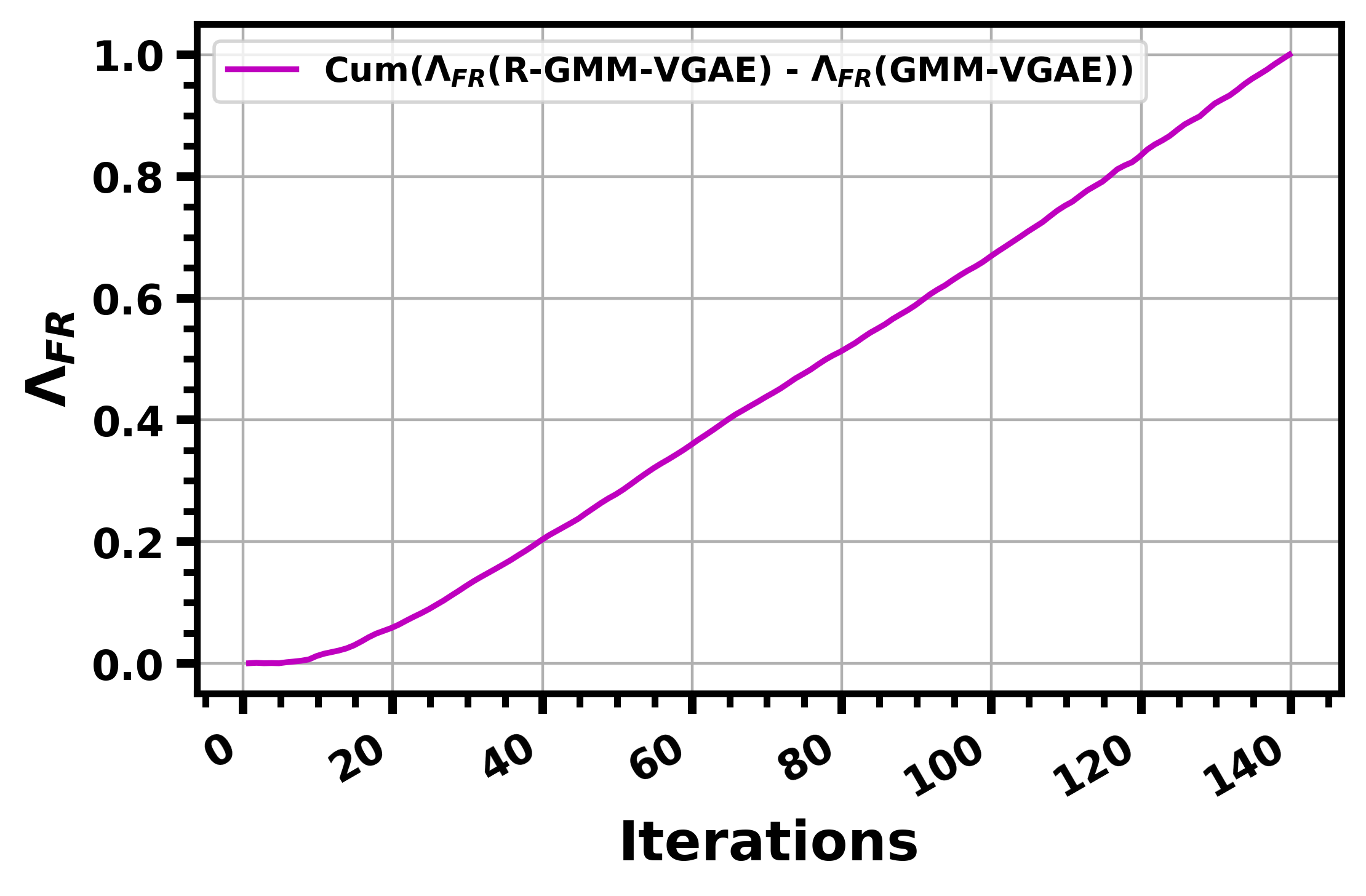}
     \caption{GMM-VGAE training}
  \end{subfigure} \hfil
  \begin{subfigure}[b]{0.33\textwidth}
     \includegraphics[width=\linewidth]{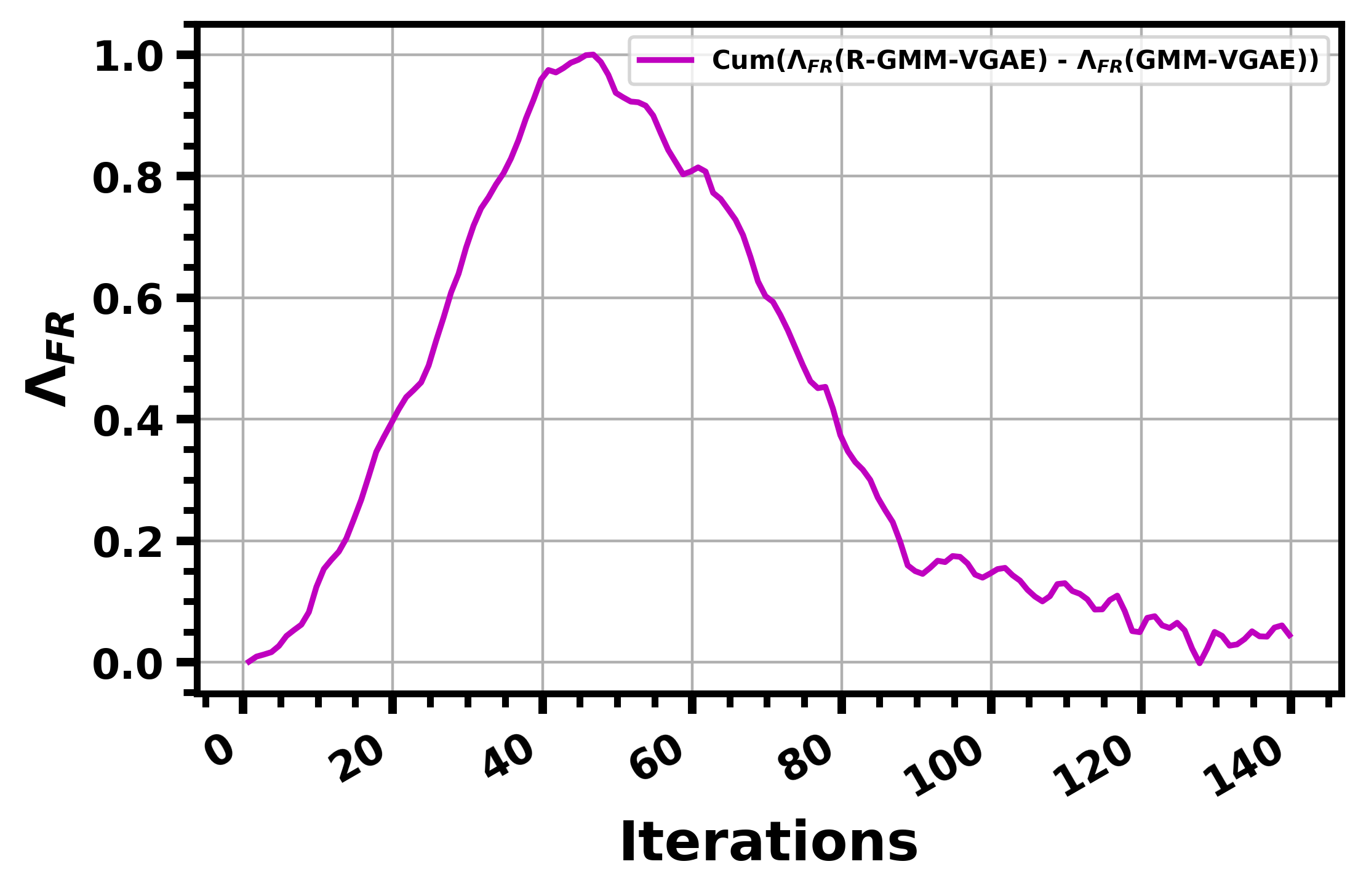}
     \caption{Independent training}
  \end{subfigure}
  
  \caption{Performance of R-GMM-VGAE and GMM-VGAE in terms of $\Lambda_{FR}$ on Cora. Blue line: $\Lambda_{FR}$ values of R-GMM-VGAE, during training of R-GMM-VGAE. Green line: $\Lambda_{FR}$ values of GMM-VGAE, during training of R-GMM-VGAE. Gold line: $\Lambda_{FR}$ values of R-GMM-VGAE, during training of GMM-VGAE. Red line: $\Lambda_{FR}$ values of GMM-VGAE, during training of GMM-VGAE. Purple line: normalized cumulative difference between $\Lambda_{FR}$ values of R-GMM-VGAE and $\Lambda_{FR}$ values of GMM-VGAE.}
  \label{fig:FR_R-GMM-VGAE}
\end{figure*}

\begin{figure*}[!h]
  \centering
  \begin{subfigure}[b]{0.33\textwidth}
    \includegraphics[width=\linewidth]{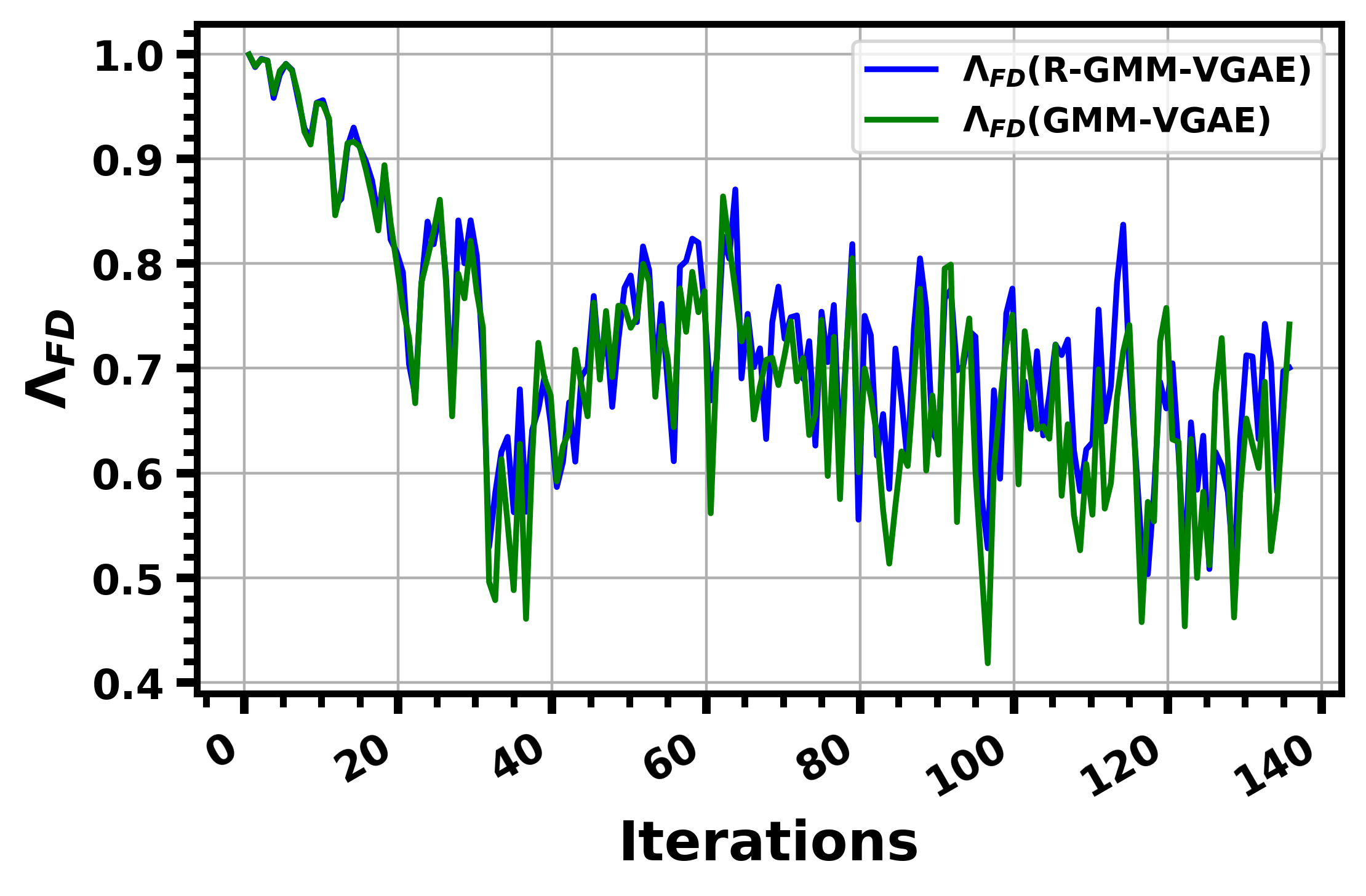}
    \caption{R-GMM-VGAE training}
  \end{subfigure} \hfil
  \begin{subfigure}[b]{0.33\textwidth}
     \includegraphics[width=\linewidth]{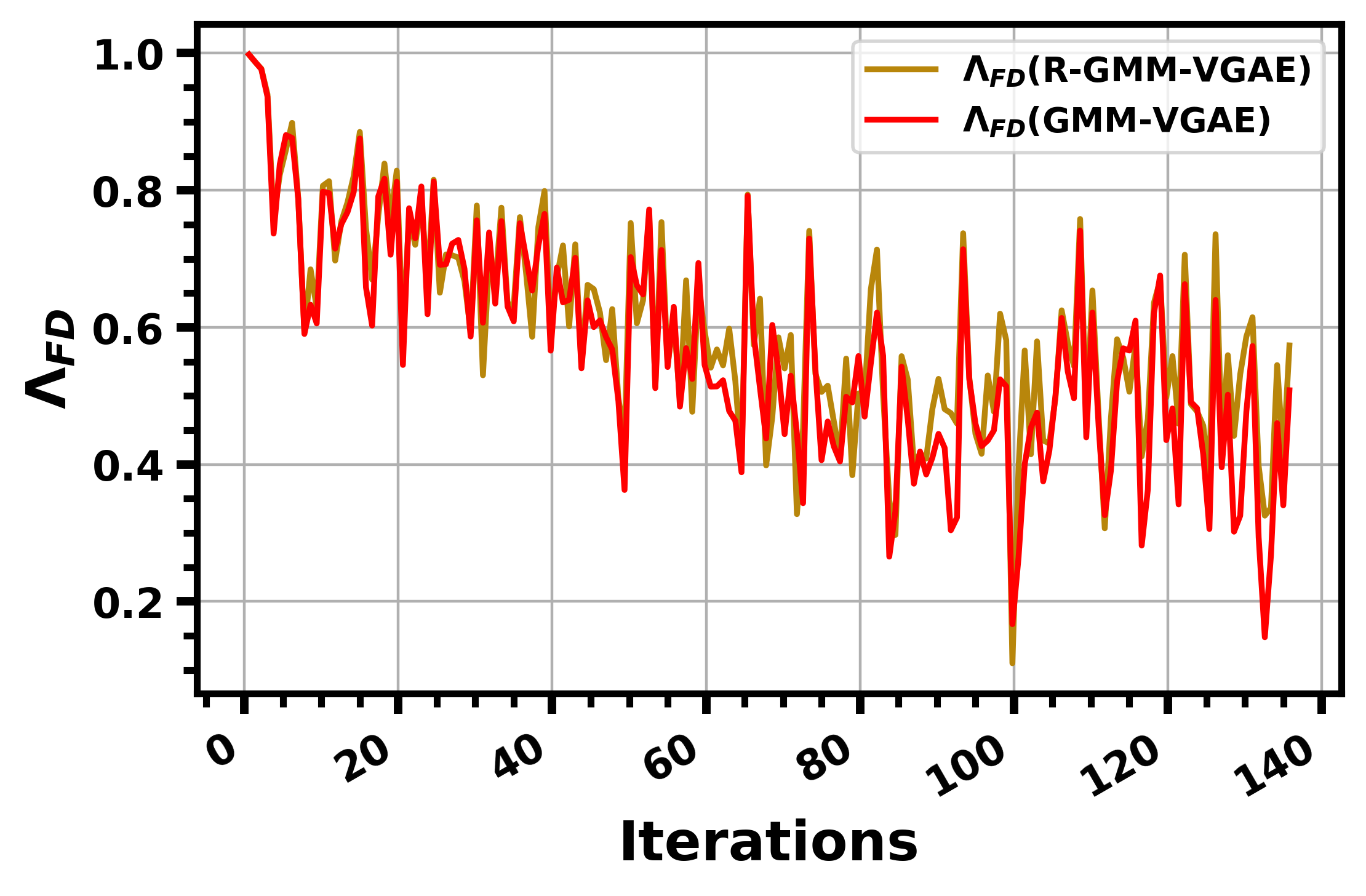}
     \caption{GMM-VGAE training}
  \end{subfigure} \hfil
  \begin{subfigure}[b]{0.33\textwidth}
     \includegraphics[width=\linewidth]{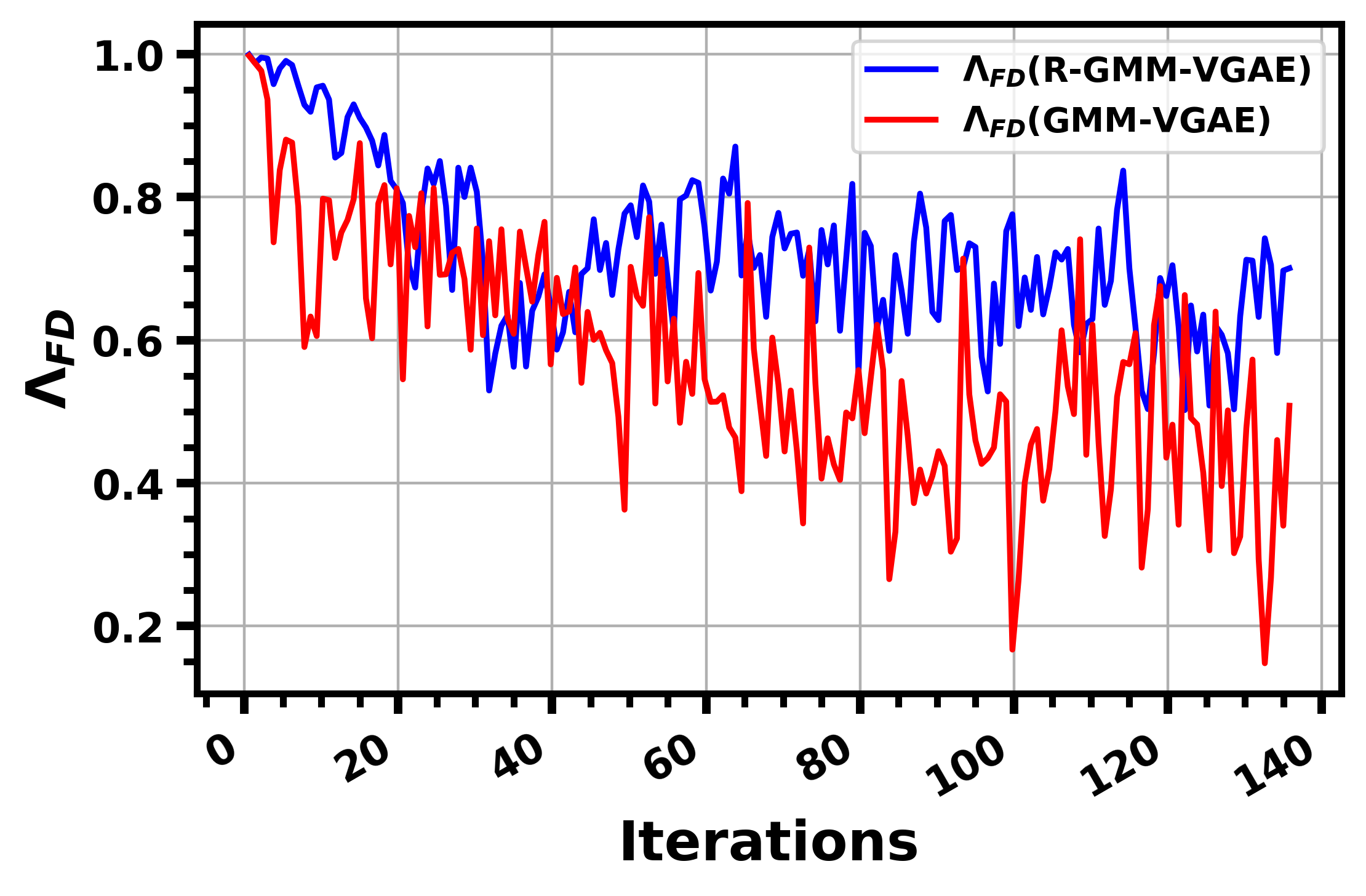}
     \caption{Independent training}
  \end{subfigure} \hfil
  \medskip
  \begin{subfigure}[b]{0.33\textwidth}
    \includegraphics[width=\linewidth]{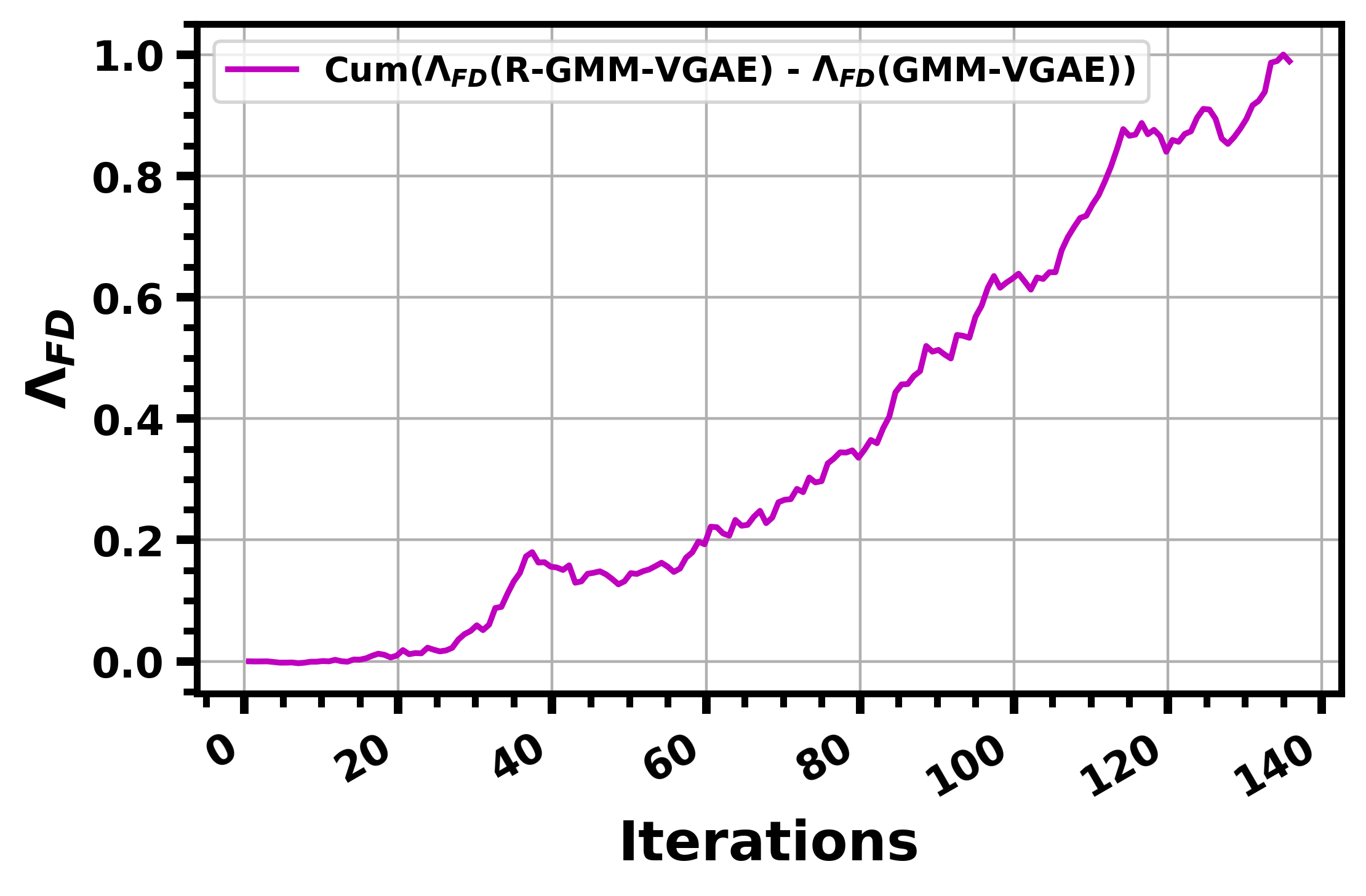}
    \caption{R-GMM-VGAE training}
  \end{subfigure} \hfil
  \begin{subfigure}[b]{0.33\textwidth}
     \includegraphics[width=\linewidth]{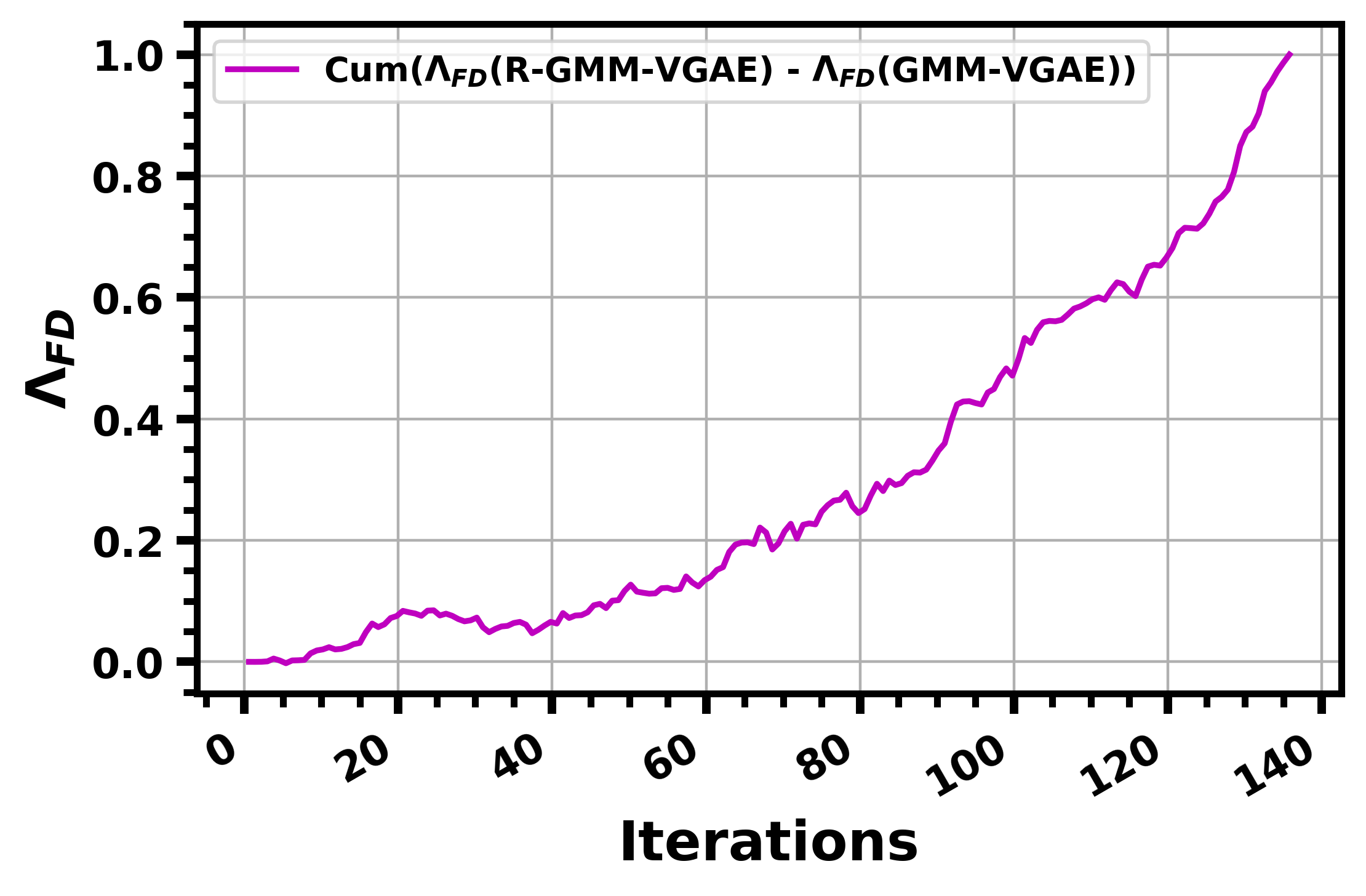}
     \caption{GMM-VGAE training}
  \end{subfigure} \hfil
  \begin{subfigure}[b]{0.33\textwidth}
     \includegraphics[width=\linewidth]{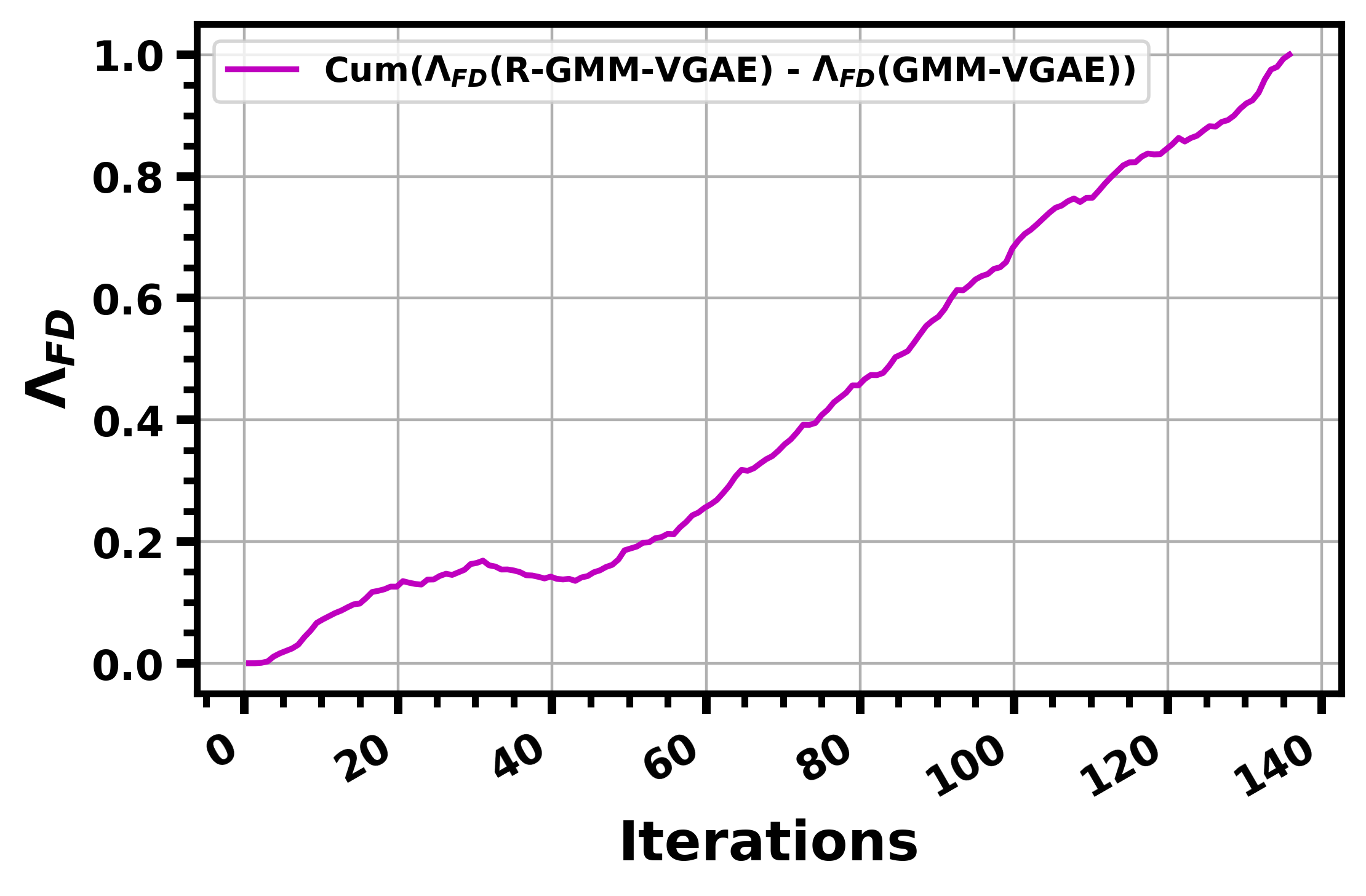}
     \caption{Independent training}
  \end{subfigure}
  
  \caption{Performance of R-GMM-VGAE and GMM-VGAE in terms of $\Lambda_{FD}$ on Cora. Blue line: $\Lambda_{FD}$ values of R-GMM-VGAE, during training of R-GMM-VGAE. Green line: $\Lambda_{FD}$ values of GMM-VGAE, during training of R-GMM-VGAE. Gold line: $\Lambda_{FD}$ values of R-GMM-VGAE during training of GMM-VGAE. Red line: $\Lambda_{FD}$ values of GMM-VGAE, during training of GMM-VGAE. Purple line: normalized cumulative difference between $\Lambda_{FD}$ values of R-GMM-VGAE and $\Lambda_{FD}$ values of GMM-VGAE.}
  \label{fig:FD_R-GMM-VGAE}
\end{figure*}

\begin{table*}[!h]
  \caption{Correction-style mechanism against FR vs. protection-style mechanism against FR for R-GMM-VGAE and R-DGAE on Cora.}
  \begin{center}
  \begin{small}
  \begin{tabular}{|p{2.2cm}|c|c|c|c|c|c|c|c|c|c|c|c|}
  \hline
    {Method} & \multicolumn{2}{c|}{\textbf{Protection}} & \multicolumn{10}{c|}{\textbf{Correction}} \\
    \cline{2-13}
    {\textbf{}} & \multicolumn{2}{c|}{No delay} & \multicolumn{2}{c|}{After 10 epochs}  & \multicolumn{2}{c|}{After 30 epochs} & \multicolumn{2}{c|}{After 50 epochs} & \multicolumn{2}{c|}{After 100 epochs} & \multicolumn{2}{c|}{After 150 epochs} \\
    \cline{2-13}
    & ACC & NMI & ACC & NMI & ACC & NMI  & ACC & NMI & ACC & NMI & ACC & NMI \\ \hline
    \textbf{R-GMM-VGAE} & 76.7 & 57.3 & 74.5 & 53.9 & 73.6 & 54.8 & 70.4 & 51.9 & 71.6 & 52.7 & 70.1 & 51.0 \\ \hline
    \textbf{R-DGAE} & 73.7 & 56.0 & 71.1 & 52.0 & 70.4 & 50.5 & 69.8 & 50.1 & 69.6 & 50.0 & 69.7 & 49.6 \\ \hline
  \end{tabular}
  \end{small}
  \end{center}
  \vskip -0.1in
  \label{Table:FR_protection_vs_correction}
\end{table*}

\begin{table*}[!h]
  \caption{Protection-style mechanism against FD vs. correction-style mechanism against FD for R-GMM-VGAE and R-DGAE on Cora.}
  \begin{center}
  \begin{small}
  \begin{tabular}{|p{2.45cm}|c|c|c|c|c|c|}
  \hline
    {Method} & \multicolumn{3}{c|}{\textbf{Protection}} & \multicolumn{3}{c|}{\textbf{Correction}} \\
    \cline{2-7}
    & ACC & NMI & ARI & ACC & NMI & ARI  \\ \hline
    \textbf{R-GMM-VGAE} & 73.4 & 52.1 & 51.6 & 76.7 & 57.3 & 57.9  \\ \hline
    \textbf{R-DGAE} & 71.3 & 54.5 & 50.4 & 73.7 & 56.0 & 54.1 \\ \hline 
  \end{tabular}
  \end{small}
  \end{center}
  \vskip -0.1in
  \label{Table:FD_protection_vs_correction}
\end{table*}

\begin{table*}[!h]
  \caption{Performance of R-GMM-VGAE and R-DGAE on Cora, after ablation of the confidence thresholds $\alpha_{1}$ and $\alpha_{2}$. }
  \begin{center}
  \begin{small}
  \begin{tabular}{|p{2.45cm}|c|c|c|c|c|c|c|c|c|c|c|c|}
    \hline
    {Method} & \multicolumn{3}{c|}{\textbf{Ablation of $\alpha_{2}$}} & \multicolumn{3}{c|}{\textbf{Ablation of $\alpha_{1}$}} & \multicolumn{3}{c|}{\textbf{Ablation of both}} & \multicolumn{3}{c|}{\textbf{No Ablation}} \\
    \cline{2-13}
    & ACC & NMI & ARI & ACC & NMI & ARI & ACC & NMI & ARI & ACC & NMI & ARI \\ \hline
    \textbf{R-GMM-VGAE} & 74.2 & 53.7 & 53.7 & 73.3 & 52.0 & 51.8 & 71.2 & 52.5 & 48.3 & 76.7 & 57.3 & 57.9 \\ \hline
    \textbf{R-DGAE} & 72.7 & 55.1 & 52.2 & 72.8 & 54.6 & 52.2 & 70.5 & 50.4 & 47.7 & 73.7 & 56.0 & 54.1 \\ \hline
  \end{tabular}
  \end{small}
  \end{center}
  \label{Table:ablation_confidence}
\end{table*}

\begin{table*}[!h]
  \caption{Performance of R-GMM-VGAE and R-DGAE on Cora, after ablation of ``drop\_edge" and ``add\_edge" operations.}
  \begin{center}
  \begin{small}
  \begin{tabular}{|p{2.45cm}|c|c|c|c|c|c|c|c|c|c|c|c|}
    \hline                               
    {Method} & \multicolumn{3}{c|}{\textbf{Ablation of ``drop\_edge"}} & \multicolumn{3}{c|}{\textbf{Ablation of ``add\_edge"}} & \multicolumn{3}{c|}{\textbf{Ablation of both}} & \multicolumn{3}{c|}{\textbf{No Ablation}}  \\
    \cline{2-13}
    & ACC & NMI & ARI & ACC & NMI & ARI & ACC & NMI & ARI & ACC & NMI & ARI \\ \hline
    \textbf{R-GMM-VGAE} & 75.4 & 55.1 & 55.6 & 72.8 & 52.5 & 50.4 & 74.0 & 53.6 & 52.8 & 76.7 & 57.3 & 57.9  \\ \hline
    \textbf{R-DGAE} & 72.4 & 54.6 & 52.5 & 72.5 & 54.4 & 51.9 & 71.7 & 53.5 & 50.5 & 73.7 & 56.0 & 54.1 \\ \hline
  \end{tabular}
  \end{small}
  \end{center}
  \label{Table:ablation_edges}
\end{table*}

\textbf{Feature Drift:} In this part, we discuss the evolution of $\Lambda_{FD}$ values for GMM-VGAE and R-GMM-VGAE on Cora. The cosine similarity between the gradient of $L_{bce}(\hat{A}(Z(\theta)), \, A)$ and the gradient of $L_{bce}(\hat{A}(Z(\theta)), \, \Upsilon(A, Q'(Z(\theta)), \mathcal{V}))$ is denoted by $\Lambda_{FD} \text{(GMM-VGAE)}$, whereas $\Lambda_{FD}\text{(R-GMM-VGAE)}$ denotes the cosine similarity between the gradient of $L_{bce}(\hat{A}(Z(\theta)), \, \Upsilon(A, P(\Xi(Z(\theta))), \Omega))$ and the gradient of $L_{bce}(\hat{A}(Z(\theta)), \, \Upsilon(A, Q'(Z(\theta)), \mathcal{V}))$. We illustrate both metrics, during training of R-GMM-VAGE and GMM-VGAE, in Figures \ref{fig:FD_R-GMM-VGAE} (a) and (b), respectively. To facilitate our analysis, we also provide the normalized cumulative difference between $\Lambda_{FD}\text{(R-GMM-VGAE)}$ and $\Lambda_{FD} \text{(GMM-VGAE)}$, during training of R-GMM-VAGE and GMM-VGAE, in Figures \ref{fig:FR_R-GMM-VGAE} (d) and (e), respectively. As a general observation from Figures \ref{fig:FD_R-GMM-VGAE} (a), (b), and (c), $\Lambda_{FD} \text{(R-GMM-VGAE)}$ and $\Lambda_{FD}$ (GMM-VGAE) start from very high values (close to one) then gradually decrease. This implies that the unsupervised gradient, at an early training stage, has the same direction as the supervised one. A recent body of work \cite{paper52, paper46} has shown that a neural network learns simple patterns first using the early layers. In another work, the authors of \cite{paper95} have shown that these simple patterns can be learned through self-supervision just as well as through real supervision (with ground-truth labels). Thus, optimizing a supervised objective function has the same effect (learning low-level patterns) as optimizing a self-supervised objective function for the first few iterations.

For the first experiment (Figures \ref{fig:FD_R-GMM-VGAE} (a) and (d)), we train R-GMM-VGAE and we report $\Lambda_{FD} \text{(GMM-VGAE)}$, $\Lambda_{FD} \text{(R-GMM-VGAE)}$, and the normalized cumulative difference between both of them. We can see that there are two stages. The first stage ranges from iteration 0 to 40, and the second stage ranges from iteration 40 to 140. For the first stage, we observe that $\Lambda_{FD} \text{(R-GMM-VGAE)}$ values are very close to $\Lambda_{FD} \text{(GMM-VGAE)}$ values. A possible explanation is that $\Upsilon$ can only affect a small part of the self-supervisory graph $A_{clus}^{self}$ at the beginning, and most of the graph remains identical to $A$. Furthermore, we observe that $\Lambda_{FD} \text{(R-GMM-VGAE)}$ is decreasing rapidly for this stage. This aspect is desirable. In fact, $\Upsilon$ allows FD to occur at the beginning to counter random projections. From Figure \ref{fig:FD_R-GMM-VGAE} (d), we can see that the cumulative difference between $\Lambda_{FD} \text{(R-GMM-VGAE)}$ and $\Lambda_{FD} \text{(GMM-VGAE)}$ has a low slope for the first stage. This result confirms that our operator $\Upsilon$ allows FD to take place, during the first stage. For the second stage, we observe that $\Lambda_{FD} \text{(R-GMM-VGAE)}$ is increasing slowly between iterations 40 and 60. After allowing FD to occur, during the first stage, $\Upsilon$ gradually attenuates this problem during the second stage. From Figure \ref{fig:FD_R-GMM-VGAE} (d), we can see that the cumulative difference between $\Lambda_{FD} \text{(R-GMM-VGAE)}$ and $\Lambda_{FD} \text{GMM-VGAE)}$ has a pronounced increasing tendency compared with the first phase. After allowing FD to occur, $\Upsilon$ gradually attenuates this problem, during the second stage.

For the second experiment (Figures \ref{fig:FD_R-GMM-VGAE} (b) and (e)), we train GMM-VGAE and we report $\Lambda_{FD} \text{(GMM-VGAE)}$, $\Lambda_{FD} \text{(R-GMM-VGAE)}$, and the normalized cumulative difference between both of them. From Figure \ref{fig:FD_R-GMM-VGAE} (e), we can see that the cumulative difference between $\Lambda_{FD} \text{(R-GMM-VGAE)}$ and $\Lambda_{FD} \text{(GMM-VGAE)}$ has a pronounced increasing tendency starting from iteration 40. This result suggests that $\Upsilon$ can consistently construct a reliable self-supervisory signal even after learning based on unreliable nodes. Additionally, we observe a decreasing tendency of $\Lambda_{FD}$ between iterations 0 and 100. After 100 iterations, the two curves of $\Lambda_{FD}$ in Figure \ref{fig:FD_R-GMM-VGAE} (b) oscillate around a horizontal line (indicating the stability of FD). The absence of a considerable time slot, where $\Lambda_{FD} \text{(GMM-VGAE)}$ achieves a clear increasing tendency, suggests that GMM-VGAE does not have any implicit or explicit mechanism to reduce FD. Based on the same experiment, we can see that $\Lambda_{FD} \text{(GMM-VGAE)}$ can reach very low values compared with $\Lambda_{FR} \text{(GMM-VGAE)}$ (see Figure \ref{fig:FR_R-GMM-VGAE} (b)). In addition to that, we observe that $\Lambda_{FD} \text{(GMM-VGAE)}$ has more pronounced fluctuations than $\Lambda_{FR} \text{(GMM-VGAE)}$. While GMM-VGAE does not have any explicit mechanism against FR or FD, the reconstruction loss is an implicit mechanism against FR.

For the third experiment (Figures \ref{fig:FD_R-GMM-VGAE} (c) and (f)), we train GMM-VGAE and report $\Lambda_{FD} \text{(GMM-VGAE)}$, we train R-GMM-VGAE and report $\Lambda_{FD} \text{(R-GMM-VGAE)}$, and we finally report the normalized cumulative difference between both of them. We observe that R-GMM-VGAE considerably outperforms GMM-VGAE in terms of $\Lambda_{FD}$. More interestingly, while R-GMM-VGAE can attenuate FD after the initial decrease of $\Lambda_{FD}$, GMM-VGAE falls short of this capacity. 

\textbf{Protection vs correction:} In Table \ref{Table:FR_protection_vs_correction}, we compare between a protection mechanism and a correction mechanism against FR, during the training of R-GMM-VGAE and R-DGAE on Cora. A protection mechanism is established by initiating the sampling technique directly after the pretraining phase. For the correction case, we delay the sampling technique for different epochs (10, 30, 50, 100, and 150) to allow FR to occur. This experiment aims to test if a correction mechanism can reverse the effect of labels' randomness. As we can see from Table \ref{Table:FR_protection_vs_correction}, the protection strategy yields better results than the correction approaches for both models. Moreover, further delay of correction is generally associated with lower clustering performance. These results show that a correction mechanism can not reverse the effect of labels' randomness. In Table \ref{Table:FD_protection_vs_correction}, we compare between a protection mechanism and a correction mechanism against FD, during the training of R-GMM-VGAE and R-DGAE on Cora. A protection mechanism is established by transforming the self-supervisory signal $A$ into a clustering-oriented signal $\Upsilon(A, P(Z(\theta)), \mathcal{V})$, in a single step. This is done by applying $\Upsilon$ to the whole set of nodes $\mathcal{V}$, instead of $\Omega$, to eliminate the reconstruction. We observe that the correction strategy yields better results than the protection approach for both models. We conclude that a correction mechanism, which allows FD to take place then gradually attenuates this problem, is a more advantageous solution.

\textbf{One confidence threshold vs two confidence thresholds:} In this part, we perform an ablation study to investigate the performance overhead provided by $\Xi$. Our investigation includes four cases: ablation of the sampling criteria related to $\alpha_{1}$, ablation of the sampling criteria related to $\alpha_{2}$, ablation of both (i.e., eliminating the operator $\Xi$), and no ablation. As shown in Table \ref{Table:ablation_confidence}, the obtained results show the importance of using two criteria for selecting reliable nodes. Specifically, we observe that ablating the requirement related to $\alpha_{2}$ leads to a degradation in performance. In fact, $\alpha_{2}$ helps in excluding points, which are situated near the borderline of two similar clusters. 

\textbf{Adding edges vs dropping edges:} In this part, we perform an ablation study to investigate the performance contribution of $\Upsilon$. Our investigation includes four cases: ablation of ``drop\_edge", ablation of ``add\_edge", ablation of both (i.e., eliminating the operator $\Upsilon$), no ablation. As shown in Table \ref{Table:ablation_edges}, the obtained results show the importance of ``add\_edge" and ``drop\_edge" operations for building a reliable self-supervisory signal $A_{clus}^{self}$. 

\section{Conclusion}
In this manuscript, we advocate a new vision for building GAE-based clustering models from the perspective of Feature Randomness and Feature Drift. We start by introducing a new conceptual design that gradually reduces Feature Drift without causing an abrupt rise in random features. Our strategy depends on two operators. In this regard, we design a sampling function $\Xi$ that triggers a protection mechanism against random projections. Moreover, we propose a function $\Upsilon$ that triggers a correction mechanism against Feature Drift. As a key advantage, $\Xi$ and $\Upsilon$ can be easily tailored to existing GAE-based clustering models. Experiments on standard benchmarks demonstrate that our operators improve the clustering performance. Furthermore, our results show that: (1) $\Xi$ effectively delays the impact of Feature Randomness, and (2) $\Upsilon$ allows Feature Drift to occur then gradually reduces this problem. Our operators can be viewed as the first initiative to control Feature Randomness and Feature Drift for GAE-based clustering models. For future work, we plan to investigate the extensibility of our operators to multiplex graphs, in which each couple of nodes can be connected by multiple edges.

\ifCLASSOPTIONcaptionsoff
  \newpage
\fi


\bibliographystyle{IEEEtran}
\bibliography{RGAE_arXiv.bbl}

\begin{thebibliography}{10}
\providecommand{\url}[1]{#1}
\csname url@samestyle\endcsname
\providecommand{\newblock}{\relax}
\providecommand{\bibinfo}[2]{#2}
\providecommand{\BIBentrySTDinterwordspacing}{\spaceskip=0pt\relax}
\providecommand{\BIBentryALTinterwordstretchfactor}{4}
\providecommand{\BIBentryALTinterwordspacing}{\spaceskip=\fontdimen2\font plus
\BIBentryALTinterwordstretchfactor\fontdimen3\font minus
  \fontdimen4\font\relax}
\providecommand{\BIBforeignlanguage}[2]{{%
\expandafter\ifx\csname l@#1\endcsname\relax
\typeout{** WARNING: IEEEtran.bst: No hyphenation pattern has been}%
\typeout{** loaded for the language `#1'. Using the pattern for}%
\typeout{** the default language instead.}%
\else
\language=\csname l@#1\endcsname
\fi
#2}}
\providecommand{\BIBdecl}{\relax}
\BIBdecl

\bibitem{paper91}
D.~Bouzas, N.~Arvanitopoulos, and A.~Tefas, ``Graph embedded nonparametric
  mutual information for supervised dimensionality reduction,'' \emph{IEEE
  TNNLS}, vol.~26, no.~5, pp. 951--963, 2015.

\bibitem{paper30}
J.~Tang, M.~Qu, M.~Wang, M.~Zhang, J.~Yan, and Q.~Mei, ``Line: Large-scale
  information network embedding.'' in \emph{WWW}, 2015, pp. 1067--1077.

\bibitem{paper115}
W.~Wu, Y.~Jia, S.~Kwong, and J.~Hou, ``Pairwise constraint propagation-induced
  symmetric nonnegative matrix factorization,'' \emph{IEEE TNNLS}, vol.~29,
  no.~12, pp. 6348--6361, 2018.

\bibitem{paper12}
A.~Grover and J.~Leskovec, ``node2vec: Scalable feature learning for
  networks.'' in \emph{SIGKDD}, 2016, pp. 855--864.

\bibitem{paper92}
Z.~Wu, S.~Pan, F.~Chen, G.~Long, C.~Zhang, and S.~Y. Philip, ``A comprehensive
  survey on graph neural networks,'' \emph{IEEE TNNLS}, 2020.

\bibitem{paper114}
B.~Jiang, L.~Wang, J.~Cheng, J.~Tang, and B.~Luo, ``Gpens: Graph data learning
  with graph propagation-embedding networks,'' \emph{IEEE TNNLS}, pp. 1--14,
  2021.

\bibitem{paper9}
T.~N. Kipf and M.~Welling, ``Semi-supervised classification with graph
  convolutional networks.'' in \emph{ICLR}, 2017.

\bibitem{paper11}
------, ``Variational graph auto-encoders.'' in \emph{NeurIPS workshop}, 2016,
  pp. 1--3.

\bibitem{paper25}
S.~Pan, R.~Hu, G.~Long, J.~Jiang, L.~Yao, and C.~Zhang, ``Adversarially
  regularized graph autoencoder for graph embedding.'' in \emph{IJCAI}, 2018,
  pp. 2609--2615.

\bibitem{paper26}
B.~Hui, P.~Zhu, and Q.~Hu, ``Collaborative graph convolutional networks:
  Unsupervised learning meets semi-supervised learning.'' in \emph{AAAI}, 2020,
  pp. 4215--4222.

\bibitem{paper27}
C.~Wang, S.~Pan, R.~Hu, G.~Long, J.~Jiang, and C.~Zhang, ``Attributed graph
  clustering: A deep attentional embedding approach.'' in \emph{IJCAI}, 2019,
  pp. 2609--2615.

\bibitem{paper17}
X.~Guo, L.~Gao, X.~Liu, and J.~Yin, ``Improved deep embedded clustering with
  local structure preservation.'' in \emph{IJCAI}, 2017, pp. 1753--1759.

\bibitem{paper14}
\BIBentryALTinterwordspacing
N.~Mrabah, M.~Bouguessa, and R.~Ksantini, ``Adversarial deep embedded
  clustering: on a better trade-off between feature randomness and feature
  drift.'' \emph{IEEE TKDE}, 2020. [Online]. Available:
  \url{10.1109/TKDE.2020.2997772}
\BIBentrySTDinterwordspacing

\bibitem{paper64}
H.~Maennel, I.~M. Alabdulmohsin, I.~O. Tolstikhin, R.~Baldock, O.~Bousquet,
  S.~Gelly, and D.~Keysers, ``What do neural networks learn when trained with
  random labels?'' \emph{NeurIPS}, 2020.

\bibitem{paper65}
J.~Frankle, D.~J. Schwab, and A.~S. Morcos, ``The early phase of neural network
  training,'' in \emph{ICLR}, 2020.

\bibitem{paper80}
C.~Wang, S.~Pan, G.~Long, X.~Zhu, and J.~Jiang, ``Mgae: Marginalized graph
  autoencoder for graph clustering,'' in \emph{CIKM}, 2017, pp. 889--898.

\bibitem{paper47}
A.~Ansuini, A.~Laio, J.~H. Macke, and D.~Zoccolan, ``Intrinsic dimension of
  data representations in deep neural networks,'' in \emph{NeurIPS}, 2019, pp.
  6111--6122.

\bibitem{paper77}
H.~W. Kuhn, ``The hungarian method for the assignment problem,'' \emph{Naval
  research logistics quarterly}, vol.~2, no. 1-2, pp. 83--97, 1955.

\bibitem{paper79}
C.~Zhang, S.~Bengio, M.~Hardt, B.~Recht, and O.~Vinyals, ``Understanding deep
  learning requires rethinking generalization,'' in \emph{ICLR}, 2017.

\bibitem{paper50}
N.~S. Keskar, D.~Mudigere, J.~Nocedal, M.~Smelyanskiy, and P.~T.~P. Tang, ``On
  large-batch training for deep learning: Generalization gap and sharp
  minima,'' in \emph{ICLR}, 2017.

\bibitem{paper48}
X.~Ma, Y.~Wang, M.~E. Houle, S.~Zhou, S.~Erfani, S.~Xia, S.~Wijewickrema, and
  J.~Bailey, ``Dimensionality-driven learning with noisy labels,'' in
  \emph{ICML}, 2018, pp. 3355--3364.

\bibitem{paper82}
G.~Cui, J.~Zhou, C.~Yang, and Z.~Liu, ``Adaptive graph encoder for attributed
  graph embedding,'' in \emph{SIGKDD}, 2020, pp. 976--985.

\bibitem{paper113}
Y.~You, T.~Chen, Z.~Wang, and Y.~Shen, ``When does self-supervision help graph
  convolutional networks?'' in \emph{ICML}, 2020, pp. 10\,871--10\,880.

\bibitem{paper87}
M.~Arjovsky, S.~Chintala, and L.~Bottou, ``Wasserstein generative adversarial
  networks,'' in \emph{ICML}, 2017, pp. 214--223.

\bibitem{paper89}
M.~Arjovsky and L.~Bottou, ``Towards principled methods for training generative
  adversarial networks,'' \emph{ICLR}, 2017.

\bibitem{paper13}
N.~Mrabah, N.~M. Khan, R.~Ksantini, and Z.~Lachiri, ``Deep clustering with a
  dynamic autoencoder: From reconstruction towards centroids construction.''
  \emph{Neural Networks}, vol. 130, pp. 206--228, 2020.

\bibitem{paper88}
M.~Caron, P.~Bojanowski, A.~Joulin, and M.~Douze, ``Deep clustering for
  unsupervised learning of visual features,'' in \emph{ECCV}, 2018, pp.
  132--149.

\bibitem{paper96}
X.~Zhu, Z.~Ghahramani, and J.~D. Lafferty, ``Semi-supervised learning using
  gaussian fields and harmonic functions,'' in \emph{ICML}, 2003, pp. 912--919.

\bibitem{paper97}
D.~Zhou, O.~Bousquet, T.~N. Lal, J.~Weston, and B.~Sch{\"o}lkopf, ``Learning
  with local and global consistency,'' in \emph{NeurIPS}, 2004, pp. 321--328.

\bibitem{paper98}
J.~Weston, F.~Ratle, H.~Mobahi, and R.~Collobert, ``Deep learning via
  semi-supervised embedding,'' in \emph{Neural networks: Tricks of the
  trade}.\hskip 1em plus 0.5em minus 0.4em\relax Springer, 2012, pp. 639--655.

\bibitem{paper100}
K.~He, H.~Fan, Y.~Wu, S.~Xie, and R.~Girshick, ``Momentum contrast for
  unsupervised visual representation learning,'' in \emph{CVPR}, 2020, pp.
  9729--9738.

\bibitem{paper101}
P.~Bachman, R.~D. Hjelm, and W.~Buchwalter, ``Learning representations by
  maximizing mutual information across views,'' in \emph{NeurIPS}, 2019, pp.
  15\,509--15\,519.

\bibitem{paper102}
J.-B. Grill, F.~Strub, F.~Altch{\'e}, C.~Tallec, P.~H. Richemond,
  E.~Buchatskaya, C.~Doersch, B.~A. Pires, Z.~D. Guo, M.~G. Azar \emph{et~al.},
  ``Bootstrap your own latent: A new approach to self-supervised learning,'' in
  \emph{NeurIPS}, 2020.

\bibitem{paper106}
M.~Cisse, P.~Bojanowski, E.~Grave, Y.~Dauphin, and N.~Usunier, ``Parseval
  networks: Improving robustness to adversarial examples,'' in \emph{ICML},
  2017, pp. 854--863.

\bibitem{paper104}
B.~Neyshabur, S.~Bhojanapalli, D.~Mcallester, and N.~Srebro, ``Exploring
  generalization in deep learning,'' in \emph{NeurIPS}, vol.~30, 2017, pp.
  5947--5956.

\bibitem{paper105}
B.~Neyshabur, S.~Bhojanapalli, and N.~Srebro, ``A pac-bayesian approach to
  spectrally-normalized margin bounds for neural networks,'' in \emph{ICLR},
  2018.

\bibitem{paper103}
J.~Sokoli{\'c}, R.~Giryes, G.~Sapiro, and M.~R. Rodrigues, ``Robust large
  margin deep neural networks,'' \emph{IEEE TSP}, vol.~65, no.~16, pp.
  4265--4280, 2017.

\bibitem{paper107}
H.~Yang, K.~Ma, and J.~Cheng, ``Rethinking graph regularization for graph
  neural networks,'' in \emph{AAAI}, 2021, pp. 4573--4581.

\bibitem{paper56}
L.~Jiang, Z.~Zhou, T.~Leung, L.-J. Li, and L.~Fei-Fei, ``Mentornet: Learning
  data-driven curriculum for very deep neural networks on corrupted labels,''
  in \emph{ICML}, 2018, pp. 2304--2313.

\bibitem{paper57}
E.~Malach and S.~Shalev-Shwartz, ``Decoupling 'when to update' from 'how to
  update','' in \emph{NeurIPS}, 2017, pp. 960--970.

\bibitem{paper58}
B.~Han, Q.~Yao, X.~Yu, G.~Niu, M.~Xu, W.~Hu, I.~Tsang, and M.~Sugiyama,
  ``Co-teaching: Robust training of deep neural networks with extremely noisy
  labels,'' in \emph{NeurIPS}, 2018, pp. 8527--8537.

\bibitem{paper84}
P.~Sen, G.~Namata, M.~Bilgic, L.~Getoor, B.~Galligher, and T.~Eliassi-Rad,
  ``Collective classification in network data,'' \emph{AI magazine}, vol.~29,
  no.~3, pp. 93--93, 2008.

\bibitem{paper28}
L.~F. Ribeiro, P.~H. Saverese, and D.~R. Figueiredo, ``struc2vec: Learning node
  representations from structural identity.'' in \emph{SIGKDD}, 2017, pp.
  385--394.

\bibitem{paper85}
J.~Wu, J.~He, and J.~Xu, ``Net: Degree-specific graph neural networks for node
  and graph classification,'' in \emph{SIGKDD}, 2019, pp. 406--415.

\bibitem{paper52}
D.~Arpit, S.~Jastrz{\k{e}}bski, N.~Ballas, D.~Krueger, E.~Bengio, M.~S. Kanwal,
  T.~Maharaj, A.~Fischer, A.~Courville, Y.~Bengio \emph{et~al.}, ``A closer
  look at memorization in deep networks,'' in \emph{ICML}, 2017, pp. 233--242.

\bibitem{paper46}
E.~Facco, M.~d’Errico, A.~Rodriguez, and A.~Laio, ``Estimating the intrinsic
  dimension of datasets by a minimal neighborhood information,''
  \emph{Scientific reports}, vol.~7, no.~1, pp. 1--8, 2017.

\bibitem{paper95}
Y.~M. Asano, C.~Rupprecht, and A.~Vedaldi, ``A critical analysis of
  self-supervision, or what we can learn from a single image,'' in \emph{ICLR},
  2020.

\bibitem{paper110}
C.~Yang, Z.~Liu, D.~Zhao, M.~Sun, and E.~Chang, ``Network representation
  learning with rich text information,'' in \emph{IJCAI}, 2015.

\bibitem{paper81}
S.~Pan, R.~Hu, G.~Long, J.~Jiang, L.~Yao, and C.~Zhang, ``Adversarially
  regularized graph autoencoder for graph embedding.'' in \emph{IJCAI}, 2018,
  pp. 2609--2615.

\bibitem{paper112}
S.~Pan, R.~Hu, S.-f. Fung, G.~Long, J.~Jiang, and C.~Zhang, ``Learning graph
  embedding with adversarial training methods,'' \emph{IEEE transactions on
  cybernetics}, vol.~50, no.~6, pp. 2475--2487, 2019.

\bibitem{paper111}
P.~Velickovic, W.~Fedus, W.~L. Hamilton, P.~Li{\`o}, Y.~Bengio, and R.~D.
  Hjelm, ``Deep graph infomax.'' in \emph{ICLR}, 2019.

\bibitem{paper73}
X.~Zhang, H.~Liu, Q.~Li, and X.-M. Wu, ``Attributed graph clustering via
  adaptive graph convolution,'' in \emph{IJCAI}.\hskip 1em plus 0.5em minus
  0.4em\relax AAAI Press, 2019, pp. 4327--4333.

\end{thebibliography}



\clearpage

\appendices

\section{Hardware and Software Configurations} \label{appendix_A}
All experiments are conducted on a server under the same environment.

\subsection*{Hardware:}
\begin{itemize}
    \item Operating System: Ubuntu 18.04.5 LTS
    \item CPU: Intel(R) Xeon(R) CPU E5-2620 V4 @ 2.10GHz
\end{itemize}

\subsection*{Software:}
\begin{itemize}
    \item Python 3.7.4
    \item PyTorch 1.3.1
    \item Sklearn 0.23.1
\end{itemize}

\section{DGAE technical details} \label{appendix_B}
Due to the limited number of second group models, we propose a new approach abbreviated by DGAE (Discriminative Graph Auto-Encoder) from the second group. DGAE has a simple graph auto-encoder architecture with two GCN layers. We pretrain this model using vanilla reconstruction for $200$ epochs. For the clustering phase, DGAE minimizes a linear combination of clustering and reconstruction. The clustering loss of DGAE is the Kullback Leibler divergence between a soft clustering assignment distribution $P=(p_{ij})_{i,j}$ and its associated hard clustering assignment distribution $Q=(q_{ij})_{i,j}$ as described by Equation (\ref{eq:L_clus}):

\begin{equation} \label{eq:L_clus}
    L_{clus}(P(Z(\theta))) = KL(Q||P) = \sum_{i} \sum_{j}q_{ij} log(\frac{q_{ij}}{p_{ij}}).
\end{equation}

The soft clustering assignment $P$ is computed based on the Student’s t-distribution as follows:

\begin{equation} \label{eq:p_ij}
    p_{ij} = \frac{1 + \norm{z_{i} - \mu_{j}}^{2}}{\sum_{j'}(1 + \norm{z_{i} - \mu_{j'}}^{2})},
\end{equation}

where $\mu_{j}$ represent the clustering centers of the embedded representations $Z$. At the beginning of the training process, the embedded centers $\mu_{j}$ are initialized based on K-means. Then, DGAE is trained to jointly optimize the embedded representations and the clustering centers by minimizing $L_{clus}(P(Z(\theta))) + \gamma L_{bce}(\hat{A}(Z(\theta)), \,A)$. All settings of DGAE are described in Table \ref{Table:DGAE}.

\begin{table}[!h]
  \caption{Settings of DGAE.}
  \begin{center}
  \begin{small}
  \begin{tabular}{|p{5cm}|c|}
    \hline 
    {\textbf{Parameter}} & \textbf{Value} \\ \hline
    \textbf{Dimension of the first GCN layer} & 32 \\ \hline
    \textbf{Dimension of the second GCN layer} & 16 \\ \hline
    \textbf{Number of pretraining epochs} & 200 \\ \hline
    \textbf{Pretraining optimizer} & Adam \\ \hline
    \textbf{Learning rate for pretraining} & 0.01 \\ \hline
    \textbf{Number of training epochs} &  200 \\ \hline
    \textbf{Training optimizer} & Adam \\ \hline
    \textbf{Learning rate for training} & 0.01 \\ \hline
    \textbf{Balancing coefficient} & 0.001 \\ \hline
  \end{tabular}
  \end{small}
  \end{center}
  \vskip -0.1in
  \label{Table:DGAE}
\end{table}

\newpage

\section{Hyper-parameter settings} \label{appendix_C}
We report the hyper-parameter settings for R-GAE, R-VGAE, R-ARGAE, R-ARVGAE, R-GMM-VGAE, and R-DGAE on Cora in Table \ref{Table:hyperparameters_Cora}, on Citeseer in Table \ref{Table:hyperparameters_Citeseer}, on Pubmed in Table \ref{Table:hyperparameters_Pubmed}, on Brazil Air traffic in Table \ref{Table:hyperparameters_brasil}, on Europe Air traffic in Table \ref{Table:hyperparameters_europe}, and on USA Air traffic in Table \ref{Table:hyperparameters_usa}.

\begin{table}[!h]
  \caption{Hyper-parameter settings on Cora.}
  \begin{center}
  \begin{small}
  \begin{tabular}{|p{2.7cm}|c|c|c|}
    \hline
    {\textbf{Method}} & \textbf{$\alpha_{1}$} & \textbf{$M_{1}$} & \textbf{$M_{2}$}
     \\ \hline
    \textbf{R-GAE} & 0.3 & 20 epochs & 10 epochs  \\ \hline
    \textbf{R-VGAE} & 0.3 & 20 epochs & 10 epochs \\ \hline
    \textbf{R-ARGAE} & 0.3 & 50 epochs & 1 epoch \\ \hline
    \textbf{R-ARVGAE} & 0.3 & 50 epochs & 1 epoch \\ \hline
    \textbf{R-DGAE} & 0.3 & 20 epochs & 15 epochs \\ \hline
    \textbf{R-GMM-VGAE} & 0.3 & 20 epochs & 10 epochs \\ \hline
  \end{tabular}
  \end{small}
  \end{center}
  \vskip -0.1in
  \label{Table:hyperparameters_Cora}
\end{table}

\begin{table}[!h]
  \caption{Hyper-parameter settings on Citeseer.}
  \begin{center}
  \begin{small}
  \begin{tabular}{|p{2.7cm}|c|c|c|}
    \hline
    {\textbf{Method}} & \textbf{$\alpha_{1}$} & \textbf{$M_{1}$} & \textbf{$M_{2}$}
     \\ \hline
    \textbf{R-GAE} & 0.2 & 20 epochs & 10 epochs \\ \hline
    \textbf{R-VGAE} & 0.2 & 20 epochs & 1 epoch  \\ \hline
    \textbf{R-ARGAE} & 0.1 & 50 epochs & 1 epoch \\ \hline
    \textbf{R-ARVGAE} & 0.1 & 50 epochs & 1 epoch \\ \hline
    \textbf{R-DGAE} & 0.2 & 50 epochs & 1 epoch  \\ \hline
    \textbf{R-GMM-VGAE} & 0.2 & 50 epochs & 1 epoch \\ \hline
  \end{tabular}
  \end{small}
  \end{center}
  \vskip -0.1in
  \label{Table:hyperparameters_Citeseer}
\end{table}

\begin{table}[!h]
  \caption{Hyper-parameter settings on Pubmed.}
  \begin{center}
  \begin{small}
  \begin{tabular}{|p{2.7cm}|c|c|c|}
    \hline
    {\textbf{Method}} & \textbf{$\alpha_{1}$} & \textbf{$M_{1}$} & \textbf{$M_{2}$} \\ \hline
    \textbf{R-GAE} & 0.4 & 50 epochs & 5 epochs \\ \hline
    \textbf{R-VGAE} & 0.4 & 50 epochs & 5 epochs \\ \hline
    \textbf{R-ARGAE} & 0.3 & 50 epochs & 1 epoch \\ \hline
    \textbf{R-ARVGAE} & 0.3 & 50 epochs & 1 epoch \\ \hline
    \textbf{R-DGAE} & 0.3 & 50 epochs & 5 epochs \\ \hline
    \textbf{R-GMM-VGAE} & 0.4 & 50 epochs & 5 epochs \\ \hline
  \end{tabular}
  \end{small}
  \end{center}
  \vskip -0.1in
  \label{Table:hyperparameters_Pubmed}
\end{table}

\begin{table}[!h]
  \caption{Hyper-parameter settings on Brazil Air traffic.}
  \begin{center}
  \begin{small}
  \begin{tabular}{|p{2.7cm}|c|c|c|}
    \hline
    {\textbf{Method}} & \textbf{$\alpha_{1}$} & \textbf{$M_{1}$} & \textbf{$M_{2}$}
     \\ \hline
    \textbf{R-DGAE} & 0.25 & 50 epochs & 1 epoch \\ \hline
    \textbf{R-GMM-VGAE} & 0.25 & 50 epochs & 1 epoch \\ \hline
  \end{tabular}
  \end{small}
  \end{center}
  \vskip -0.1in
  \label{Table:hyperparameters_brasil}
\end{table}

\begin{table}[!h]
  \caption{Hyper-parameter settings on Europe Air traffic.}
  \begin{center}
  \begin{small}
  \begin{tabular}{|p{2.7cm}|c|c|c|}
    \hline
    {\textbf{Method}} & \textbf{$\alpha_{1}$} & \textbf{$M_{1}$} & \textbf{$M_{2}$}
     \\ \hline
    \textbf{R-DGAE} & 0.08 & 20 epochs & 15 epochs \\ \hline
    \textbf{R-GMM-VGAE} & 0.01 & 50 epochs & 1 epoch \\ \hline
  \end{tabular}
  \end{small}
  \end{center}
  \vskip -0.1in
  \label{Table:hyperparameters_europe}
\end{table}

\begin{table}[!h]
  \caption{Hyper-parameter settings on USA Air traffic.}
  \begin{center}
  \begin{small}
  \begin{tabular}{|p{2.7cm}|c|c|c|}
    \hline
    {\textbf{Method}} & \textbf{$\alpha_{1}$} & \textbf{$M_{1}$} & \textbf{$M_{2}$}
     \\ \hline
    \textbf{R-DGAE} & 0.1 & 50 epochs & 1 epoch \\ \hline
    \textbf{R-GMM-VGAE} & 0.3 & 50 epochs & 1 epoch \\ \hline
  \end{tabular}
  \end{small}
  \end{center}
  \vskip -0.1in
  \label{Table:hyperparameters_usa}
\end{table}

\newpage

\section{Further results} \label{appendix_J}
\textbf{Comparison with graph clustering methods:} 
In this part, we compare R-GMM-VGAE and R-DGAE with several recent graph clustering approaches. We report the paper results if the code is not publicly available. Otherwise, we run each experiment $10$ times and we report the best results among these trials. As we can see from Table \ref{Table:acc_nmi_ari_citation_networks_full}, R-DGAE and R-GMM-VGAE yield generally better results than the other methods. Although being competitive on Cora and Citeseer, our methods outperform AGE on Pubmed.  


\begin{table*}[!h]
  \caption{Clustering performance of several graph clustering methods on Cora, Citeseer, and Pubmed. Best in bold, second best underlined.}
  \begin{center}
  \begin{small}
  \begin{tabular}{|p{2.45cm}|c|c|c|c|c|c|c|c|c|c|}
    \hline
    {\textbf{Method}} & {\textbf{Input}} & \multicolumn{3}{c|}{\textbf{Cora}} & \multicolumn{3}{c|}{\textbf{Citeseer}} & \multicolumn{3}{c|}{\textbf{Pubmed}}  \\
    \cline{3-11}
    {\textbf{}} & {\textbf{}} & \textbf{ACC} & \textbf{NMI} & \textbf{ARI} & \textbf{ACC} & \textbf{NMI} & \textbf{ARI} & \textbf{ACC} & \textbf{NMI} & \textbf{ARI}  \\ 
    \hline
    \textbf{TADW \cite{paper110}} & C\&S & 53.6 & 36.6 & 24.0 & 52.9 & 32.0 & 28.6 & 56.5 & 22.4 & 17.7 \\ \hline
    \textbf{GAE \cite{paper11}} & C\&S & 61.3 & 44.4 & 38.1 & 48.2 & 22.7 & 19.2 & 63.2 & 24.9 & 24.6 \\ \hline
    \textbf{VGAE \cite{paper11}} & C\&S & 64.7 & 43.4 & 37.5 & 51.9 & 24.9 & 23.8 & 69.6 & 28.6 & 31.7 \\ \hline
    \textbf{MGAE \cite{paper80}} & C\&S & 68.1 & 48.9 & 43.6 & 66.9 & 41.6 & 42.5 & 59.3 & 28.2 & 24.8 \\ \hline
    \textbf{ARGE \cite{paper81}} & C\&S & 64.0 & 44.9 & 35.2 & 57.3 & 35.0 & 34.1 & 68.1 & 27.6 & 29.1 \\ \hline
    \textbf{ARVGE \cite{paper81}} & C\&S & 63.8 & 45.0 & 37.4 & 54.4 & 26.1 & 24.5 & 63.5 & 23.2 & 22.5 \\ \hline
    \textbf{ARGVA \cite{paper112}} & C\&S & 71.1 & 52.6 & 49.5 & 58.1 & 33.8 & 30.1 & 69.0 & 30.5 & 30.6 \\ \hline
    \textbf{DGI \cite{paper111}}  & C\&S & 71.3 & 56.4 & 51.1 & 68.8 & 44.4 & 45.0 & 53.3 & 18.1 & 16.6 \\ \hline
    \textbf{AGC \cite{paper73}} & C\&S & 68.9 & 53.7 & 48.6 & 67.0 & 41.1 & 41.9 & 69.8 & 31.6 & 31.9 \\ \hline
    \textbf{DAEGC \cite{paper27}} & C\&S & 70.4 & 52.8 & 49.6 & 67.2 & 39.7 & 41.0 & 67.1 & 26.6 & 27.8 \\ \hline
    \textbf{GMM-VGAE \cite{paper26}} & C\&S & 71.5 & 53.1 & 47.4 & 67.5 & 40.7 & 42.4 & 71.1 & 29.9 & 33.0 \\ \hline
    \textbf{AGE \cite{paper82}} & C\&S & \underline{76.1} & \textbf{59.9} & \underline{54.5} & \underline{70.1} & \underline{44.3} & \underline{45.4} & 70.9 & 30.8 & 32.9 \\ \hline
    \hline
    \textbf{R-DGAE} & C\&S & 73.7 & 56.0 & 54.1 & \textbf{70.5} & \textbf{45.0} & \textbf{47.1} & \underline{71.4} & \textbf{34.4} & \underline{34.6} \\  \hline
    \textbf{R-GMM-VGAE} & C\&S & \textbf{76.7} & \underline{57.3} & \textbf{57.9} & 68.9 & 42.0 & 43.9 & \textbf{74.0} & \underline{33.4} & \textbf{37.9} \\ \hline
  \end{tabular}
  \end{small}
  \end{center}
  \vskip -0.1in
  \label{Table:acc_nmi_ari_citation_networks_full}
\end{table*}

\textbf{Robustness:} In this part, we compare R-DGAE to DGAE on Cora after including or dropping edges or features. To establish a fair comparison, we ensure that the randomly included or dropped edges and features are the same for the two compared models. Moreover, we ensure that the evaluated models (i.e., DGAE and R-DGAE) share the same pretraining weights for each experiment. In Figure \ref{fig:robustness_noisy_edges_DGAE}, we present two experiments. For the first one, we randomly connect pairs of unlinked nodes, we run both models, and we report their ACCs and ARIs. We observe that R-DGAE consistently outperforms DGAE with a various number of noisy edges. As the training progresses, we can see that the gap between both models increases. These results can be explained by the ability of our operator $\Upsilon$ to drop random edges from the output graph. For the second experiment, we randomly add Gaussian noise with a mean value equal to zero, and a variance ranging in $\left [0, 0.2 \right]$. We observe that R-DGAE yields better results than DGAE with various amounts of noise. In fact, $\Xi$ only selects the most reliable samples. Therefore, the nodes, which are highly affected by noise, are probably ruled out by $\Xi$. In Figure \ref{fig:robustness_dropped_edges_DGAE}, we present two additional experiments. In the first experiment, we randomly drop pairs of linked nodes. We observe that R-DGAE consistently outperforms DGAE with a various number of dropped edges. Interestingly, dropping few edges (less than 400) does not harm the clustering performance of R-DGAE (it even induces a small improvement). However, this is not the case for DGAE. While DGAE reconstructs the corrupted input graph, R-DGAE is endowed with a correction mechanism $\Upsilon$ that can construct new clustering-friendly edges. For the second experiment, we randomly drop features columns from the matrix $X$. Similar to previous experiments, we observe that R-DGAE surpasses DGAE with various amounts of dropped features. A possible explanation suggests that $\Xi$ can exclude the nodes, which are highly affected by the randomly dropped features.

\begin{figure*}[!h]
  \begin{subfigure}[b]{0.24\textwidth}
    \includegraphics[width=\linewidth]{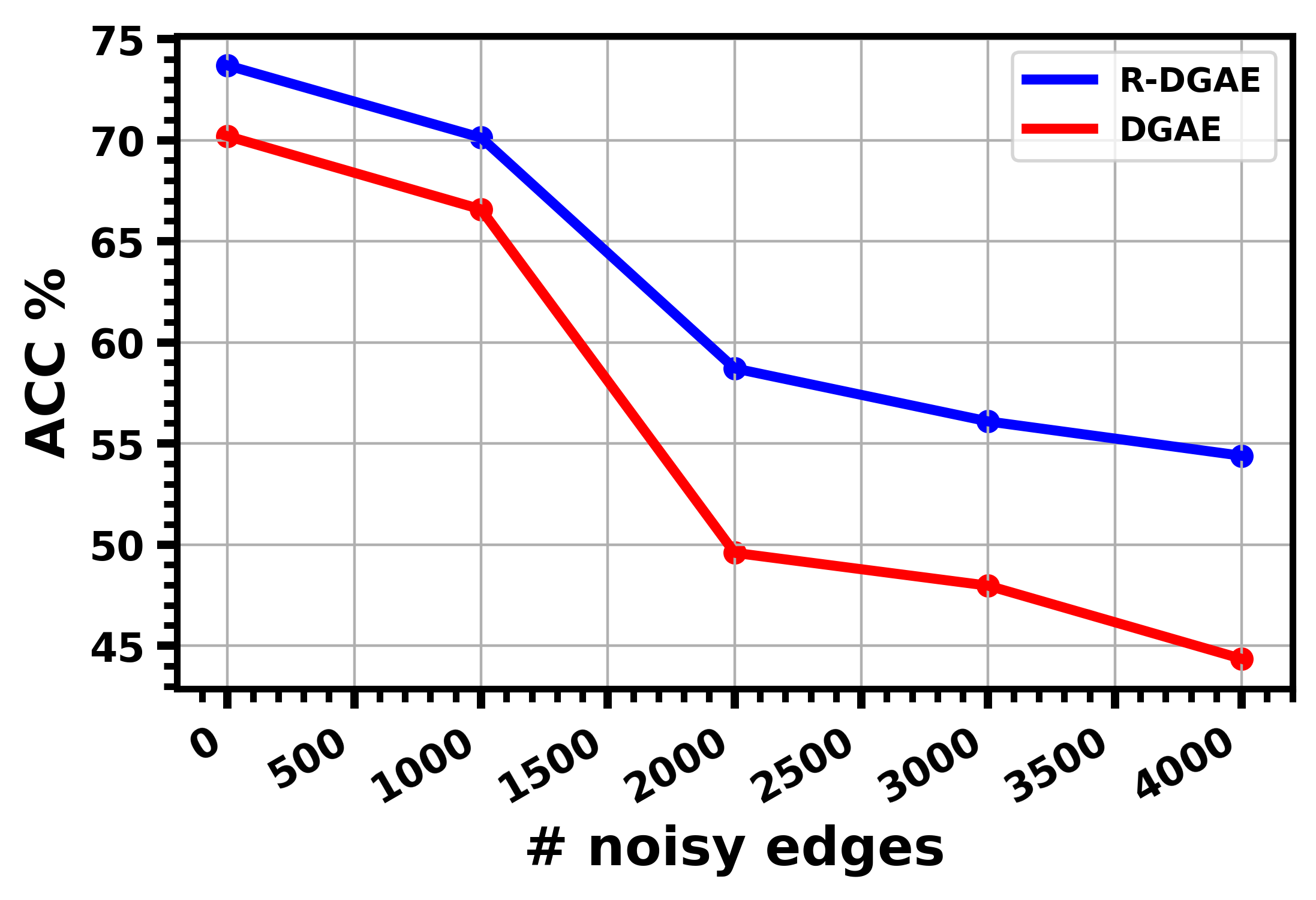}
  \end{subfigure}
  \begin{subfigure}[b]{0.24\textwidth}
     \includegraphics[width=\linewidth]{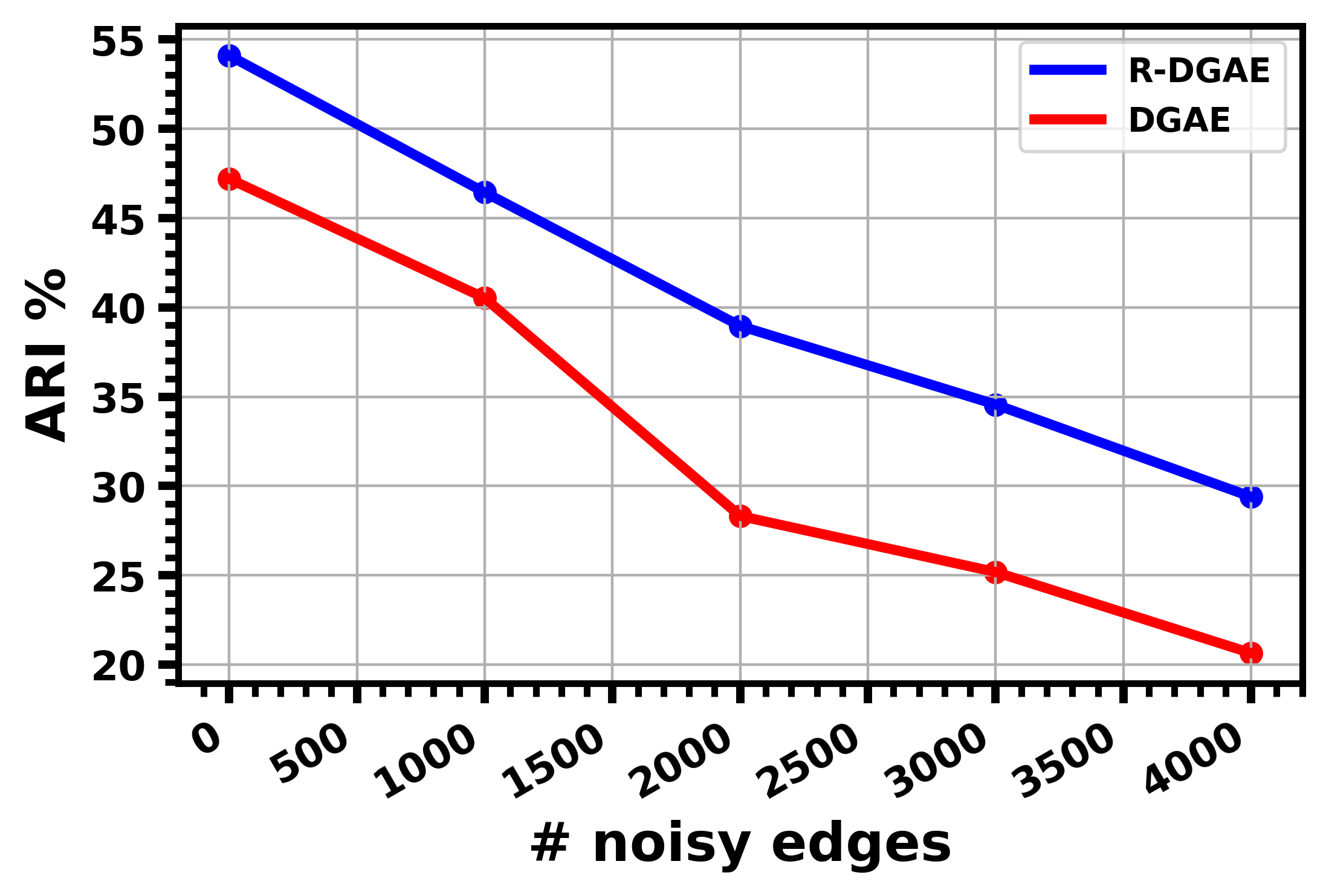}
  \end{subfigure}
  \begin{subfigure}[b]{0.24\textwidth}
     \includegraphics[width=\linewidth]{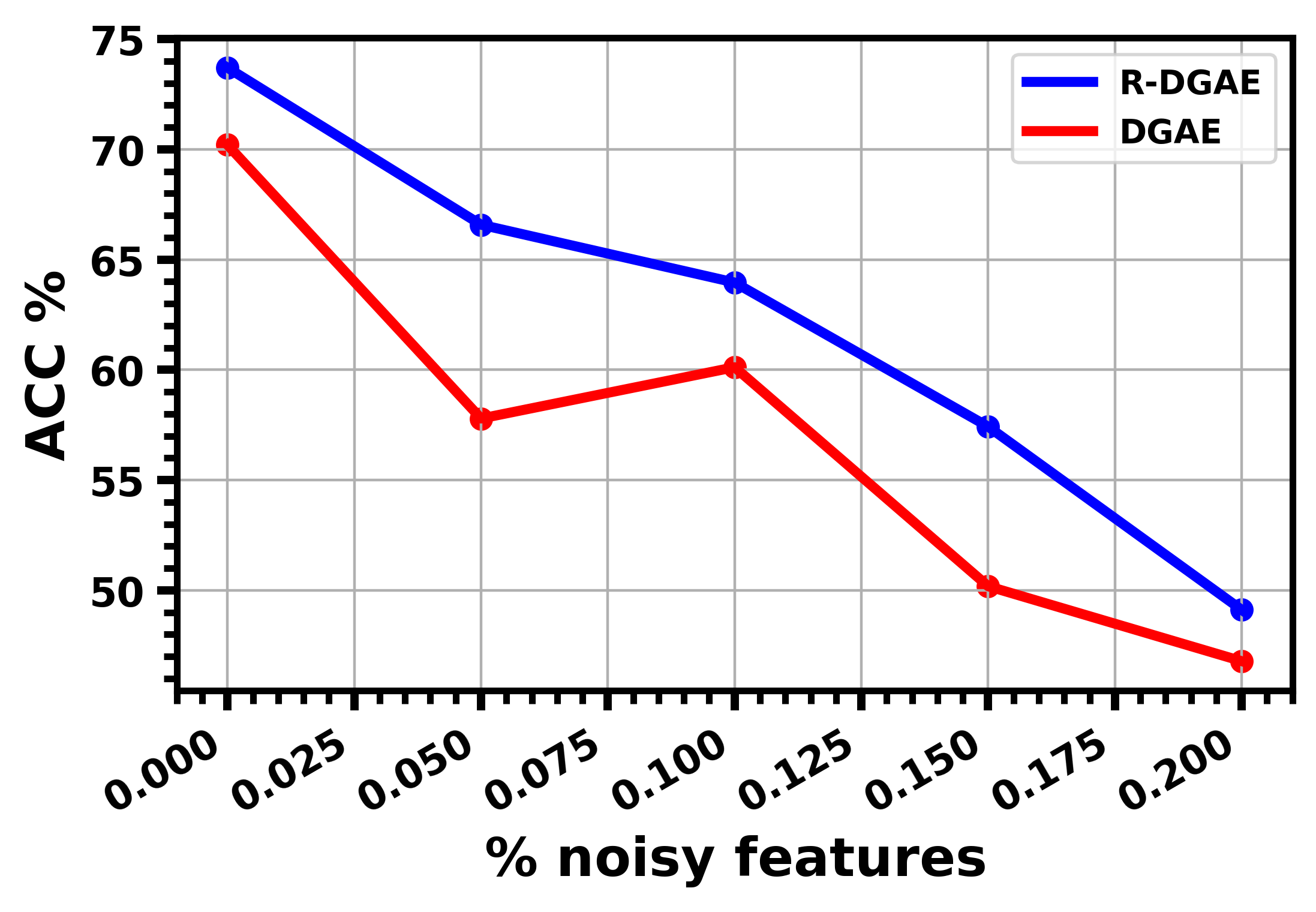}
  \end{subfigure}
  \begin{subfigure}[b]{0.24\textwidth}
     \includegraphics[width=\linewidth]{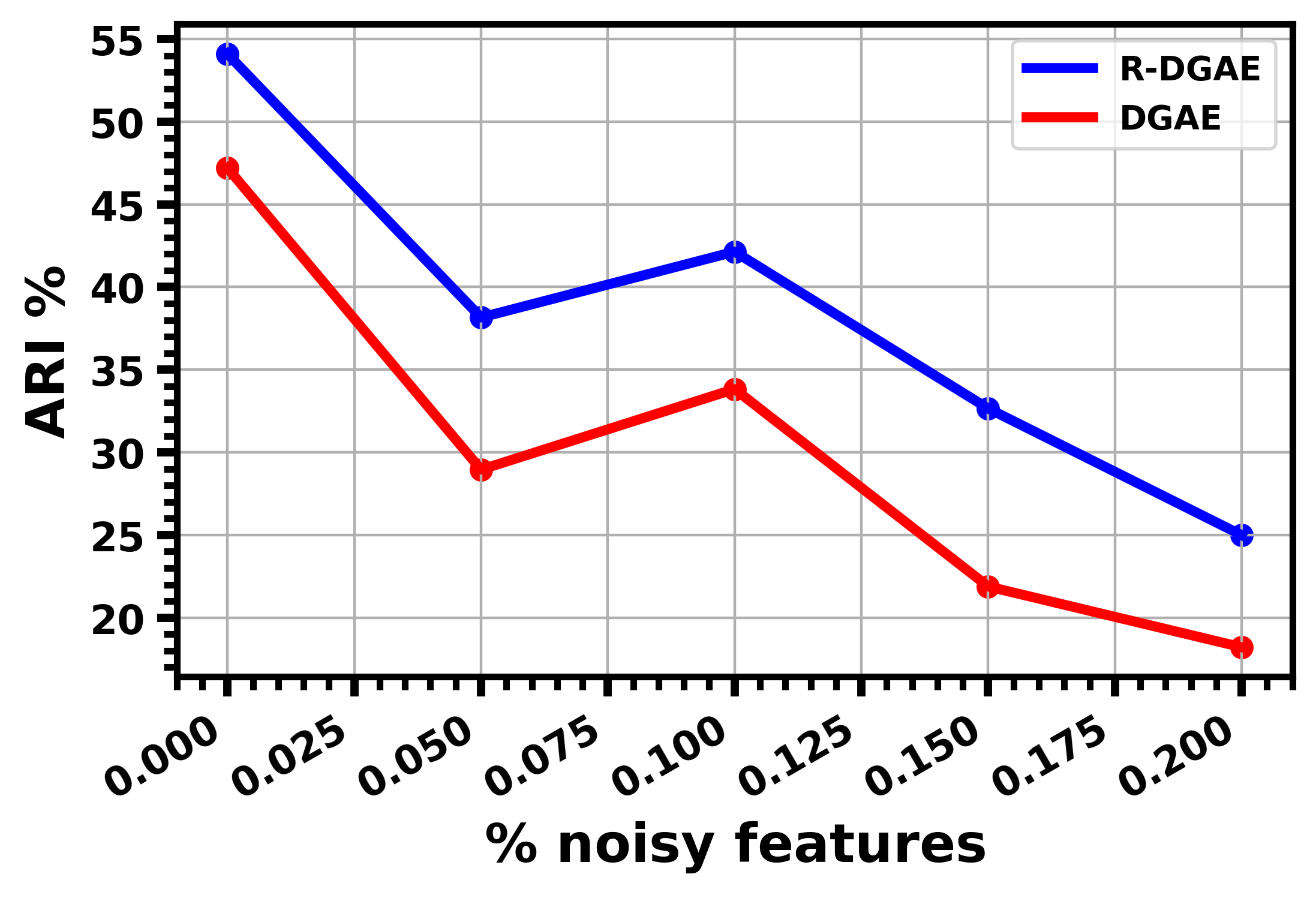}
  \end{subfigure}
  \vskip 0.1in
  \caption{Performance of R-DGAE and DGAE on Cora, in terms of ACC and ARI, after adding noisy edges and features.}
  \label{fig:robustness_noisy_edges_DGAE}
\end{figure*}

\begin{figure*}[!h]
  \begin{subfigure}[b]{0.24\textwidth}
    \includegraphics[width=\linewidth]{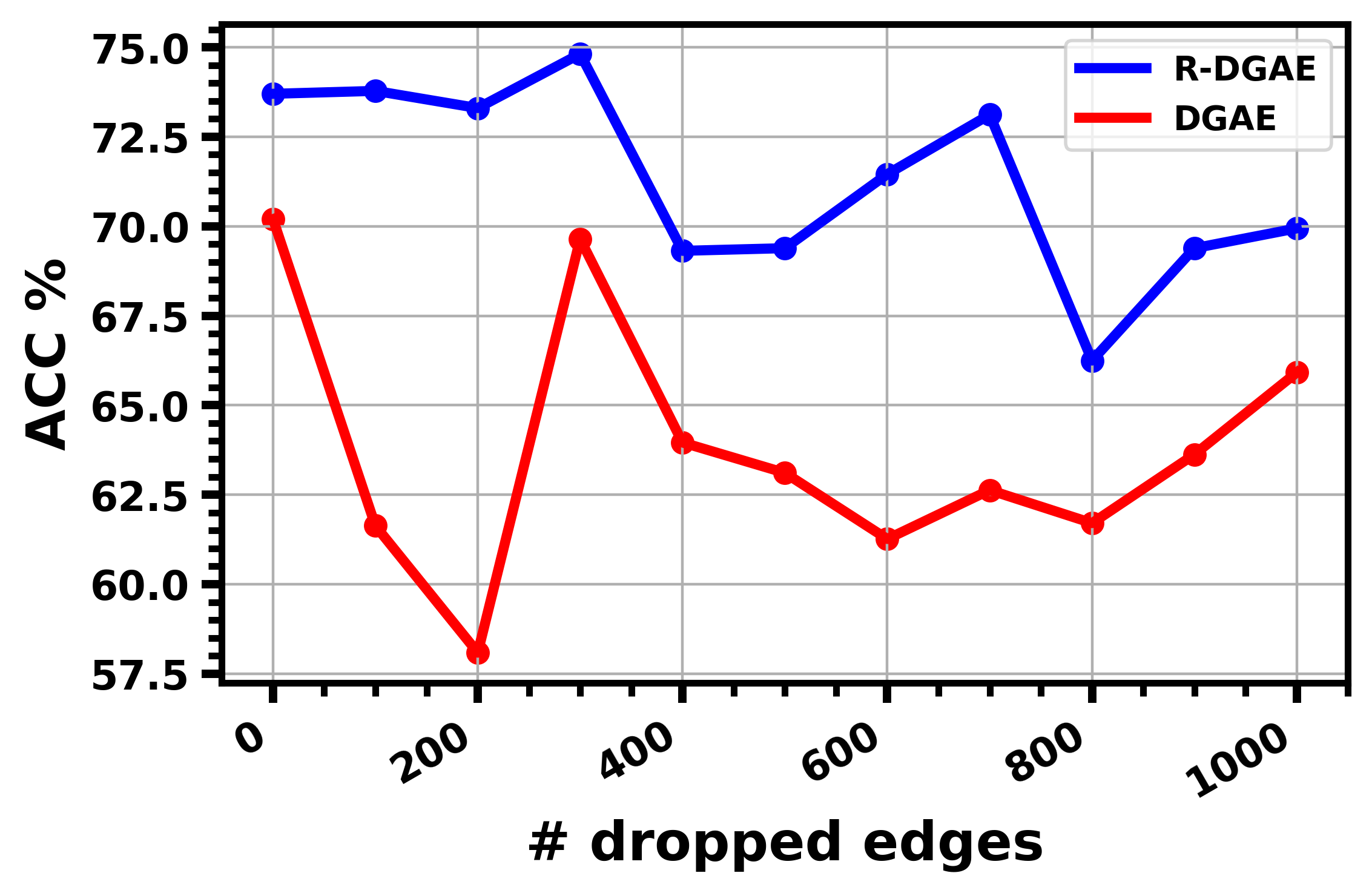}
  \end{subfigure}
  \begin{subfigure}[b]{0.24\textwidth}
     \includegraphics[width=\linewidth]{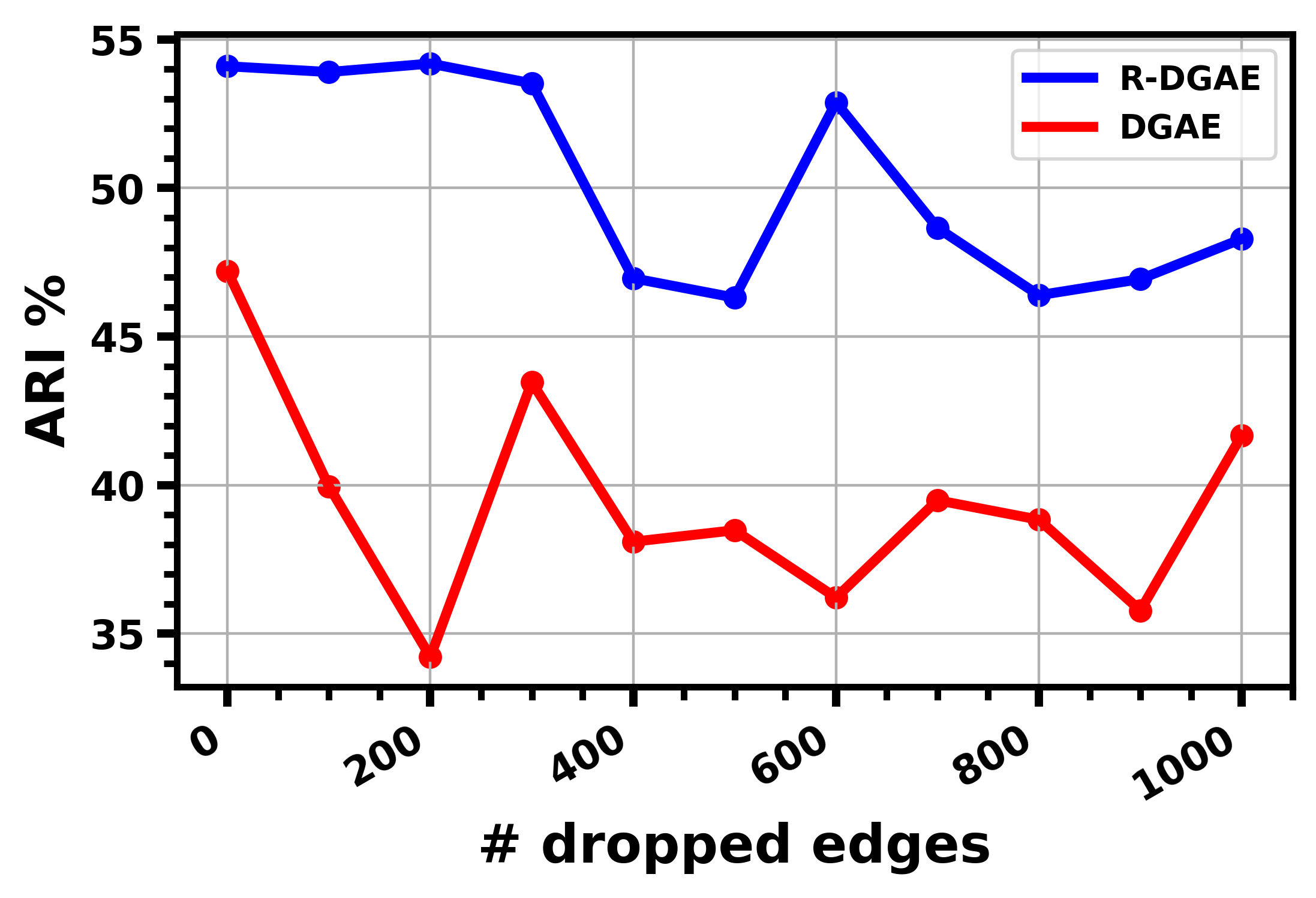}
  \end{subfigure}
  \begin{subfigure}[b]{0.24\textwidth}
     \includegraphics[width=\linewidth]{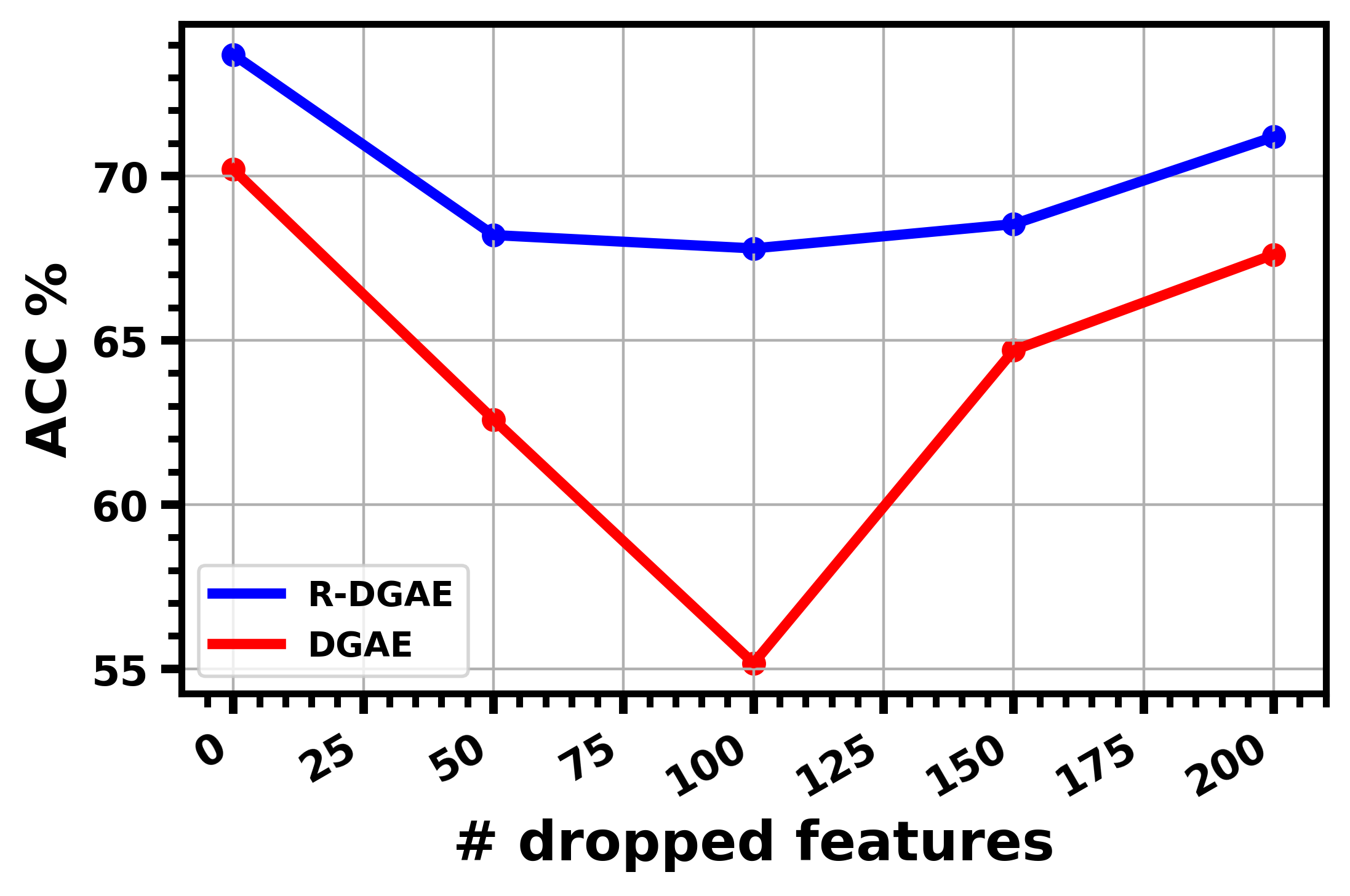}
  \end{subfigure}
  \begin{subfigure}[b]{0.24\textwidth}
     \includegraphics[width=\linewidth]{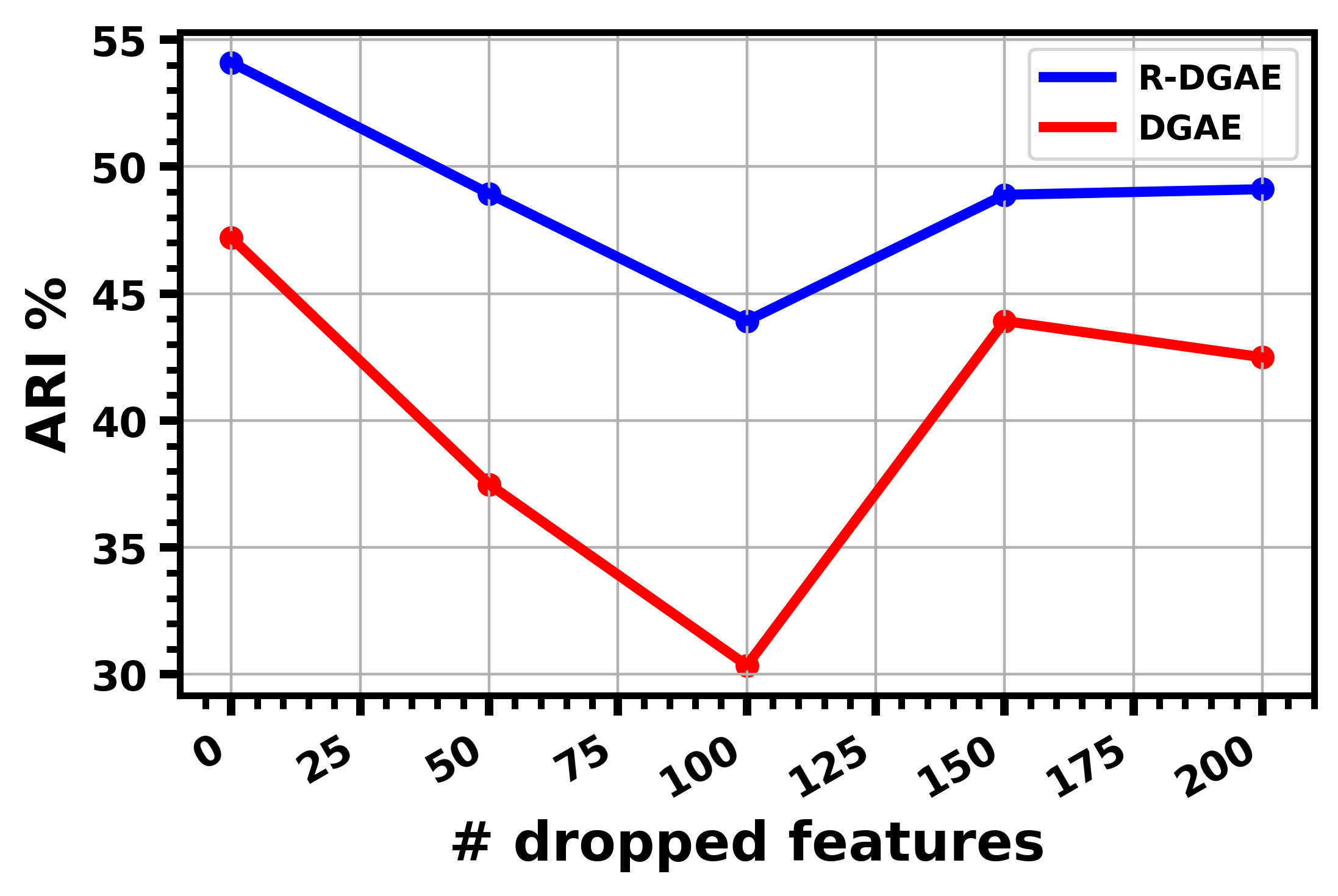}
  \end{subfigure}
  \caption{Performance of R-DGAE and DGAE on Cora, in terms of ACC and ARI, after dropping edges and features.}
  \label{fig:robustness_dropped_edges_DGAE}
\end{figure*}

\textbf{Learning dynamics:} In this part, we discuss the learning dynamics of R-GMM-VGAE on Cora. As we can see from Figure \ref{fig:lr_dyn_R-GMM-VGAE} (a), the number of decidable nodes (i.e., nodes in $\Omega$) increases gradually. The gradual increase of $\Omega$ demonstrates that performing embedded clustering with reliable nodes allows to gradually capture more challenging nodes. In Figure \ref{fig:lr_dyn_R-GMM-VGAE} (b), we illustrate the evolution of ACC for the whole set of nodes $\mathcal{V}$, and in Figure \ref{fig:lr_dyn_R-GMM-VGAE} (c), we illustrate the evolution of ACC for $\Omega$ and $\mathcal{V} - \Omega$ (i.e., undecidable nodes whose clustering assignments are not sufficient to decide to which clusters they belong). In the beginning, the number of decidable nodes is equal to $586$, and its ACC is around $0.88$. At the end of the training, the accuracy of $\Omega$ remains higher than $0.8$, and the size of $\Omega$ constitutes more than $90\%$ of $\mathcal{V}$. These results provide evidence that $\Xi$ can collect a sufficient portion of nodes with reliable clustering assignments.  

To investigate the role of our graph-transforming operator $\Upsilon$, we conduct a series of experiments to understand the evolution of the constructed graph $A_{clus}^{self}$. The obtained results are illustrated in Figures \ref{fig:lr_dyn_R-GMM-VGAE} (d), (e), and (f). As we can see from Figure \ref{fig:lr_dyn_R-GMM-VGAE} (d), the number of links for $A_{clus}^{self}$ increases gradually. At the end of the training process, the number of links exceeds $10,000$. Most importantly, the number of false links (i.e., links between nodes with different labels) remains small compared to the number of true links (i.e., links between nodes with the same labels). From Figure \ref{fig:lr_dyn_R-GMM-VGAE} (e), we can see that most of the added links are true links, and the number of false links among the added links is considerably inferior to the number of added true links. From Figure \ref{fig:lr_dyn_R-GMM-VGAE} (f), we observe that the number of deleted links is one order of magnitude smaller than the number of added links. Thus, we expect that the impact of adding edges on clustering effectiveness is much stronger than the impact of dropping edges.  We have investigated this aspect in our ablation study. Starting from epoch 60 of Figure \ref{fig:lr_dyn_R-GMM-VGAE} (f), we observe that the number of false links among the deleted links is not always inferior to the number of deleted true links. This result indicates the possibility of improving our results by early stopping the operation "dropping edges". In the absence of a clear explanation to this observation, and to keep our solution as simple as possible, we refrain from adjusting the "dropping edges" operation according to the obtained results. Globally, our analysis suggests that $\Upsilon$ gradually constructs a more clustering-oriented graph $A_{clus}^{self}$ compared with the initial graph $A$.

\begin{figure*}[!h]
  \centering
  \begin{subfigure}[b]{0.33\textwidth}
    \includegraphics[width=\linewidth]{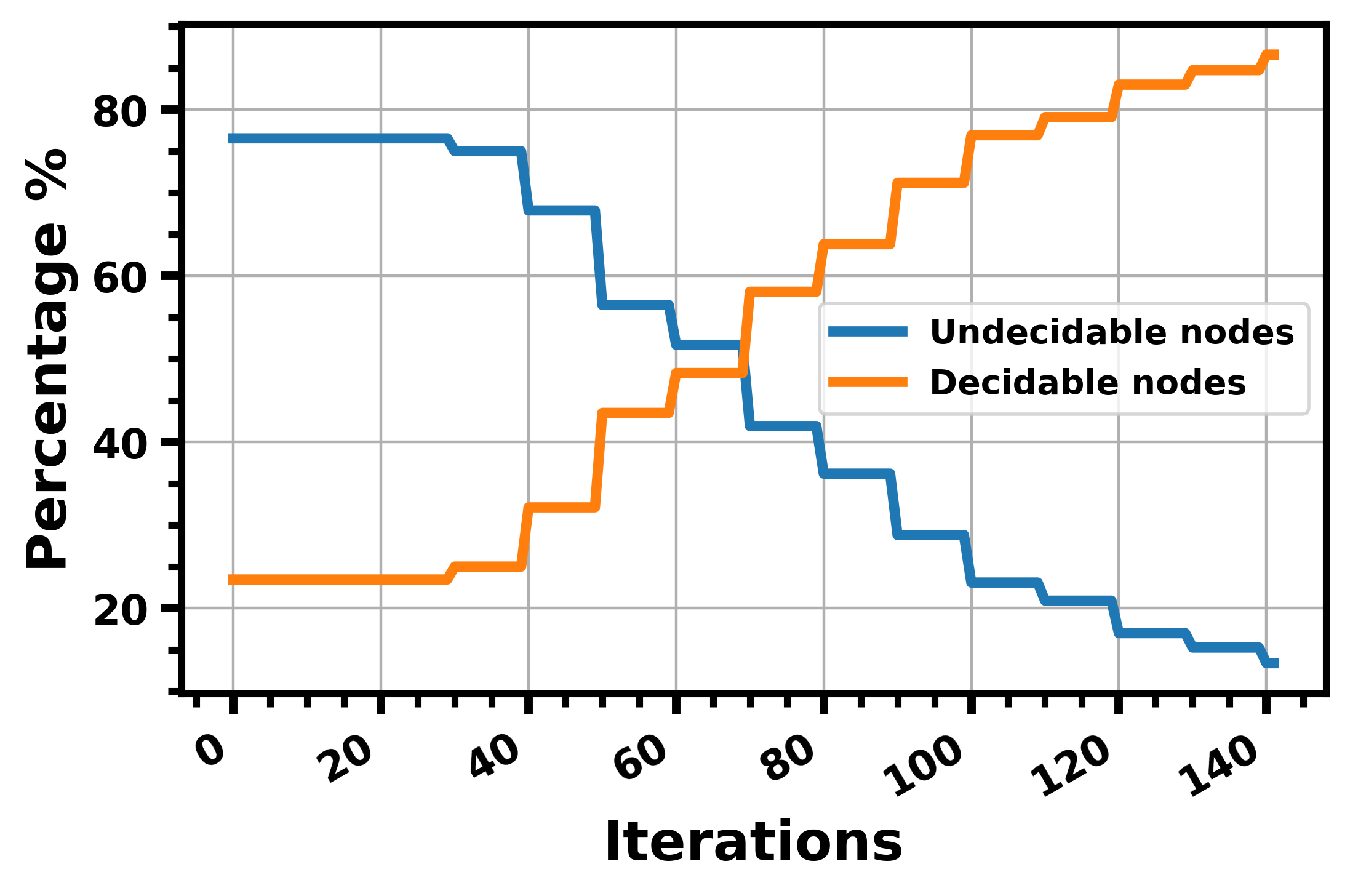}
    \caption{\% of decidable and undecidable nodes}
  \end{subfigure} \hfil
  \begin{subfigure}[b]{0.33\textwidth}
     \includegraphics[width=\linewidth]{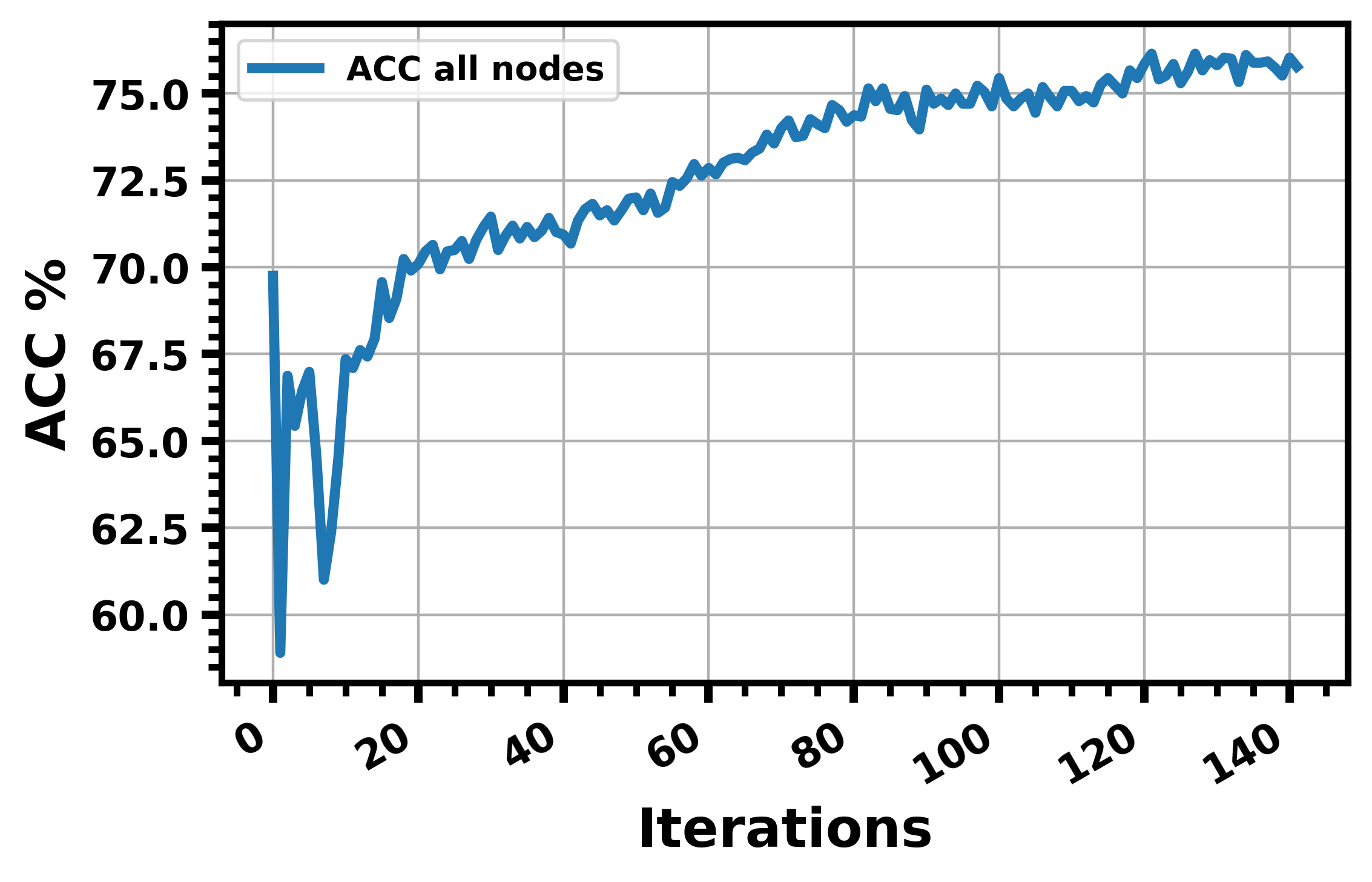}
     \caption{ACC: all nodes}
  \end{subfigure} \hfil
  \begin{subfigure}[b]{0.33\textwidth}
    \includegraphics[width=\linewidth]{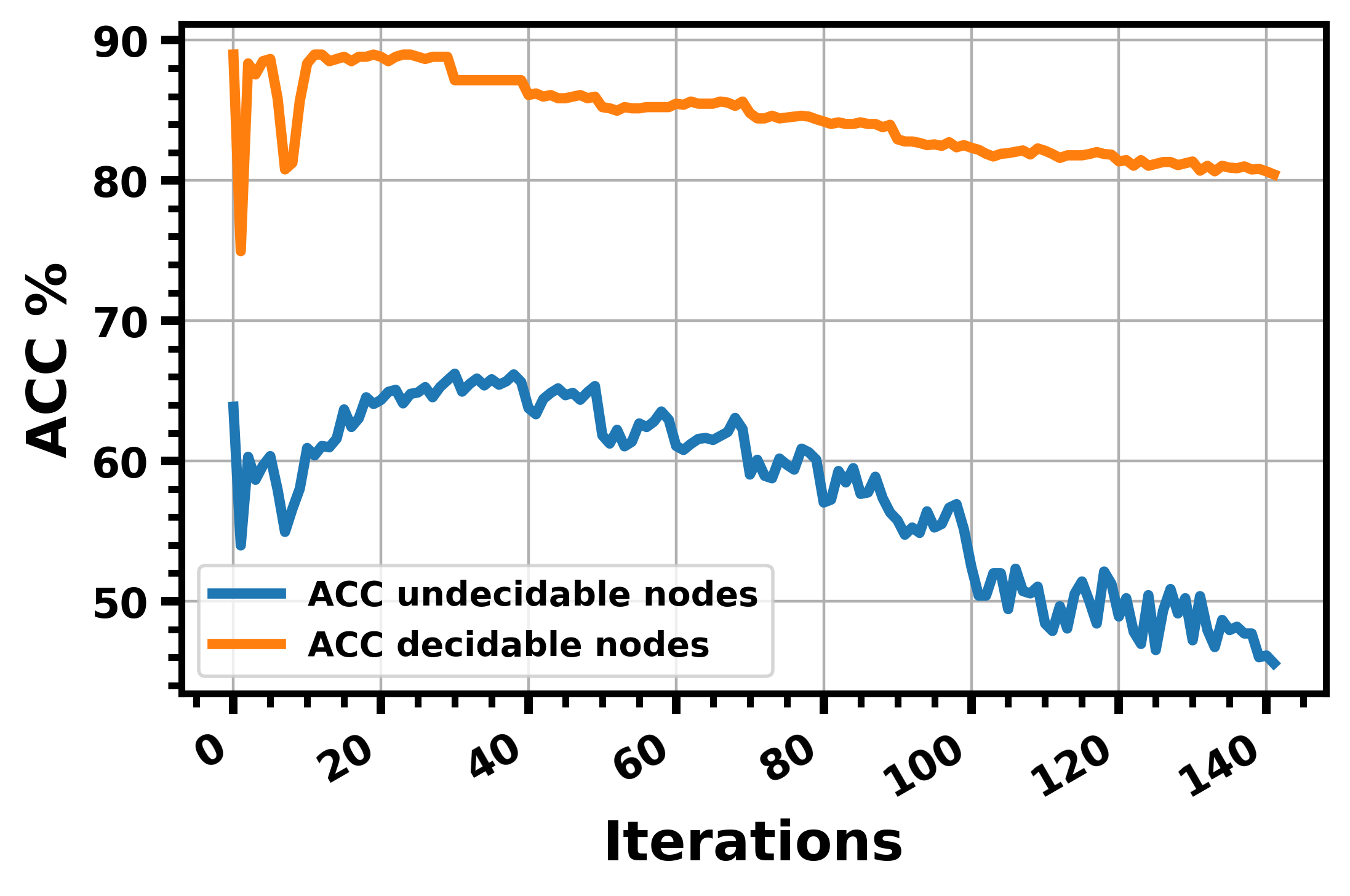}
    \caption{ACC: decidable and undecidable nodes}
  \end{subfigure} \hfil
  
  \medskip
  \begin{subfigure}[b]{0.33\textwidth}
    \includegraphics[width=\linewidth]{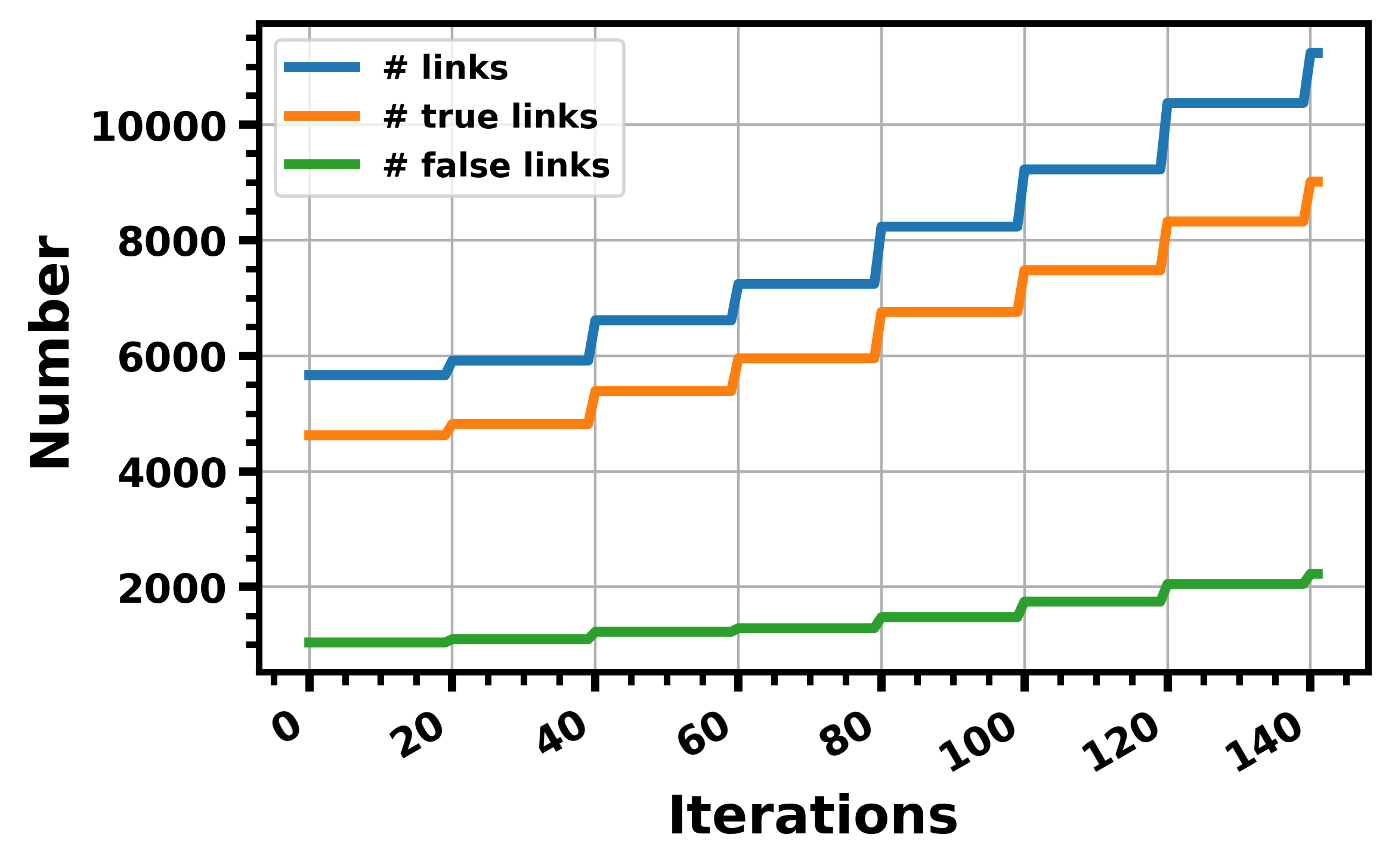}
    \caption{\# links $A_{clus}^{self}$}
  \end{subfigure} \hfil
  \begin{subfigure}[b]{0.33\textwidth}
     \includegraphics[width=\linewidth]{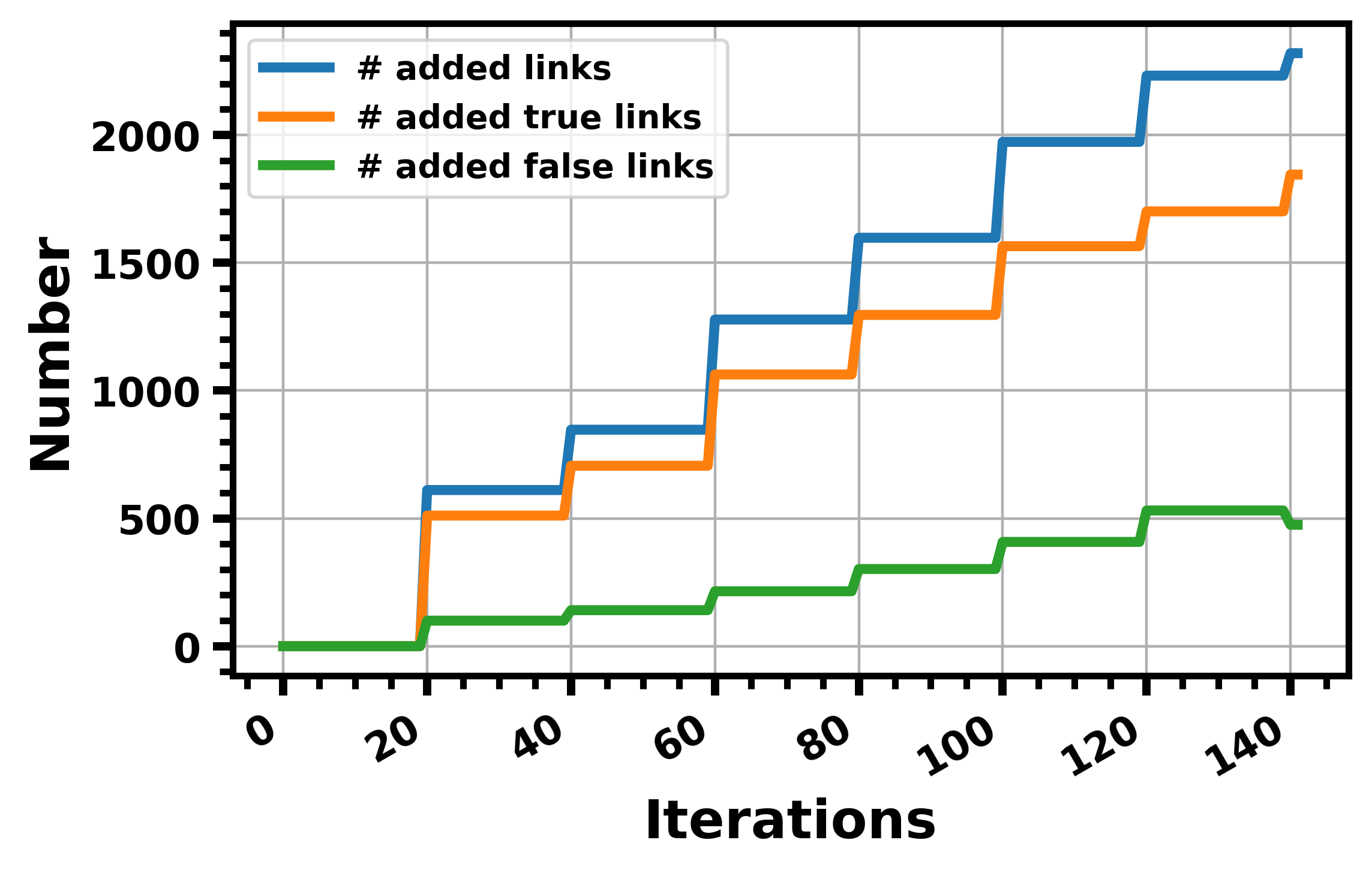}
     \caption{\# added links $A_{clus}^{self}$}
  \end{subfigure} \hfil
  \begin{subfigure}[b]{0.33\textwidth}
    \includegraphics[width=\linewidth]{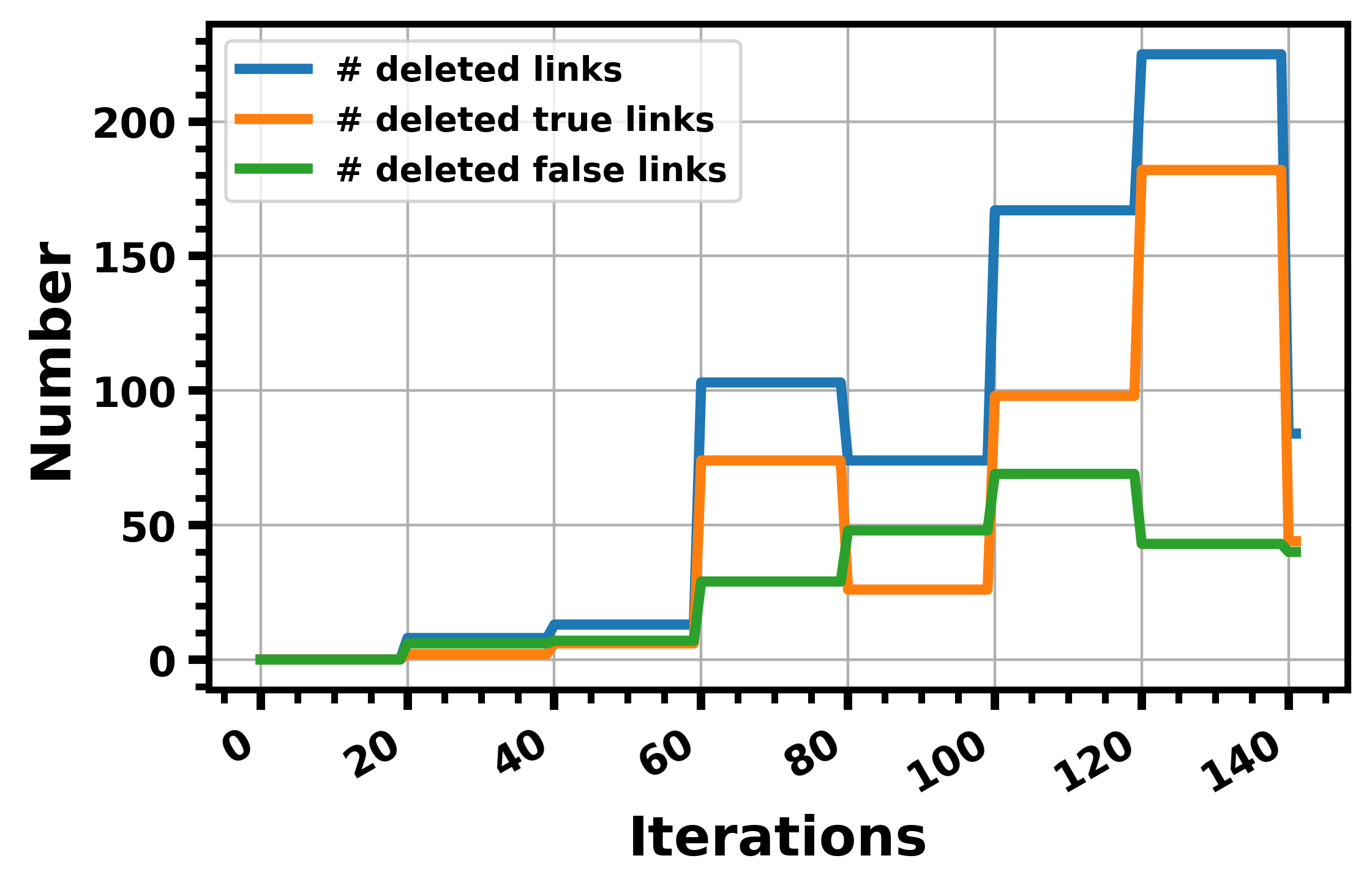}
    \caption{\# deleted links $A_{clus}^{self}$}
  \end{subfigure} \hfil
  \caption{Learning dynamics of R-GMM-VGAE on Cora}
  \label{fig:lr_dyn_R-GMM-VGAE}
\end{figure*}

\textbf{Visualisation of $Z$:} In Figure \ref{fig:vis_tnse}, we visualize the latent representations of GMM-VGAE and R-GMM-VGAE, during training on Cora. It is noteworthy that both models share the same pretraining weights. At epoch $40$, we observe that R-GMM-VGAE makes minor modifications to the embedded representations compared with GMM-VGAE. At this level, GMM-VGAE has already formed some well-separated clusters. Unlike GMM-VGAE, R-GMM-VGAE only uses the decidable nodes for performing embedded clustering. Therefore, it takes more iterations to obtain clustering-friendly representations. Finally, at epoch 120, we observe that R-GMM-VGAE has better separability between the different clusters than GMM-VGAE. Mainly, R-MM-VGAE is more able to separate between the red and purple groups. Furthermore, unlike R-GMM-VGAE, GMM-VGAE can not separate between the blue and pink clusters. These results confirm the importance of our operator $\Xi$ in building high-quality clusters.

\begin{figure*}[!h]
  \centering
  \begin{subfigure}[b]{0.24\textwidth}
    \includegraphics[width=\linewidth]{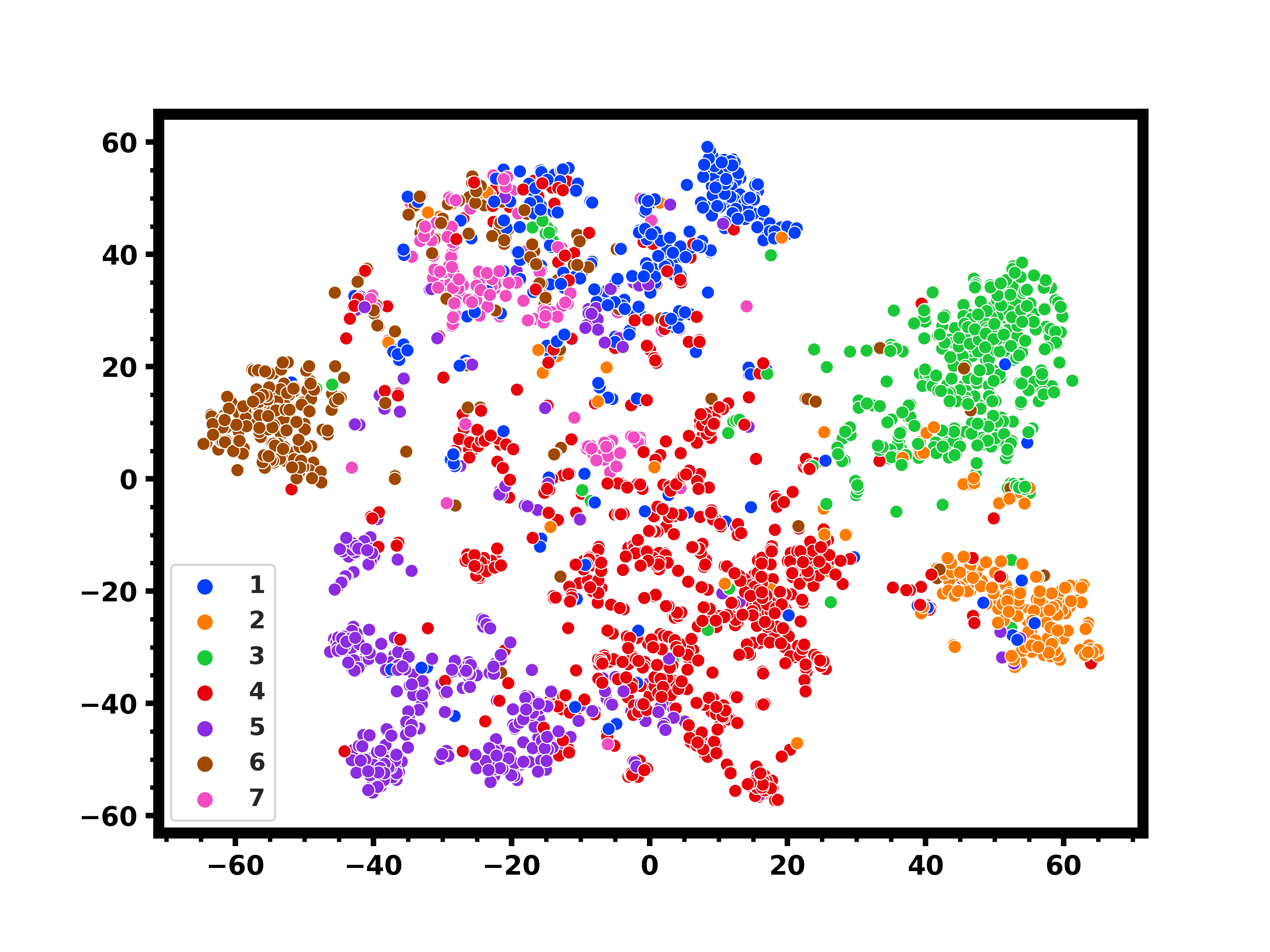}
    \caption{Epoch 0}
 \end{subfigure} \hfil
  \begin{subfigure}[b]{0.24\textwidth}
     \includegraphics[width=\linewidth]{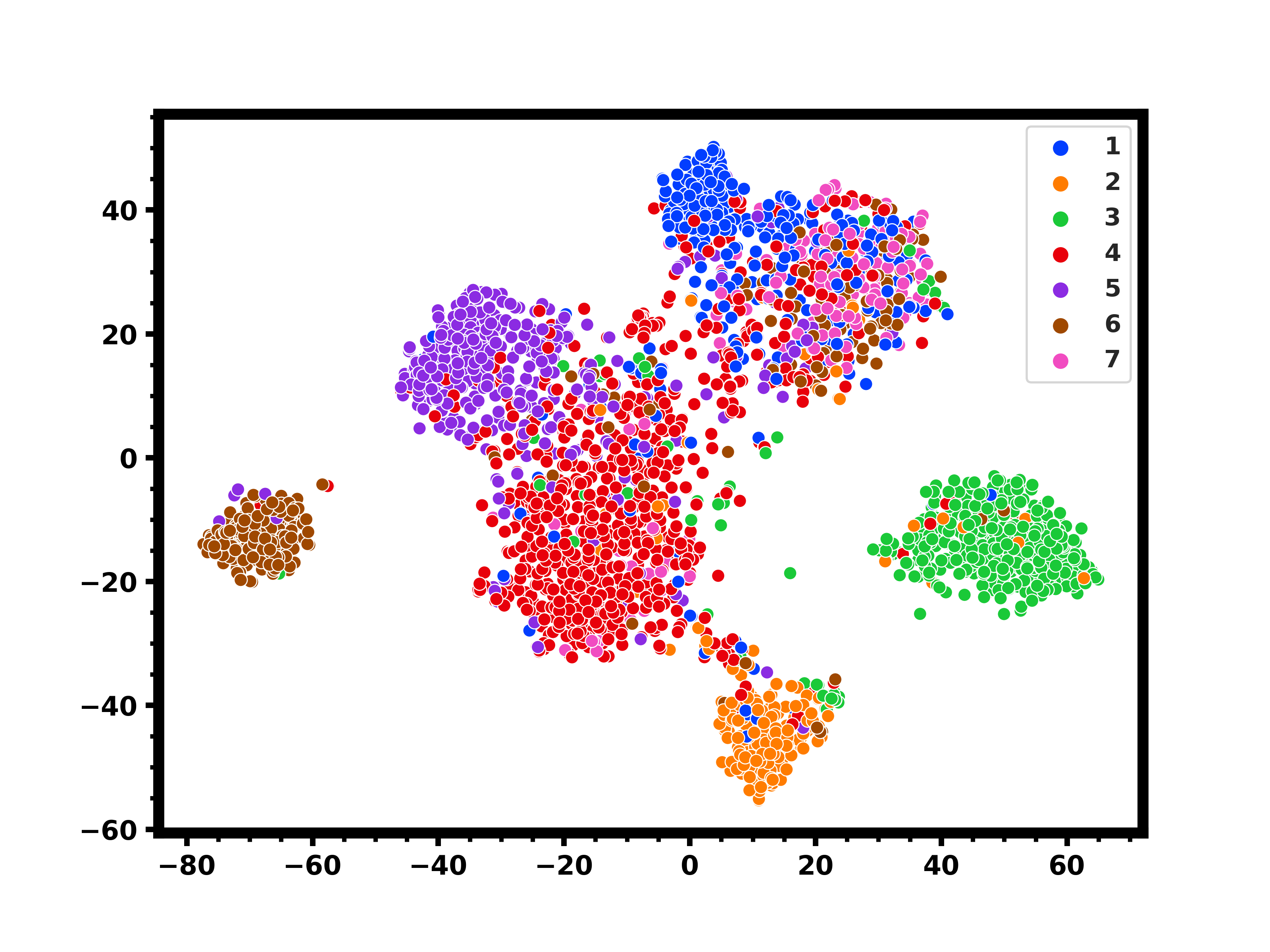}
     \caption{Epoch 40}
  \end{subfigure} \hfil
  \begin{subfigure}[b]{0.24\textwidth}
    \includegraphics[width=\linewidth]{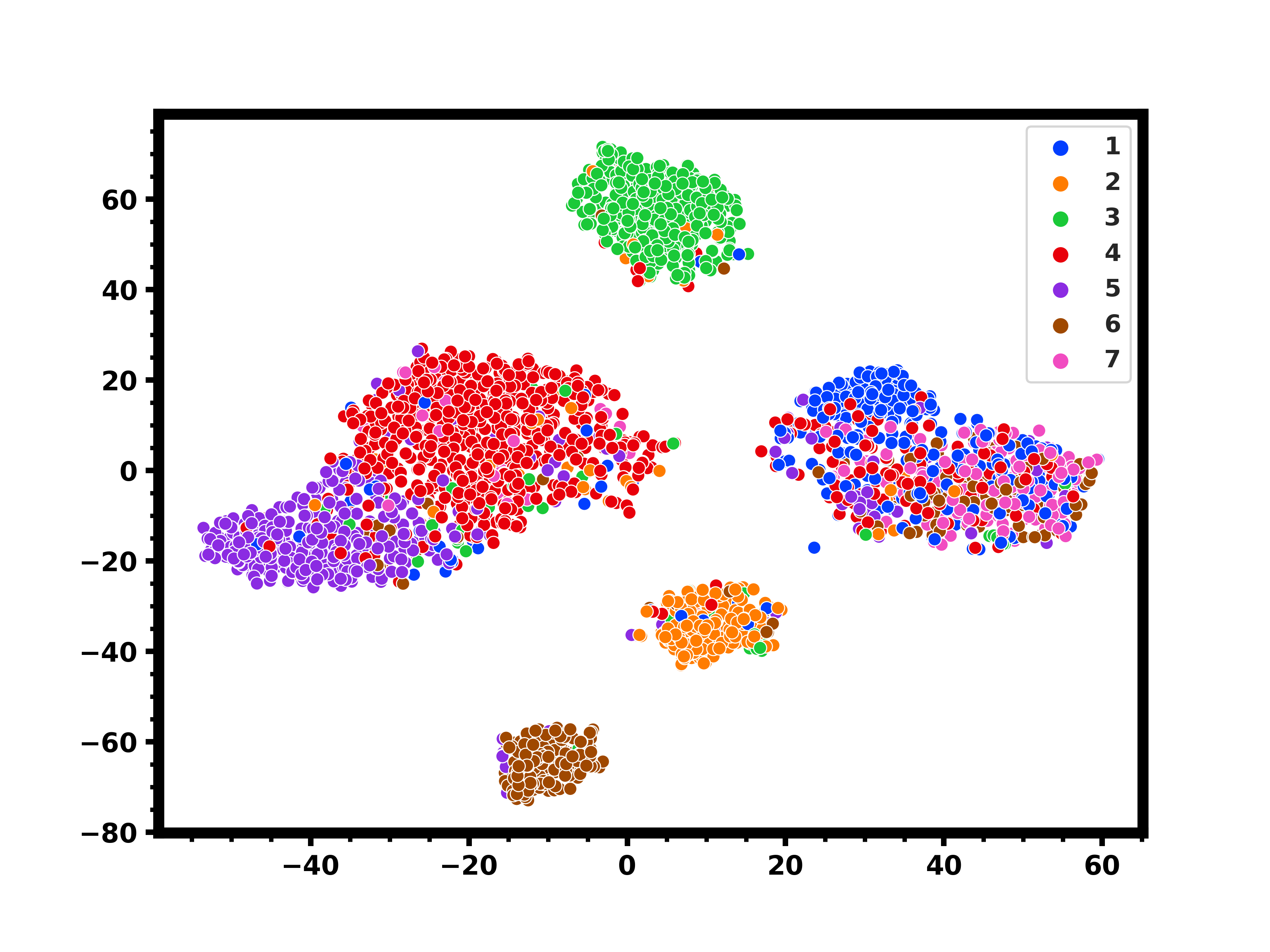}
    \caption{Epoch 80}
  \end{subfigure} \hfil
  \begin{subfigure}[b]{0.24\textwidth}
    \includegraphics[width=\linewidth]{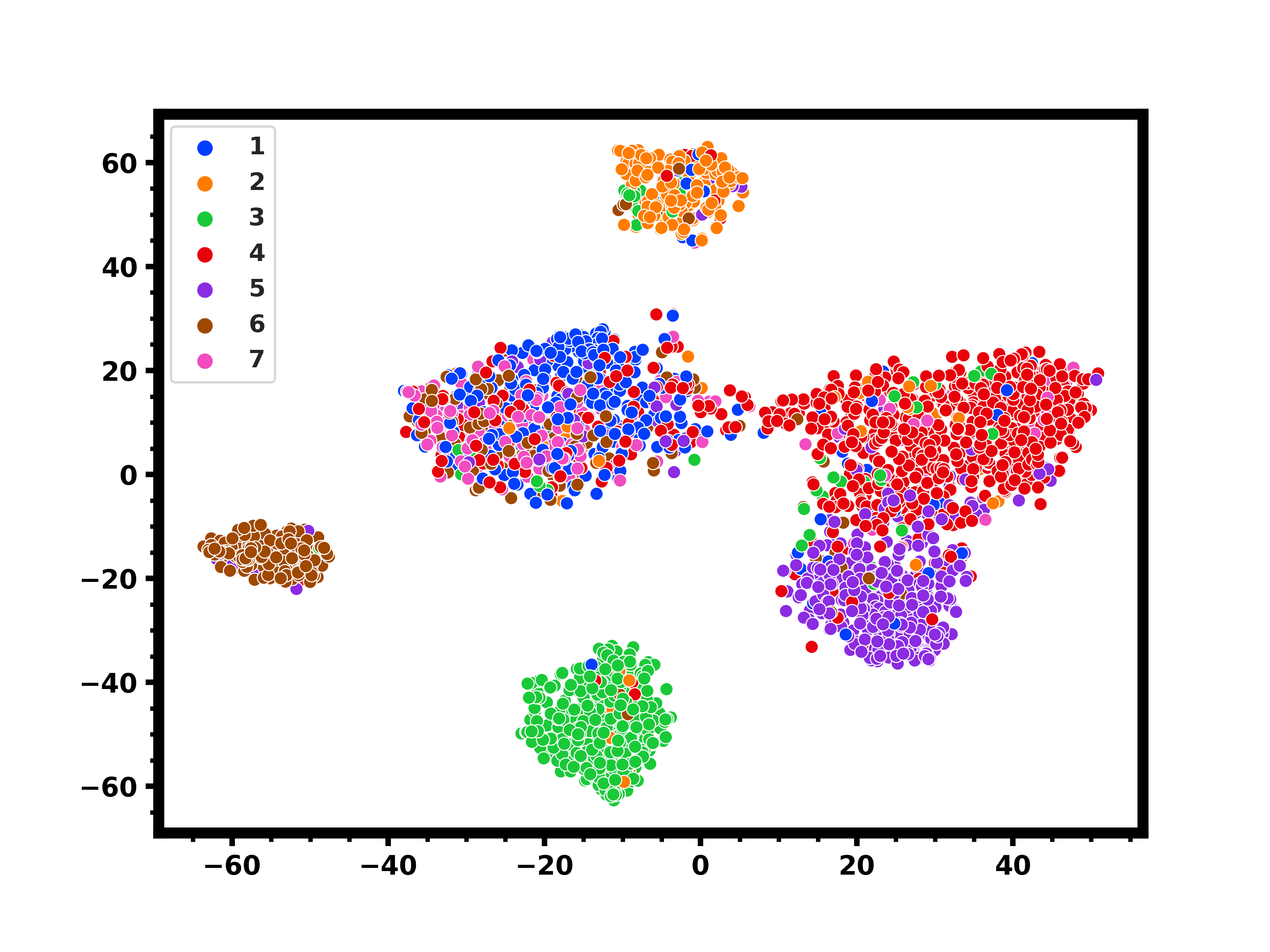}
    \caption{Epoch 120}
  \end{subfigure} \hfil
  
  \medskip
  \begin{subfigure}[b]{0.24\textwidth}
    \includegraphics[width=\linewidth]{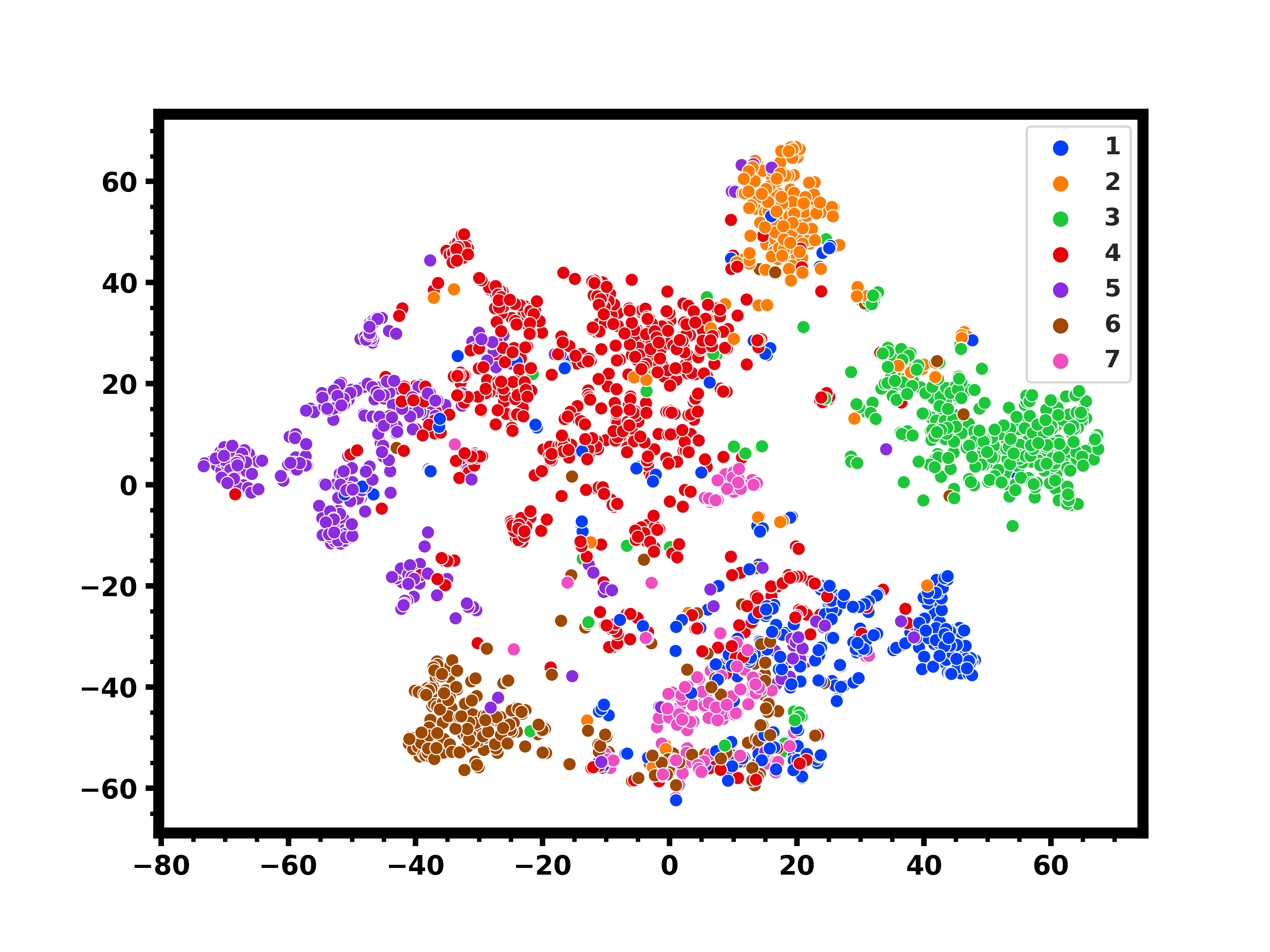}
    \caption{Epoch 0}
  \end{subfigure} \hfil
  \begin{subfigure}[b]{0.24\textwidth}
     \includegraphics[width=\linewidth]{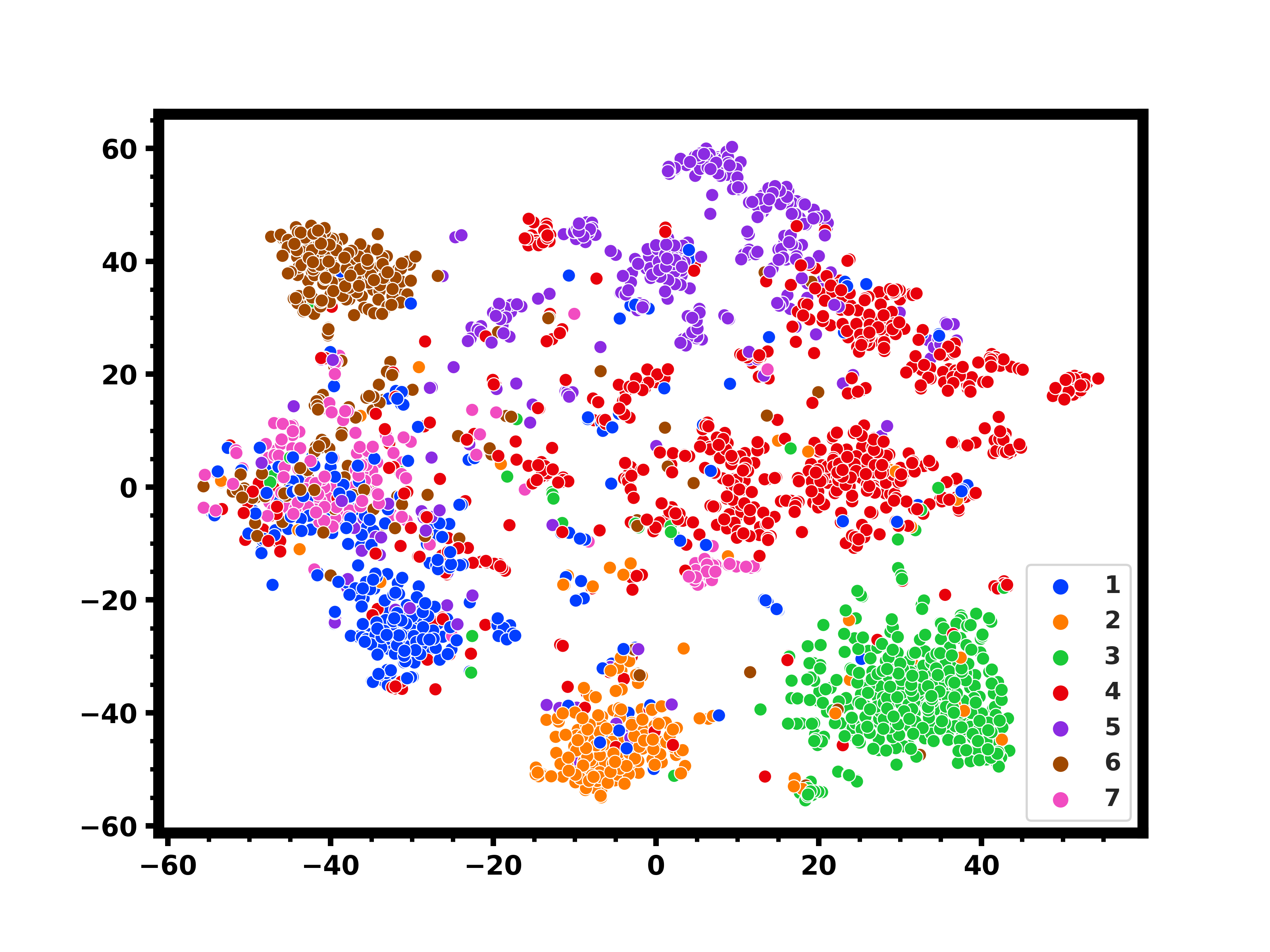}
     \caption{Epoch 40}
  \end{subfigure} \hfil
  \begin{subfigure}[b]{0.24\textwidth}
    \includegraphics[width=\linewidth]{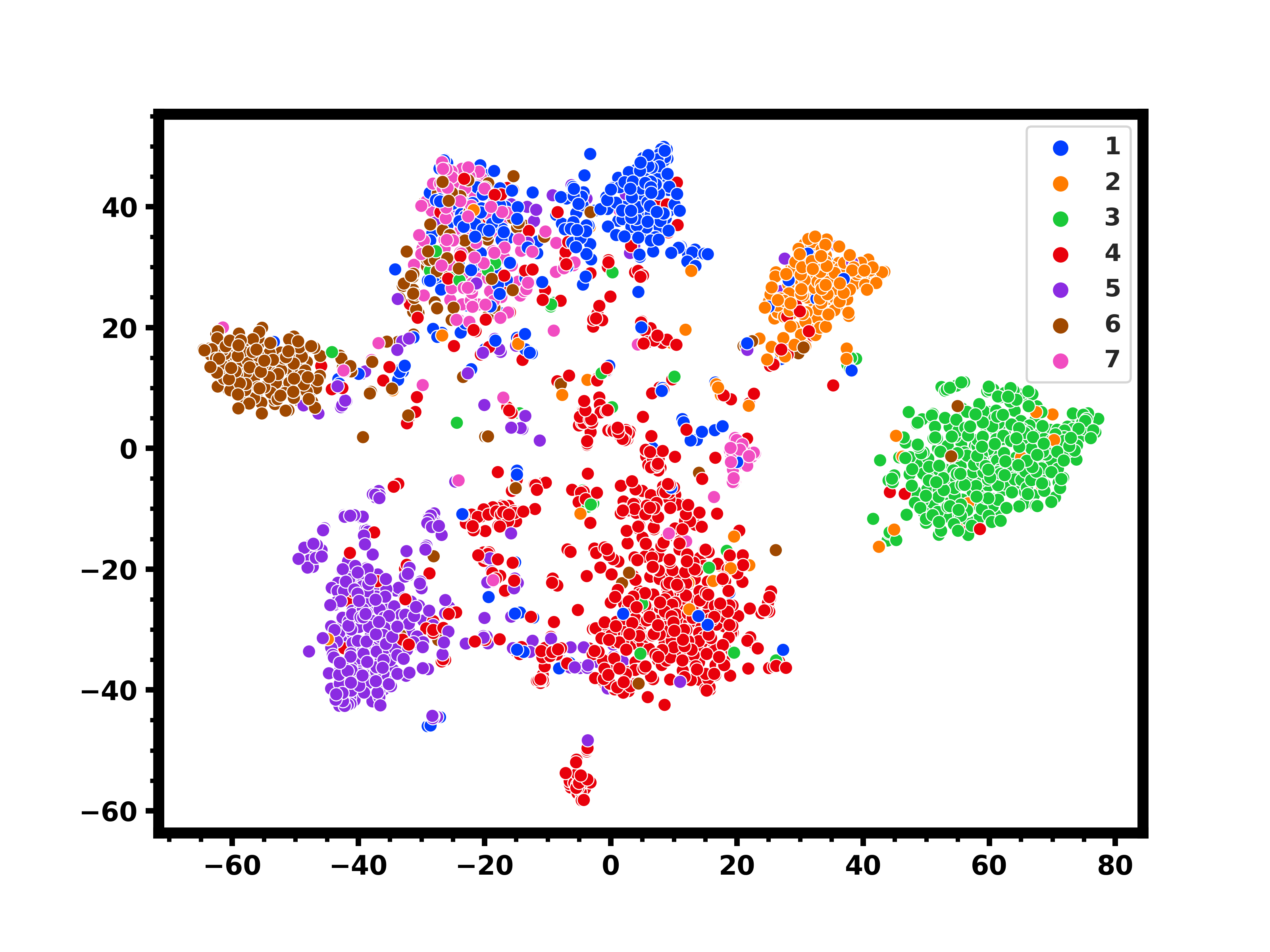}
    \caption{Epoch 80}
  \end{subfigure} \hfil
  \begin{subfigure}[b]{0.24\textwidth}
    \includegraphics[width=\linewidth]{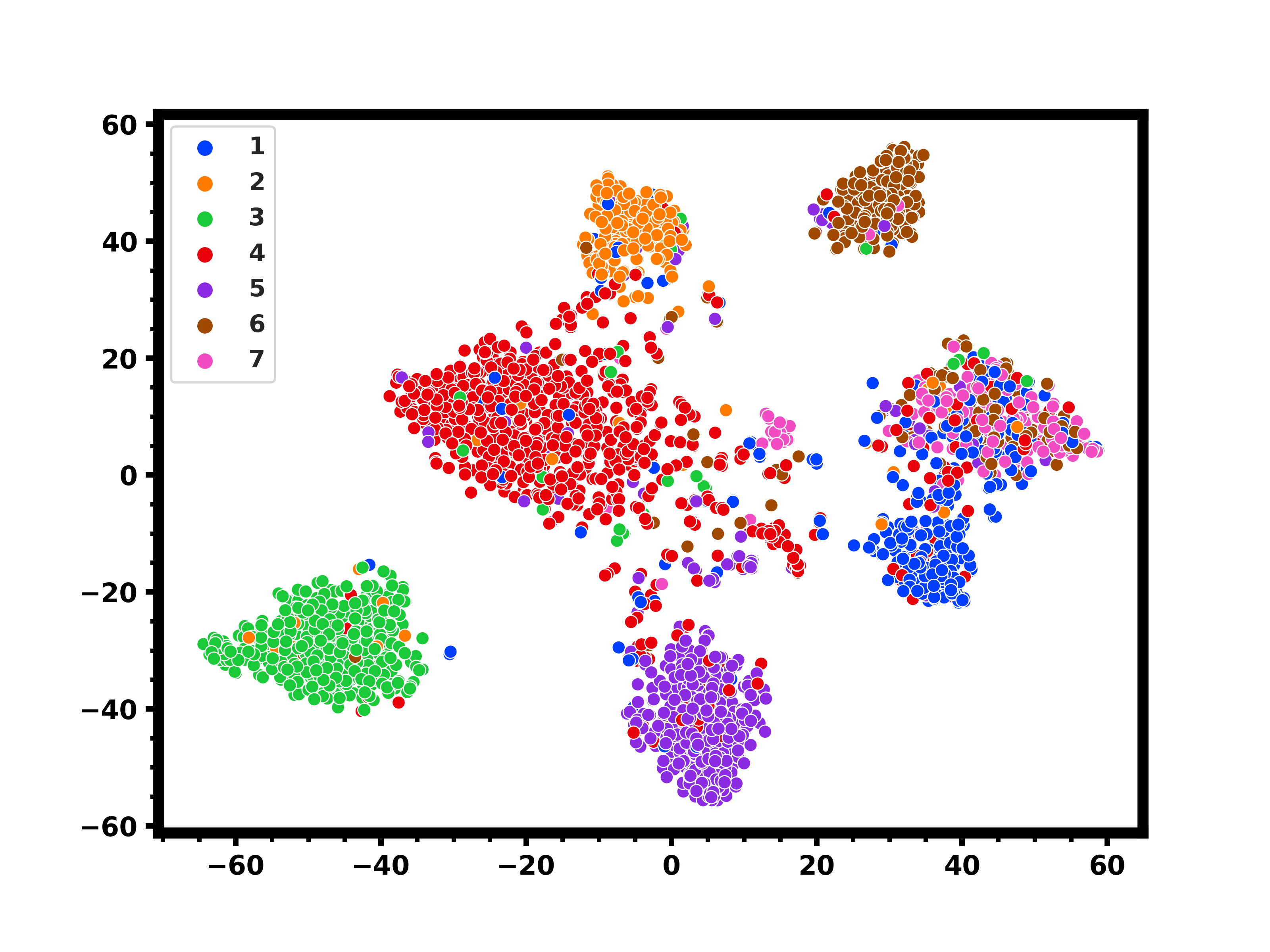}
    \caption{Epoch 120}
  \end{subfigure} \hfil
\vskip 0.1in
  \caption{2D visualizations of the latent representations of GMM-VGAE and R-GMM-VGAE, on Cora using T-SNE. Top row: latent representations of GMM-VGAE; bottom row: latent representations of R-GMM-VGAE.}
  \label{fig:vis_tnse}
\end{figure*}

\textbf{Sensitivity to the confidence thresholds:} In Figures \ref{fig:sens_confidence_cora_R-GMM-VGAE} and \ref{fig:sens_confidence_cora_R-DGAE}, we illustrate the sensitivity of R-GMM-VGAE and R-DGAE, respectively, to the confidence thresholds $\alpha_{1}$ and $\alpha_{2}$ on Cora. For $\alpha_{1}$, we try several values from the set $\left \{0.1, 0.2, 0.3, 0.4 \right \}$. We find that setting $\alpha_{1}$ higher than $0.4$ leads to an empty set $\Omega$. Therefore, $0.4$ is the highest value we can try for $\alpha_{1}$. Our strategy for setting $\alpha_{1}$ consists of choosing the highest value that can give birth to a nonempty set $\Omega$. For $\alpha_{2}$, we evaluate several values from the set $\left \{0.05, 0.1, 0.15, 0.20, 0.25 \right \}$. We find that setting $\alpha_{2}$ higher than $0.25$ leads to an empty set $\Omega$. As we can see from both figures (i.e., Figure \ref{fig:sens_confidence_cora_R-GMM-VGAE} and Figure \ref{fig:sens_confidence_cora_R-DGAE}), R-GMM-VGAE and R-DGAE give reasonable results in a wide range of parameters.

\begin{figure*}[!h]
  \vspace*{-3mm}
  \begin{subfigure}[b]{0.33\textwidth}
    \includegraphics[width=\linewidth]{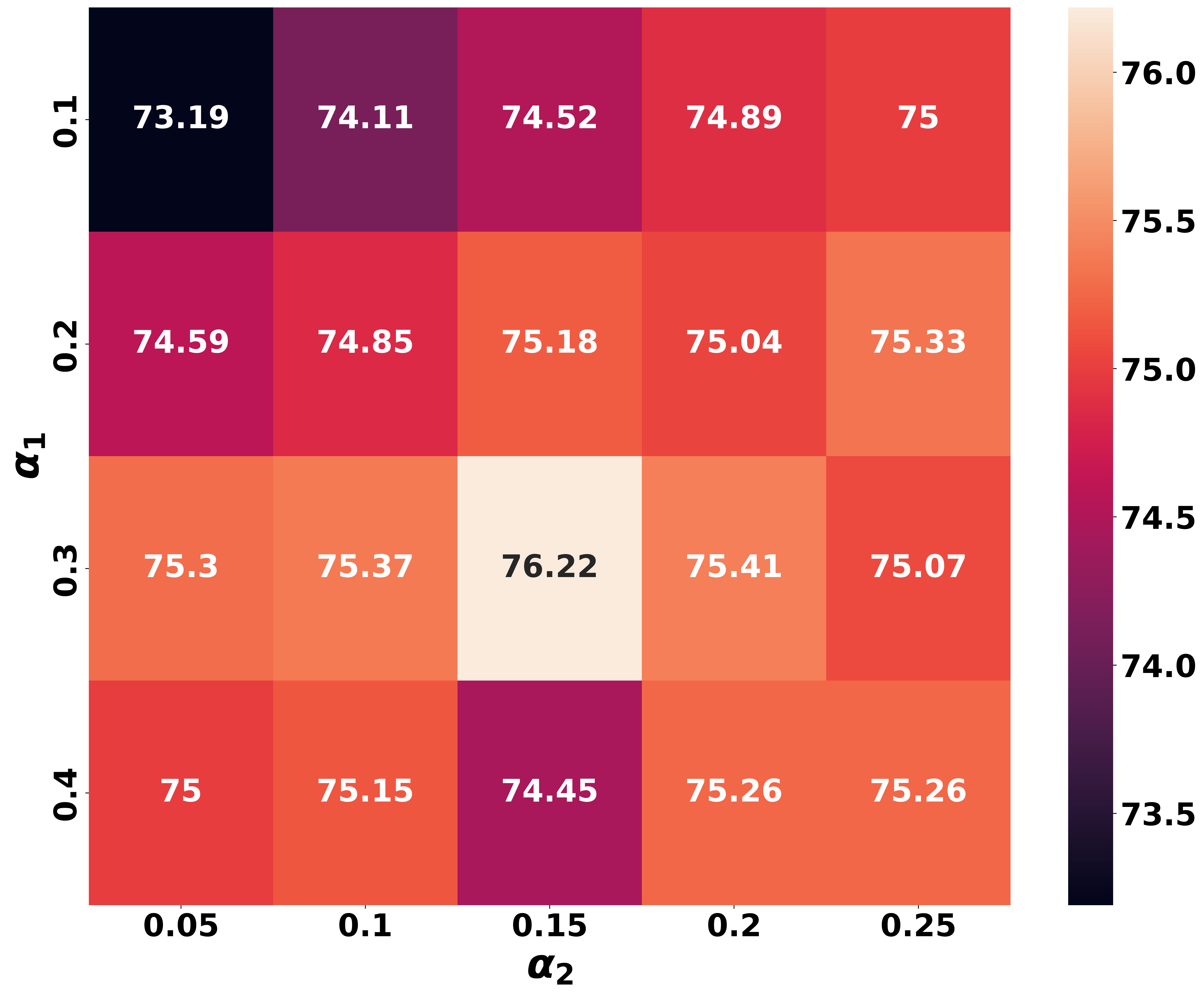}
    \caption{ACC}
  \end{subfigure}
  \begin{subfigure}[b]{0.33\textwidth}
     \includegraphics[width=\linewidth]{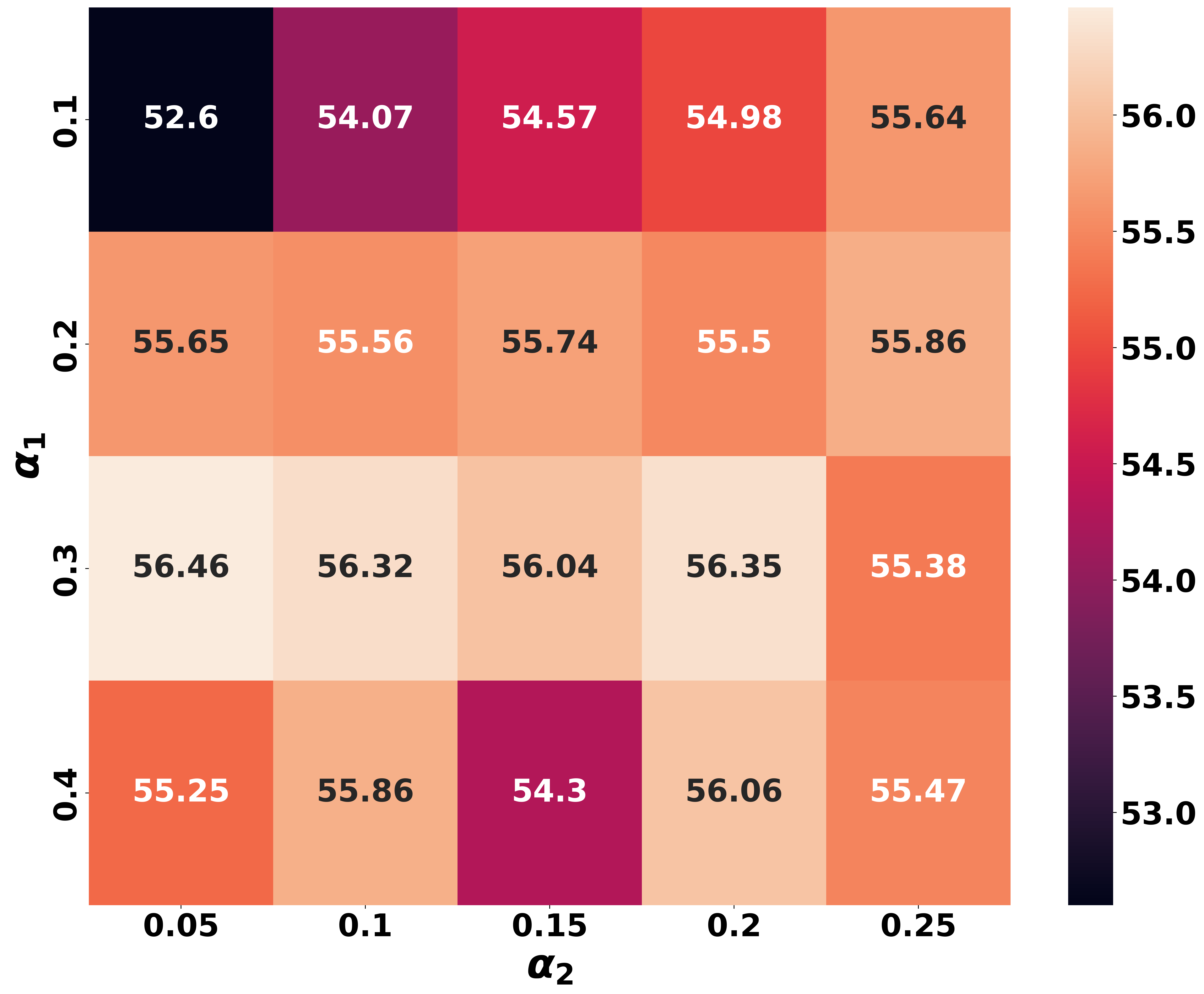}
     \caption{NMI}
  \end{subfigure}
  \begin{subfigure}[b]{0.33\textwidth}
     \includegraphics[width=\linewidth]{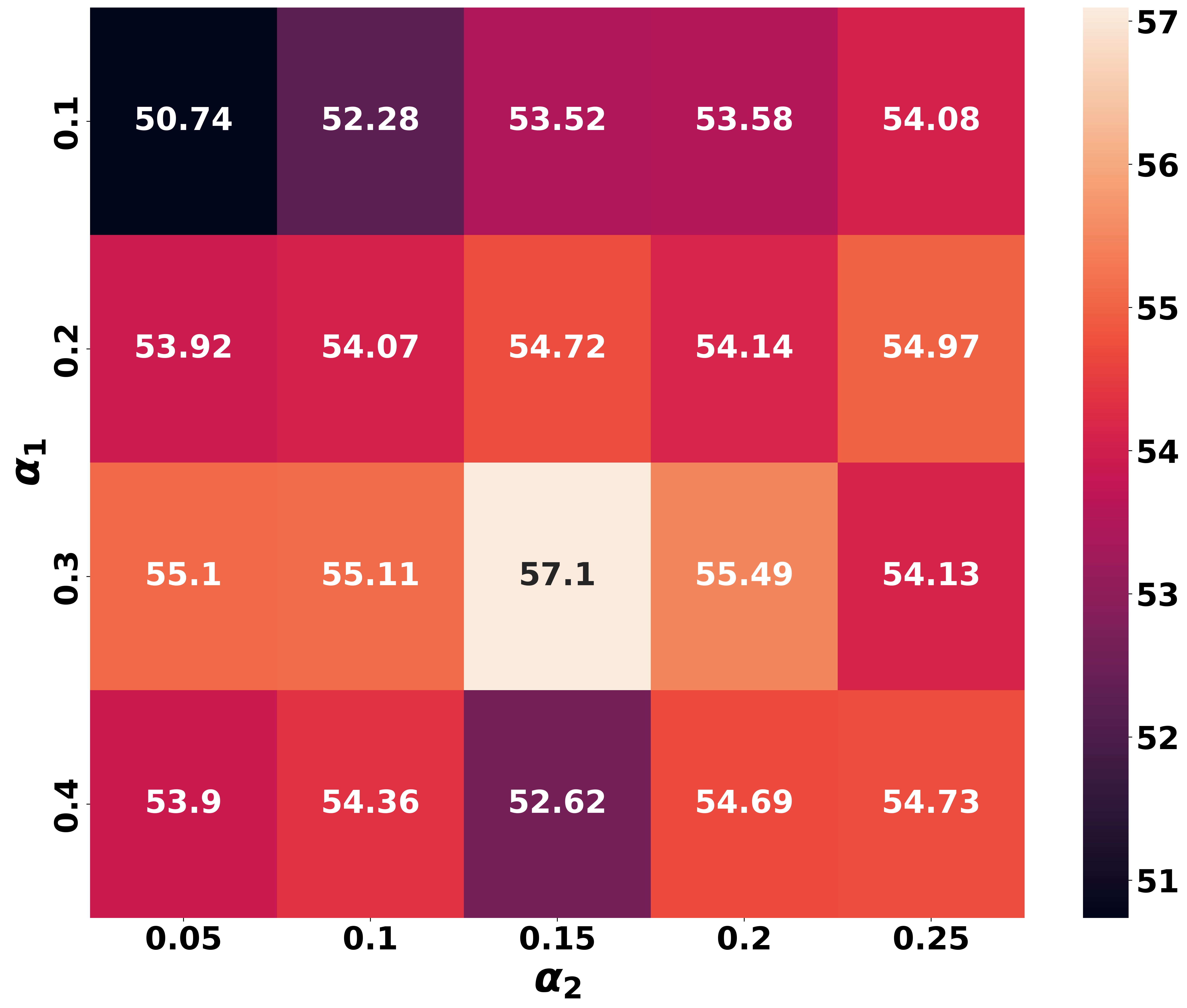}
     \caption{ARI}
  \end{subfigure}
  \caption{Influence of $\alpha_{1}$ and $\alpha_{2}$ values on ACC, NMI, and ARI for R-GMM-VGAE on Cora.}
  \label{fig:sens_confidence_cora_R-GMM-VGAE}
\end{figure*}

\begin{figure*}[!h]
  \begin{subfigure}[b]{0.33\textwidth}
    \includegraphics[width=\linewidth]{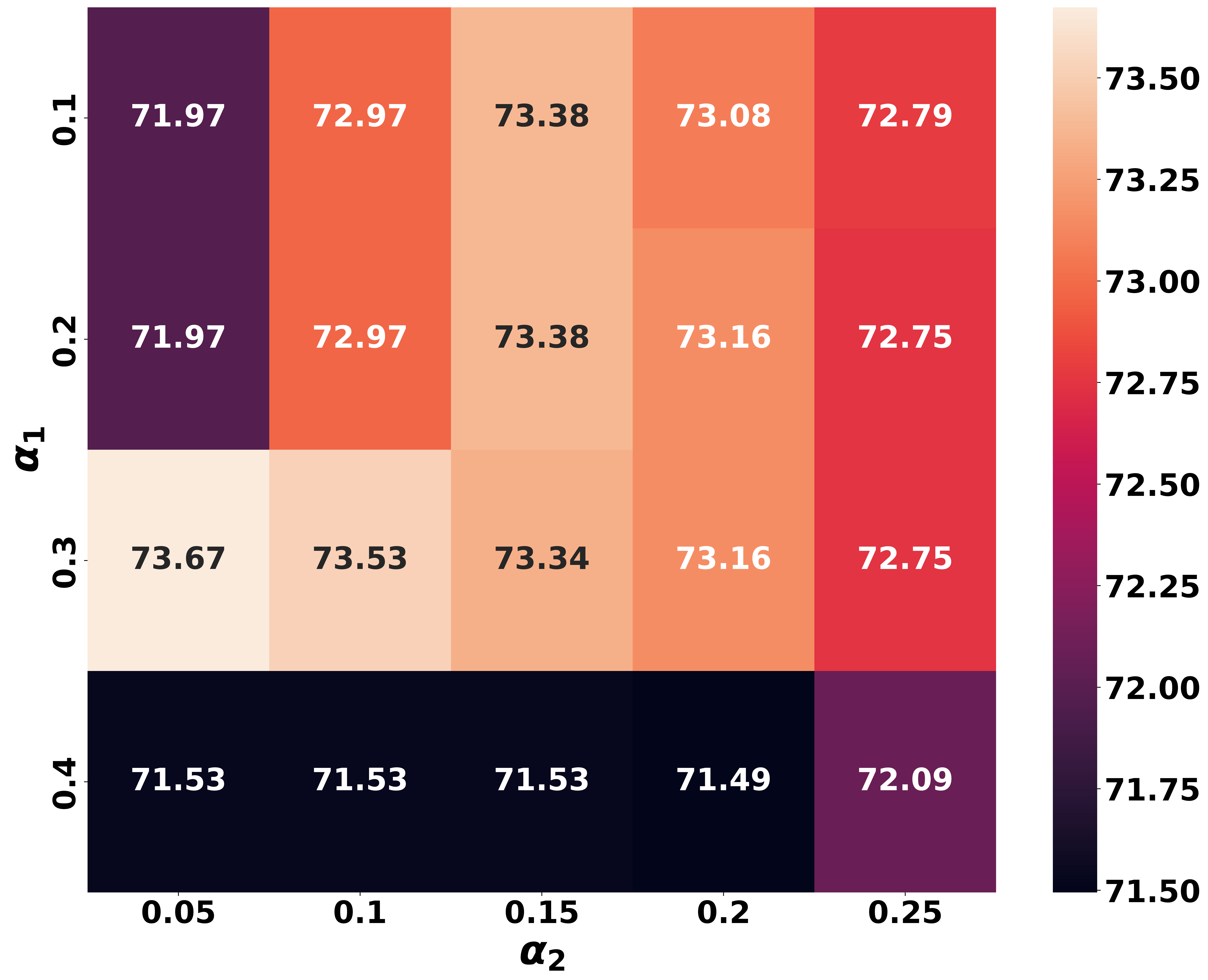}
  \end{subfigure}
  \begin{subfigure}[b]{0.33\textwidth}
     \includegraphics[width=\linewidth]{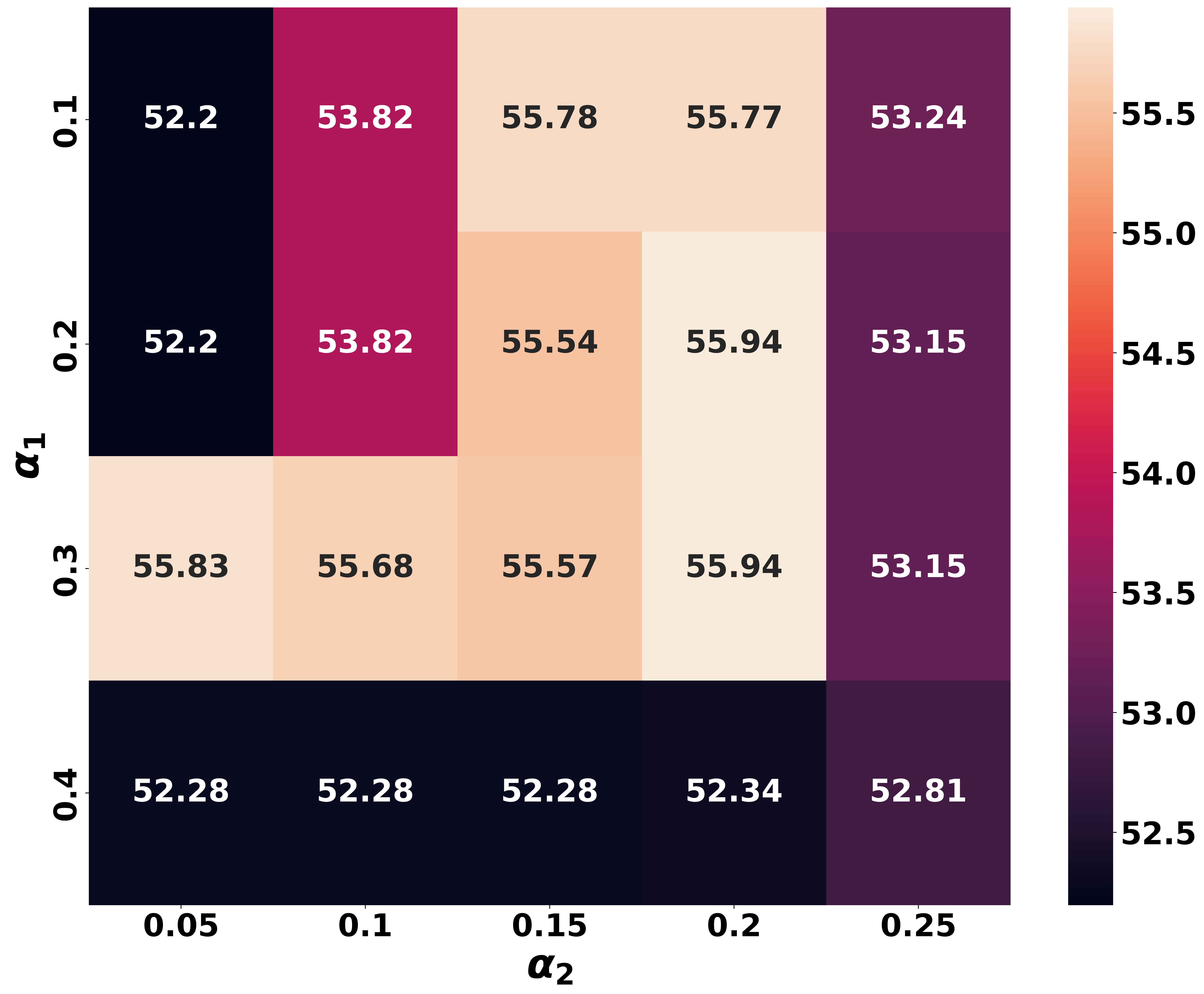}
  \end{subfigure}
  \begin{subfigure}[b]{0.33\textwidth}
     \includegraphics[width=\linewidth]{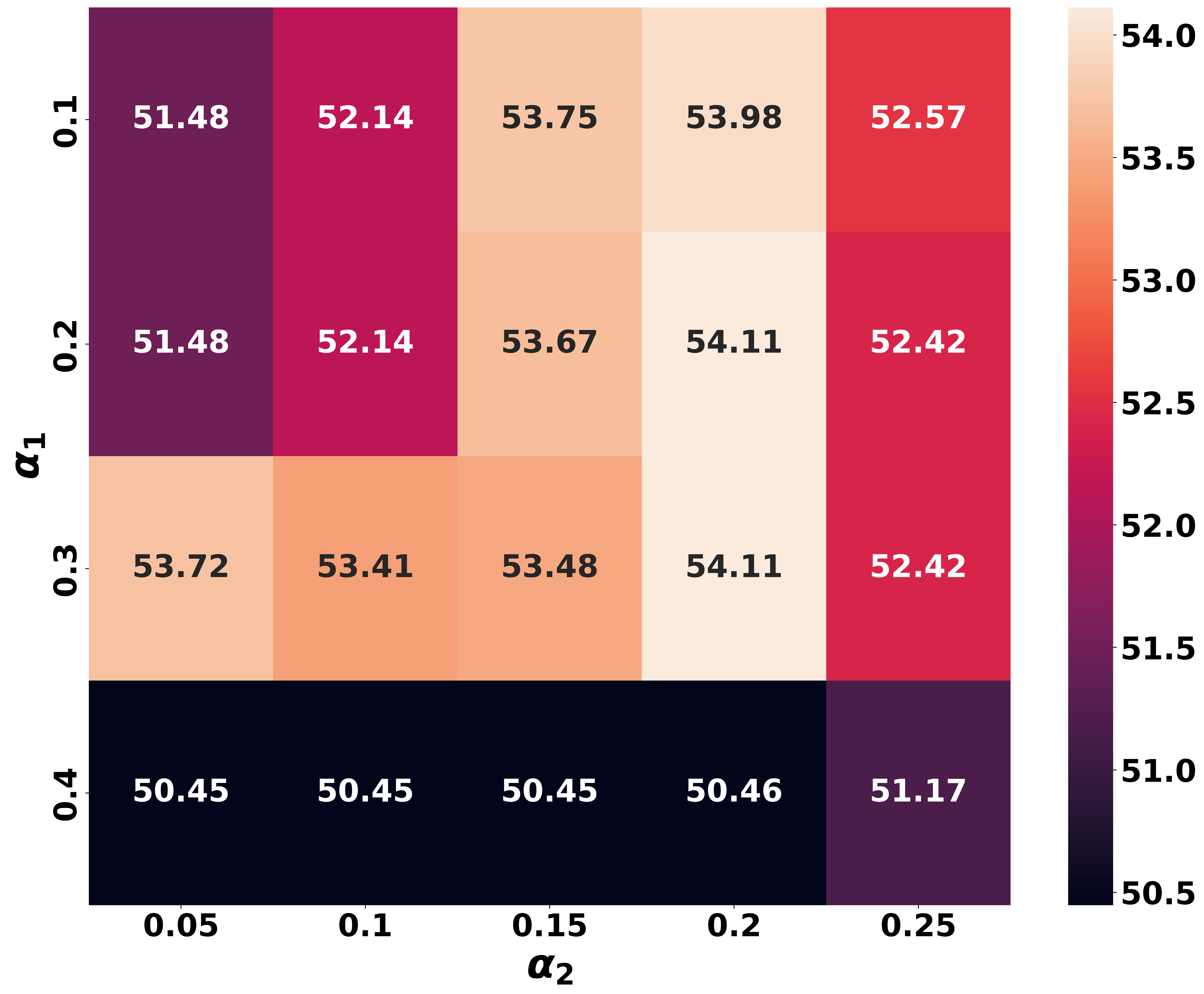}
  \end{subfigure}
  \caption{Influence of $\alpha_{1}$ and $\alpha_{2}$ values on ACC, NMI, and ARI for R-DGAE on Cora.}
  \label{fig:sens_confidence_cora_R-DGAE}
\end{figure*}

\textbf{Sensitivity to the balancing hyper-parameter:} In Figure \ref{fig:sens_hyperparameter_gamma_GMM-VGAE}, we assess the sensitivity of R-GMM-VGAE and GMM-VGAE to the balancing hyper-parameter $\gamma$ on Cora. As we can see from this figure, R-GMM-VGAE is less sensitive to $\gamma$ than GMM-VGAE. By transforming the reconstruction loss into a clustering-oriented loss, the competition between the optimized functions of R-GMM-VGAE (i.e., $L_{clus}(P(\Xi(Z(\theta))))$ and $L_{bce}(\hat{A}(Z(\theta)), \, \Upsilon(A, P(\Xi(Z(\theta))), \Omega))$) is less pronounced than the competition between the optimized functions of GMM-VGAE (i.e., $L_{clus}(P(Z(\theta)))$ and $L_{bce}(\hat{A}(Z(\theta)), \, A)$). 

\begin{figure*}[!h]
  \begin{subfigure}[b]{0.5\textwidth}
    \includegraphics[width=\linewidth]{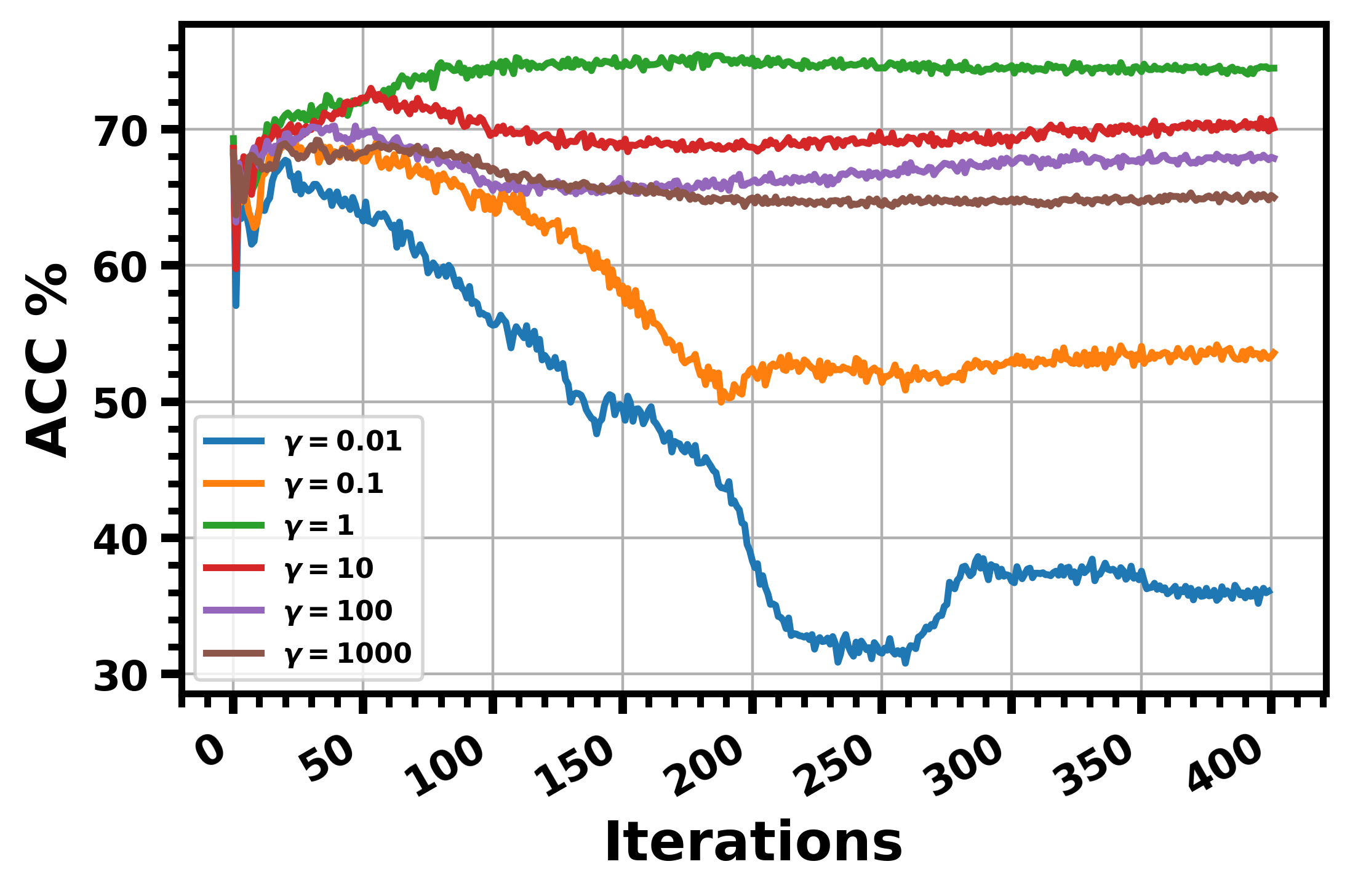}
    \caption{R-GMM-VGAE}
  \end{subfigure}
  \begin{subfigure}[b]{0.5\textwidth}
     \includegraphics[width=\linewidth]{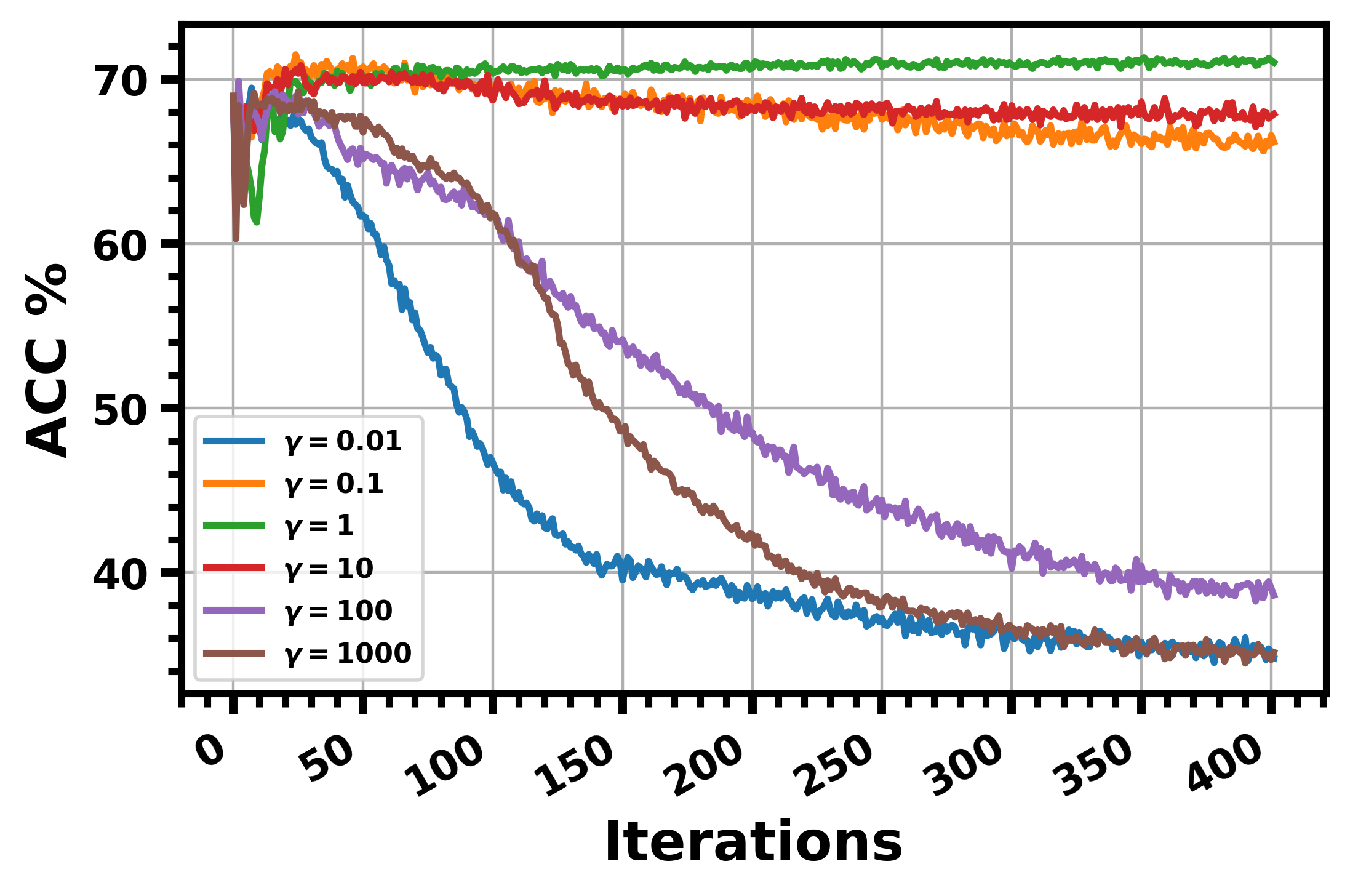}
     \caption{GMM-VGAE}
  \end{subfigure}
  \caption{Sensitivity of R-GMM-VGAE and GMM-VGAE to the balancing hyper-parameter on Cora.}
  \label{fig:sens_hyperparameter_gamma_GMM-VGAE}
\end{figure*}

\clearpage

\onecolumn

\section{Proof of Proposition 1} \label{Appendix_A}
\newtheorem*{P1}{Proposition~\ref{proposition_1}}
\begin{P1}
The reconstruction loss for a GAE model can be expressed as:
$$ L_{bce}(\hat{A}(Z(\theta)), A^{self}) = L_{\mathcal{C}}(Z(\theta), A^{self}) + L_{\mathcal{R}}(Z(\theta), A^{self}), $$
\begin{equation*}
\begin{split}
L_{\mathcal{R}}(Z(\theta), A^{self}) &=  \sum_{i,j}  \Big(log(1 + exp(z_{i}^{T} z_{j})) - \frac{1}{2}  a_{ij}^{self} (\norm{z_{i}}_{2}^{2} + \norm{z_{j}}_{2}^{2})\Big). \\
\end{split}
\end{equation*}
\end{P1}

\begin{proof}
\begin{equation*} 
\begin{split}
        L_{bce}(\hat{A}(Z(\theta)), A^{self}) & = - \sum_{1\leqslant i,j \leqslant N} \bigg( \; a_{ij}^{self} \; log(\frac{1}{1 + e^{-z_{i}^{T}z_{j}}}) + \; (1-a_{ij}^{self}) \; log(\frac{e^{-z_{i}^{T}z_{j}}}{1 + e^{-z_{i}^{T}z_{j}}}) \bigg), \\ 
        & =  \sum_{1\leqslant i,j \leqslant N} \; \bigg( a_{ij}^{self} \; log(e^{-z_{i}^{T}z_{j}}) + log(1 + e^{-z_{i}^{T}z_{j}}) - log(e^{-z_{i}^{T}z_{j}}) \bigg),\\
        & = \sum_{1\leqslant i,j \leqslant N} \;  \bigg( (1-a_{ij}^{self}) \; z_{i}^{T}z_{j} + log(1 + e^{-z_{i}^{T}z_{j}}) \bigg), \\ 
        & = \frac{1}{2} \; \sum_{1\leqslant i,j \leqslant N} \; (a_{ij}^{self} - 1) (z_{i}-z_{j})^{T}(z_{i}-z_{j}) - \frac{1}{2} \; \sum_{1\leqslant i,j \leqslant N} \; (a_{ij}^{self} - 1) (z_{i}^{T}z_{i} + z_{j}^{T}z_{j}) + \; \sum_{1\leqslant i,j \leqslant N} \; log(1+e^{-z_{i}^{T}z_{j}}) , \\ 
        & = \frac{1}{2} \; \sum_{1\leqslant i,j \leqslant N} \; (a_{ij}^{self} - 1) \norm{z_{i}-z_{j}}_{2}^{2} - \frac{1}{2} \; \sum_{1\leqslant i,j \leqslant N} \; (a_{ij}^{self} - 1) (\norm{z_{i}}_{2}^{2} + \norm{z_{j}}_{2}^{2}) + \; \sum_{1\leqslant i,j \leqslant N} \; log(1+e^{-z_{i}^{T}z_{j}}) , \\ 
\end{split}
\end{equation*}

And since \\

\begin{equation*} 
\begin{split}
        \sum_{1 \leqslant i,j \leqslant N} log(1+e^{-z_{i}^{T}z_{j}}) & =  \sum_{1 \leqslant i,j \leqslant N} log(1+ e^{\frac{1}{2}\norm{z_{i}-z_{j}}_{2}^{2} - \frac{1}{2}(\norm{z_{i}}_{2}^{2} + \norm{z_{j}}_{2}^{2})}),\\
        & = \sum_{1 \leqslant i,j \leqslant N}log \bigg( \frac{e^{-\frac{1}{2}  \norm{z_{i}-z_{j}}_{2}^{2}}+e^{-\frac{1}{2}(\norm{z_{i}}_{2}^{2}+\norm{z_{j}}_{2}^{2})}}{e^{-\frac{1}{2} \norm{z_{i}-z_{j}}_{2}^{2}}} \bigg),\\
        & = \sum_{1 \leqslant i,j \leqslant N}  log(e^{-\frac{1}{2} \norm{z_{i}-z_{j}}_{2}^{2}}+e^{-\frac{1}{2}(\norm{z_{i}}_{2}^{2}+\norm{z_{j}}_{2}^{2})}) + \frac{1}{2} \sum_{1 \leqslant i,j \leqslant N} \norm{z_{i}-z_{j}}_{2}^{2},\\
        & = \sum_{1 \leqslant i,j \leqslant N}  log(e^{-\frac{1}{2} (\norm{z_{i}}_{2}^{2}+\norm{z_{j}}_{2}^{2} - 2 z_{i}^{T}z_{j})}+e^{-\frac{1}{2}(\norm{z_{i}}_{2}^{2}+\norm{z_{j}}_{2}^{2})}) + \frac{1}{2} \sum_{1 \leqslant i,j \leqslant N} \norm{z_{i}-z_{j}}_{2}^{2},\\
        & = \sum_{1 \leqslant i,j \leqslant N}  log(e^{-\frac{1}{2} (\norm{z_{i}}_{2}^{2}+\norm{z_{j}}_{2}^{2})} (1 + e^{z_{i}^{T}z_{j}})) + \frac{1}{2} \sum_{1 \leqslant i,j \leqslant N} \norm{z_{i}-z_{j}}_{2}^{2},\\
        & = \sum_{1 \leqslant i,j \leqslant N}  \bigg(log(e^{-\frac{1}{2} (\norm{z_{i}}_{2}^{2}+\norm{z_{j}}_{2}^{2})})+  log(1 + e^{z_{i}^{T}z_{j}}) \bigg)+ \frac{1}{2} \sum_{1 \leqslant i,j \leqslant N} \norm{z_{i}-z_{j}}_{2}^{2},\\
        & = \sum_{1 \leqslant i,j \leqslant N}  log(1 + e^{z_{i}^{T}z_{j}})+ \frac{1}{2} \sum_{1 \leqslant i,j \leqslant N} \norm{z_{i}-z_{j}}_{2}^{2} - \frac{1}{2} \sum_{1 \leqslant i,j \leqslant N} (\norm{z_{i}}_{2}^{2} + \norm{z_{j}}_{2}^{2}).\\
\end{split}
\end{equation*}

Thus \\ 

\begin{equation*} 
\begin{split}
        L_{bce}(\hat{A}(Z(\theta)), A^{self}) &= \frac{1}{2} \; \sum_{1\leqslant i,j \leqslant N} \; a_{ij}^{self} \norm{z_{i}-z_{j}}_{2}^{2} - \frac{1}{2} \; \sum_{1\leqslant i,j \leqslant N} \; a_{ij}^{self} (\norm{z_{i}}_{2}^{2} + \norm{z_{j}}_{2}^{2}) + \; \sum_{1\leqslant i,j \leqslant N} \; log(1+e^{z_{i}^{T}z_{j}}),\\ 
        &= L_{\mathcal{C}}(Z(\theta), A^{self}) + L_{\mathcal{R}}(Z(\theta), A^{self}). \\
\end{split}
\end{equation*}
\end{proof}

\section{Proof of Proposition 2} \label{Appendix_B}
\newtheorem*{P2}{Proposition~\ref{proposition_2}}
\begin{P2}
The k-means clustering loss applied to the embedded representations can be expressed as: 
\begin{equation*} 
\begin{split}
L_{clus}(Z(\theta)) = L_{\mathcal{C}}(Z(\theta), A^{clus}).
\end{split}
\end{equation*}
\end{P2}

\begin{proof}
\begin{equation*} 
\begin{split}
L_{clus}(Z(\theta)) & = \sum_{k=1}^{K} \sum_{i \in C_{k}^{clus}} \norm{z_{i} - \mu_{k}}_{2}^{2}, \\ 
& = \sum_{k=1}^{K} \sum_{i \in C_{k}^{clus}} \norm{ \frac{1}{\left | C_{k}^{clus} \right |}\sum_{j \in C_{k}^{clus}}z_{i} -\frac{1}{\left | C_{k}^{clus} \right |}\sum_{j \in C_{k}^{clus}} z_{j}}_{2}^{2}, \\
& = \sum_{k=1}^{K} \frac{1}{\left | C_{k}^{clus} \right |^{2}} \sum_{i \in C_{k}^{clus}} \norm{\sum_{j \in C_{k}^{clus}}(z_{i}-  z_{j})}_{2}^{2},\\
& = \sum_{k=1}^{K} \frac{1}{\left | C_{k}^{clus} \right |^{2}} \sum_{i \in C_{k}^{clus}} \bigg(\sum_{j \in C_{k}^{clus}}(z_{i}-  z_{j}) \bigg)^{T} \bigg(\sum_{j \in C_{k}^{clus}}(z_{i}-  z_{j})\bigg),\\
& = \sum_{k=1}^{K} \frac{1}{\left | C_{k}^{clus} \right |^{2}} \sum_{i \in C_{k}^{clus}} \bigg(\sum_{j \in C_{k}^{clus}}(z_{i}-  z_{j})^{T} \bigg) \bigg(\sum_{j' \in C_{k}^{clus}}(z_{i}-  z_{j'})\bigg),\\
& = \sum_{k=1}^{K} \frac{1}{\left | C_{k}^{clus} \right |^{2}} \left [ \sum_{i,j \in C_{k}^{clus}} (z_{i}-z_{j})^{T}(z_{i}-z_{j}) +  \sum_{\substack{i,j,j' \in C_{k}^{clus} \\ j \neq j', i \neq j, i \neq j'}} (z_{i}-z_{j})^{T}(z_{i}-z_{j'}) \right ],\\
& = \sum_{k=1}^{K} \frac{1}{\left | C_{k}^{clus} \right |^{2}} \left [ \sum_{i,j \in C_{k}^{clus}} \norm{z_{i}-z_{j}}_{2}^{2} +  \sum_{\substack{i,j,j' \in C_{k}^{clus} \\ j \neq j', i \neq j, i \neq j' }} (z_{i}^{T}z_{i}-z_{i}^{T}z_{j}-z_{i}^{T}z_{j'}+z_{j}^{T}z_{j'}) \right ],\\
& = \sum_{k=1}^{K} \frac{1}{\left | C_{k}^{clus} \right |^{2}} \left [ \sum_{i,j \in C_{k}^{clus}} \norm{z_{i}-z_{j}}_{2}^{2} + \frac{1}{2} \sum_{\substack{i,j,j' \in C_{k}^{clus} \\ j \neq j', i \neq j, i \neq j' }} \bigg( \norm{z_{i} - z_{j}}_{2}^{2} + \norm{z_{i} - z_{j'}}_{2}^{2} - \norm{z_{j} - z_{j'}}_{2}^{2} \bigg) \right ],\\
& = \sum_{k=1}^{K} \frac{1}{\left | C_{k}^{clus} \right |^{2}} \left [ \sum_{i,j \in C_{k}^{clus}} \norm{z_{i}-z_{j}}_{2}^{2} + \frac{1}{2} \sum_{\substack{i,j,j' \in C_{k}^{clus} \\ j \neq j', i \neq j, i \neq j' }}  \norm{z_{i} - z_{j}}_{2}^{2} + \frac{1}{2} \sum_{\substack{i,j,j \in C_{k}^{clus} \\ j \neq j', i \neq j, i \neq j' }}\norm{z_{i} - z_{j'}}_{2}^{2} - \frac{1}{2} \sum_{\substack{i,j,j \in C_{k}^{clus} \\ j \neq j', i \neq j, i \neq j' }}\norm{z_{j} - z_{j'}}_{2}^{2}  \right ] ,\\
& = \sum_{k=1}^{K} \frac{1}{\left | C_{k}^{clus} \right |^{2}} \left [ \sum_{i,j \in C_{k}^{clus}} \norm{z_{i}-z_{j}}_{2}^{2} + \frac{1}{2} \sum_{\substack{i,j,j \in C_{k}^{clus} \\ j \neq j', i \neq j, i \neq j' }}  \norm{z_{i} - z_{j}}_{2}^{2}  \right ],\\
& = \sum_{k=1}^{K} \frac{1}{\left | C_{k}^{clus} \right |^{2}} \left [ \sum_{i,j \in C_{k}^{clus}} \norm{z_{i}-z_{j}}_{2}^{2} + \frac{1}{2} \sum_{\substack{i,j \in C_{k}^{clus} }}  \sum_{\substack{j' \in C_{k}^{clus} \\ j'\neq i, j'\neq j}}\norm{z_{i} - z_{j}}_{2}^{2}  \right ],\\
& = \sum_{k=1}^{K} \frac{1}{\left | C_{k}^{clus} \right |^{2}} \left [ \sum_{i,j \in C_{k}^{clus}} \norm{z_{i}-z_{j}}_{2}^{2} + \frac{\left | C_{k}^{clus} \right | - 2}{2} \sum_{\substack{i,j \in C_{k}^{clus} }}  \norm{z_{i} - z_{j}}_{2}^{2}  \right ],\\
& = \sum_{k=1}^{K} \frac{1}{2 \left | C_{k}^{clus} \right |}  \sum_{i,j \in C_{k}^{clus}} \norm{z_{i}-z_{j}}_{2}^{2},\\
\end{split}
\end{equation*}

Let the matrix $A^{clus} = (a_{ij}^{clus})_{1 \leqslant i,j \leqslant N} \in \mathbb{R}^{N \times N}$ be defined as $a_{ij}^{clus} = \left\{\begin{matrix}
\frac{1}{\left | C_{k}^{clus} \right |} \:\: \text{if} \:\: \exists\: k \text{ s.th } i,j \in C_{k}^{clus}  \\ 
0 \:\: \text{otherwise.}
\end{matrix}\right.$

Hence, $L_{clus}(Z(\theta)) = \frac{1}{2}\sum_{1 \leqslant i,j \leqslant N}  a_{ij}^{clus} \; \norm{z_{i}-z_{j}}_{2}^{2} = L_{\mathcal{C}}(Z(\theta), A^{clus}).$
\end{proof}

\section{Proof of Theorem 1} \label{Appendix_C}

\newtheorem*{T1}{Theorem~\ref{theorem_1}}
\begin{T1}
The linear combination between reconstruction and embedded k-means for a GAE model can be expressed as:
\begin{equation*} 
\begin{split}
L_{clus}(Z(\theta)) & + \;  \; \gamma \; L_{bce}(\hat{A}(Z(\theta)), \,A^{self})) = L_{\mathcal{C}}(Z(\theta), A^{clus} + \gamma A^{self} ) + \; \gamma \;  L_{\mathcal{R}}(Z(\theta), A^{self}).\\
\end{split}
\end{equation*}
\end{T1}

\begin{proof}
Based on Proposition \ref{proposition_1} and Proposition \ref{proposition_2}, we can conclude that
\begin{equation*} 
\begin{split}
L_{clus}(Z(\theta)) + \gamma L_{bce}(\hat{A}(Z(\theta)), \,A^{self})) &= \frac{1}{2} \sum_{1 \leqslant i,j \leqslant N}  a_{ij}^{clus}  \; \norm{z_{i}-z_{j}}_{2}^{2} + \frac{\gamma }{2}\sum_{1 \leqslant i,j \leqslant N} a_{ij}^{self} \; \norm{z_{i}-z_{j}}_{2}^{2} \;  \\
& - \frac{\gamma}{2}\sum_{1 \leqslant i,j \leqslant N}  a_{ij}^{self} (\norm{z_{i}}_{2}^{2} + \norm{z_{j}}_{2}^{2})  + \gamma \sum_{1 \leqslant i,j \leqslant N}  log(1 + exp(z_{i}^{T} z_{j})), \\ 
& = \frac{1}{2}\sum_{1 \leqslant i,j \leqslant N}  (a_{ij}^{clus} + \gamma a_{ij}^{self}) \; \norm{z_{i}-z_{j}}_{2}^{2} \; \\
&  - \frac{\gamma}{2}\sum_{1 \leqslant i,j \leqslant N}  a_{ij}^{self} (\norm{z_{i}}_{2}^{2} + \norm{z_{j}}_{2}^{2}) \; + \gamma \sum_{1 \leqslant i,j \leqslant N}  log(1 + exp(z_{i}^{T} z_{j})), \\
& =  L_{\mathcal{C}}(Z(\theta), A^{clus} + \gamma A^{self} ) + \; \gamma \;  L_{\mathcal{R}}(Z(\theta), A^{self}). \\
\end{split}
\end{equation*}
\end{proof}

\section{Proof of Proposition 3} \label{Appendix_D}

\newtheorem*{P3}{Proposition~\ref{proposition_3}}
\begin{P3}
The gradient of the reconstruction loss $L_{bce}(\hat{A}(Z(\theta)), A)$ w.r.t. the embedded representation $z_{i}$ can be expressed as:
\begin{equation*} 
\begin{split}
\frac{\partial L_{bce}(\hat{A}(Z(\theta)), A^{self})}{\partial z_{i}} = \sum_{1 \leqslant j \leqslant N} (\hat{a}_{ij}-a_{ij}^{self}) z_{j}.
\end{split}
\end{equation*}
\end{P3}

\begin{proof}
\begin{equation*}
\begin{split}
\frac{\partial L_{bce}(\hat{A}(Z(\theta)), A^{self})}{\partial z_{i}} &=  \frac{\partial \bigg( \frac{1}{2} \; \sum\limits_{1\leqslant i,j \leqslant N} \; a_{ij}^{self} \norm{z_{i}-z_{j}}_{2}^{2} - \frac{1}{2} \; \sum\limits_{1\leqslant i,j \leqslant N} \; a_{ij}^{self} (\norm{z_{i}}_{2}^{2} + \norm{z_{j}}_{2}^{2}) + \; \sum\limits_{1\leqslant i,j \leqslant N} \; log(1+e^{z_{i}^{T}z_{j}})\bigg)}{\partial z_{i}}, \\
&= \frac{1}{2} \sum_{1\leqslant i,j \leqslant N} a_{ij}^{self} \frac{\partial (z_{i}-z_{j})^{T}(z_{i}-z_{j}) }{\partial z_{i}} - \frac{1}{2} \sum_{1\leqslant i,j \leqslant N} a_{ij}^{self} \bigg(\frac{\partial z_{i}^{T}z_{i}}{\partial z_{i}} +  \frac{\partial z_{j}^{T}z_{j}}{\partial z_{i}} \bigg) +  \sum_{1\leqslant i,j \leqslant N} \frac{\partial log(1+e^{z_{i}^{T}z_{j}})}{\partial z_{i}}, \\
&= \frac{1}{2} \sum_{1\leqslant i,j \leqslant N} \; a_{ij}^{self} \frac{\partial (z_{i}^{T}z_{i} -2 z_{i}^{T}z_{j} +  z_{j}^{T}z_{j})}{\partial z_{i}} - \frac{1}{2} \sum_{1\leqslant i,j \leqslant N} a_{ij}^{self} \bigg(\frac{\partial z_{i}^{T}z_{i}}{\partial z_{i}} + \frac{\partial z_{j}^{T}z_{j}}{\partial z_{i}} \bigg) + \sum_{1\leqslant i,j \leqslant N} \frac{\partial log(1+e^{z_{i}^{T}z_{j}})}{\partial z_{i}}, \\
&= - \sum_{1\leqslant i,j \leqslant N} \; a_{ij}^{self} \frac{\partial z_{i}^{T}z_{j}}{\partial z_{i}} \; + \; \sum_{1\leqslant i,j \leqslant N} \;  \frac{\partial log(1+e^{z_{i}^{T}z_{j}})}{\partial z_{i}}, \\
&= - \sum_{1\leqslant i,j \leqslant N} \; a_{ij}^{self} z_{j}\; + \; \sum_{1\leqslant i,j \leqslant N} \; \frac{e^{z_{i}^{T}z_{j}}}{1 + e^{z_{i}^{T}z_{j}}} z_{j}, \\ 
&= - \sum_{1\leqslant i,j \leqslant N} \; a_{ij}^{self} z_{j}\; + \; \sum_{1\leqslant i,j \leqslant N} \; \frac{1}{1 + e^{-z_{i}^{T}z_{j}}} z_{j}, \\ 
&=  - \sum_{1\leqslant i,j \leqslant N} \; a_{ij}^{self} z_{j}\; + \; \sum_{1\leqslant i,j \leqslant N} \; \text{Sigmoid }(z_{i}^{T}z_{j})  z_{j}, \\ 
&= - \sum_{1\leqslant i,j \leqslant N} \; a_{ij}^{self} \; z_{j}\; + \; \sum_{1\leqslant i,j \leqslant N} \; \hat{a}_{ij} \; z_{j}, \\ 
&= \sum_{1\leqslant i,j \leqslant N} \; (\hat{a}_{ij} - a_{ij}^{self}) \; z_{j}. \\ 
\end{split}
\end{equation*}
\end{proof}

\section{Proof of Proposition 4} \label{Appendix_E}

\newtheorem*{P4}{Proposition~\ref{proposition_4}}
\begin{P4}
The gradient of the clustering loss $L_{clus}(Z(\theta))$ w.r.t. the embedded representation $z_{i}$ can be expressed as:
\begin{equation*} 
\begin{split}
\frac{\partial L_{clus}(Z(\theta))}{\partial z_{i}} = \sum_{1 \leqslant j \leqslant N} a_{ij}^{clus} (z_{i} - z_{j}).
\end{split}
\end{equation*}
\end{P4}

\begin{proof}
\begin{equation*}
\begin{split}
\frac{\partial L_{clus}(Z(\theta))}{\partial z_{i}} &= \frac{\partial \bigg(\frac{1}{2} \; \sum\limits_{1\leqslant i,j \leqslant N} \; a_{ij}^{clus} \norm{  z_{i}-z_{j}}_{2}^{2}\bigg)}{\partial z_{i}}, \\
&= \frac{1}{2} \; \sum_{1\leqslant i,j \leqslant N} \; a_{ij}^{clus} \frac{\partial (z_{i}-z_{j})^{T}(z_{i}-z_{j}) }{\partial z_{i}}, \\
&= \frac{1}{2} \; \sum_{1\leqslant i,j \leqslant N} \; a_{ij}^{clus} \frac{\partial (z_{i}^{T}z_{i} -2 z_{i}^{T}z_{j} +  z_{j}^{T}z_{j})}{\partial z_{i}}, \\
&= \frac{1}{2} \; \sum_{1\leqslant i,j \leqslant N} \; a_{ij}^{clus} (2 z_{i} - 2 z_{j}), \\
&= \sum_{1\leqslant i,j \leqslant N} \; a_{ij}^{clus} (z_{i} - z_{j}).
\end{split}
\end{equation*}
\end{proof}

\section{Proof of Theorem 2} \label{Appendix_F}

\newtheorem*{T2}{Theorem~\ref{theorem_2}}
\begin{T2} 
Given two GAE models $\mathcal{Q}_{1}$ and $\mathcal{Q}_{2}$, which have the same GCN architecture and weights. $\mathcal{Q}_{1}$ optimizes the objective function in Equation (\ref{Q_11}) and $\mathcal{Q}_{2}$ minimizes the loss function in Equation (\ref{Q_21}), where $f \in \mathcal{NN}(d, d', L)$ and $d' \ll d$. Let $\tau_{1}^{*}$ be the Lipschitz constant of $f$, $\bar{Z}_{i} = (z_{jj'} - z_{ij'})_{j,j'} \in \mathbb{R}^{N \times d}$, $\zeta_{i} = (\left\| z_{j} - z_{i} \right\|_{2})_{j} \in \mathbb{R}^{N}$, and $a_{i}$ is the $i^{th}$ row of A. 
\begin{equation}  \label{Q_11}
L_{\mathcal{Q}_{1}} = L_{clus}(Z(\theta)) + \gamma L_{bce}(\hat{A}(Z(\theta)), \,A^{self}),
\end{equation}
\begin{equation}  \label{Q_21}
L_{\mathcal{Q}_{2}} = L_{clus}(f(Z(\theta))) + \gamma L_{bce}(\hat{A}(Z(\theta)), \,A^{self}).
\end{equation}
\begin{itemize}
  \setlength\itemsep{1em}
  \item $   \Lambda'_{FD}(\mathcal{Q}_{2}, z_{i}) = \Lambda'_{FD}(\mathcal{Q}_{1}, z_{i}).$
  \item $ \text{If }$ $$\tau_{1}^{*} \leqslant \sqrt{\frac{(\bar{Z}_{i}^{T} a_{i}^{sup})^{T}  (\bar{Z}_{i}^{T} a_{i}^{clus})}{(\zeta_{i}^{T} \: a_{i}^{sup})(\zeta_{i}^{T}a_{i}^{clus})}},$$ 
$\text{then } \;\;$  $$\Lambda'_{FR}(\mathcal{Q}_{2}, z_{i}) \leqslant \Lambda'_{FR}(\mathcal{Q}_{1}, z_{i}).$$
\end{itemize}
\end{T2}

\begin{proof}
\begin{equation*}
\begin{split}
\Lambda'_{FD}(\mathcal{Q}_{2}, z_{i}) - \Lambda'_{FD}(\mathcal{Q}_{1}, z_{i}) &=  \inp*{\frac{\partial \sum_{j}  \tilde{a}_{ij}^{self} \; \norm{z_{i}-z_{j}}_{2}^{2}}{\partial z_{i}}}{\frac{\partial \sum_{j}  a_{ij}^{sup} \; \norm{z_{i} - z_{j}}_{2}^{2}}{\partial z_{i}}} -  \inp*{\frac{\partial \sum_{j}  \tilde{a}_{ij}^{self} \; \norm{z_{i}-z_{j}}_{2}^{2}}{\partial z_{i}}}{\frac{\partial \sum_{j}  a_{ij}^{sup} \; \norm{z_{i} - z_{j}}_{2}^{2}}{\partial z_{i}}} \\
&= 0. \\
\end{split}
\end{equation*}

\begin{equation*}
\begin{split}
\Lambda'_{FR}(\mathcal{Q}_{2}, z_{i}) - \Lambda'_{FR}(\mathcal{Q}_{1}, z_{i}) &=  \inp*{\frac{\partial \sum_{j}  a_{ij}^{clus} \; \norm{f(z_{i})-f(z_{j})}_{2}^{2}}{\partial z_{i}}}{\frac{\partial \sum_{j}  a_{ij}^{sup} \; \norm{f(z_{i})-f(z_{j})}_{2}^{2}}{\partial z_{i}}} 
\\ & -  \inp*{\frac{\partial \sum_{j}  a_{ij}^{clus} \; \norm{z_{i}-z_{j}}_{2}^{2}}{\partial z_{i}}}{\frac{\partial \sum_{j}  a_{ij}^{sup} \; \norm{z_{i}-z_{j}}_{2}^{2}}{\partial z_{i}}}, \\
& = \sum_{j,j'} a_{ij}^{clus}a_{ij'}^{sup}(f(z_{i})-f(z_{j}))^{T}(f(z_{i})-f(z_{j'})) - \sum_{j,j'} a_{ij}^{clus}a_{ij'}^{sup}(z_{i}-z_{j})^{T}(z_{i}-z_{j'}), \\
& \leqslant  \sum_{j,j'} a_{ij}^{clus}a_{ij'}^{sup}\norm{f(z_{i})-f(z_{j})}_{2} \norm{f(z_{i})-f(z_{j'})}_{2} - \sum_{j,j'} a_{ij}^{clus}a_{ij'}^{sup}(z_{i}-z_{j})^{T}(z_{i}-z_{j'}), \\
& \leqslant (\tau_{1}^{*})^{2} \; \sum_{j,j'} a_{ij}^{clus}a_{ij'}^{sup}\norm{z_{i}-z_{j}}_{2} \norm{z_{i}-z_{j'}}_{2} - \sum_{j,j'} a_{ij}^{clus}a_{ij'}^{sup}(z_{i}-z_{j})^{T}(z_{i}-z_{j'}), \\
\end{split}
\end{equation*}

\begin{equation*}
\begin{split}
\Lambda'_{FR}(\mathcal{Q}_{2}, z_{i}) - \Lambda'_{FR}(\mathcal{Q}_{1}, z_{i}) & \leqslant (\tau_{1}^{*})^{2} \; \sum_{j,j'} a_{ij}^{clus}a_{ij'}^{sup}\norm{z_{i}-z_{j}}_{2} \norm{z_{i}-z_{j'}}_{2} - \sum_{j,j'} a_{ij}^{clus}a_{ij'}^{sup}(z_{i}-z_{j})^{T}(z_{i}-z_{j'}), \\
& \leqslant \sum_{j,j'} a_{ij}^{clus}a_{ij'}^{sup} \big((\tau_{1}^{*})^{2} \norm{z_{i}-z_{j}}_{2} \norm{z_{i}-z_{j'}}_{2}  - (z_{i}-z_{j})^{T}(z_{i}-z_{j'})\big). \\
\end{split}
\end{equation*}

We have
\begin{equation*}
\begin{split}
\tau_{1}^{*} \leqslant \sqrt{\frac{(\bar{Z}_{i}^{T} a_{i}^{sup})^{T}  (\bar{Z}_{i}^{T} a_{i}^{clus})}{(\zeta_{i}^{T} \: a_{i}^{sup})(\zeta_{i}^{T}a_{i}^{clus})}} 
\implies (\tau_{1}^{*})^{2} & \leqslant \frac{(\bar{Z}_{i}^{T} a_{i}^{sup})^{T} (\bar{Z}_{i}^{T} a_{i}^{clus})}{(\zeta_{i}^{T} \: a_{i}^{sup})(\zeta_{i}^{T}a_{i}^{clus})}, \\
& \leqslant \frac{(a_{i}^{sup})^{T} \bar{Z}_{i}\bar{Z}_{i}^{T}(a_{i}^{clus})}{(a_{i}^{sup})^{T}\zeta_{i}\zeta_{i}^{T}(a_{i}^{clus})}, \\
& \leqslant \frac{tr\big((a_{i}^{sup})^{T}  \bar{Z}_{i}\bar{Z}_{i}^{T} (a_{i}^{clus})\big)}{tr\big((a_{i}^{sup})^{T}\zeta_{i}\zeta_{i}^{T}(a_{i}^{clus})\big)}, \\
& \leqslant \frac{tr\big((\bar{Z}_{i}\bar{Z}_{i}^{T})^{T} a_{i}^{sup} (a_{i}^{clus})^{T} \big)}{tr\big((\zeta_{i} \zeta_{i}^{T})^{T} a_{i}^{sup} (a_{i}^{clus})^{T} \big)}, \\
& \leqslant \frac{\sum\limits_{j,j'} \big((a_{i}^{sup}(a_{i}^{clus})^{T}) \circ (\bar{Z}_{i}\bar{Z}_{i}^{T})\big)_{jj'}}{\sum\limits_{j,j'}\big((a_{i}^{sup}(a_{i}^{clus})^{T}) \circ (\zeta_{i}\zeta_{i}^{T})\big)_{jj'}}, \\
& \leqslant \frac{\sum\limits_{j,j'} a_{ij}^{sup}a_{ij'}^{clus} (z_{i}-z_{j})^{T}(z_{i}-z_{j'})}{\sum\limits_{j,j'}a_{ij}^{sup}a_{ij'}^{clus} \norm{z_{i}-z_{j}}_{2}\norm{z_{i}-z_{j'}}_{2}}. \\
\end{split}
\end{equation*}

Thus
\begin{equation*}
\begin{split}
& (\tau_{1}^{*})^{2} \sum_{j,j'}a_{ij}^{sup}a_{ij'}^{clus} \norm{z_{i}-z_{j}}_{2} \norm{z_{i}-z_{j'}}_{2} \leqslant  \sum_{j,j'} a_{ij}^{sup}a_{ij'}^{clus} (z_{i}-z_{j})^{T}(z_{i}-z_{j'}), \\
& \implies \sum_{j,j'} a_{ij}^{clus}a_{ij'}^{sup} \big((\tau_{1}^{*})^{2} \norm{z_{i}-z_{j}}_{2} \norm{z_{i}-z_{j'}}_{2}  - (z_{i}-z_{j})^{T}(z_{i}-z_{j'})\big) \leqslant 0, \\
& \implies \Lambda'_{FR}(\mathcal{Q}_{2}, z_{i}) \leqslant \Lambda'_{FR}(\mathcal{Q}_{1}, z_{i}). 
\end{split}
\end{equation*}

\end{proof}

\section{Proof of Theorem 3} \label{Appendix_G}

\newtheorem*{T3}{Theorem~\ref{theorem_3}}
\begin{T3} 
Given two GAE models $\mathcal{Q}_{1}$ and $\mathcal{Q}_{2}$, which have the same GCN architecture and weights. $\mathcal{Q}_{1}$ optimizes the objective function in Equation (\ref{Q_31}) and $\mathcal{Q}_{2}$ minimizes the loss function in Equation (\ref{Q_41}), where $f \in \mathcal{NN}(d, d', L)$ an injective function and $d' \gg d$. Let $\tau_{2}^{*}$ be the Lipschitz constant of $f^{-1}:f(\mathbb{R}^{d}) \rightarrow \mathbb{R}^{d}$, $\bar{Z}_{i}^{'} = ((f(z_{j}))_{j'} - (f(z_{i}))_{j'})_{j,j'} \in \mathbb{R}^{N \times d'}$, $\zeta_{i}^{'} = (\left\| f(z_{j}) - f(z_{i}) \right\|_{2})_{j} \in \mathbb{R}^{N} $, and $a_{i}$ is the $i^{th}$ row of A.

\begin{equation}  \label{Q_31}
L_{\mathcal{Q}_{1}} = L_{clus}(Z(\theta)) + \gamma L_{bce}(\hat{A}(Z(\theta)), \,A^{self}),
\end{equation}
\begin{equation}  \label{Q_41}
L_{\mathcal{Q}_{2}} = L_{clus}(Z(\theta)) + \gamma L_{bce}(\hat{A}(f(Z(\theta))), \,A^{self}).
\end{equation}

\vspace{5mm}

\begin{itemize}
  \setlength\itemsep{1em}
  \item $ \Lambda'_{FR}(\mathcal{Q}_{2}, z_{i}) = \Lambda'_{FR}(\mathcal{Q}_{1}, z_{i}).$
  \item $ \text{If }$ $$\tau_{2}^{*} \leqslant  \sqrt{\frac{(\bar{Z}_{i}^{'T} a_{i}^{sup})^{T}  (\bar{Z}_{i}^{'T} \tilde{a}_{i}^{self})}{(\zeta_{i}^{'T} \: a_{i}^{sup})(\zeta_{i}^{'T}\tilde{a}_{i}^{self})}},$$ 
  $\text{then } \;\;$  $$\Lambda'_{FD}(\mathcal{Q}_{2}, z_{i}) \geqslant \Lambda'_{FD}(\mathcal{Q}_{1}, z_{i}).$$
\end{itemize}
\end{T3}

\begin{proof}
\begin{equation*}
\begin{split}
\Lambda'_{FR}(\mathcal{Q}_{2}, z_{i}) - \Lambda'_{FR}(\mathcal{Q}_{1}, z_{i}) &=  \inp*{\frac{\partial \sum_{j}  a_{ij}^{clus} \; \norm{z_{i}-z_{j}}_{2}^{2}}{\partial z_{i}}}{\frac{\partial \sum_{j} a_{ij}^{sup} \; \norm{z_{i} - z_{j}}_{2}^{2}}{\partial z_{i}}} -  \inp*{\frac{\partial \sum_{j} a_{ij}^{clus} \; \norm{z_{i}-z_{j}}_{2}^{2}}{\partial z_{i}}}{\frac{\partial \sum_{j} a_{ij}^{sup} \; \norm{z_{i} - z_{j}}_{2}^{2}}{\partial z_{i}}} \\
&= 0.\\
\end{split}
\end{equation*}

\begin{equation*}
\begin{split}
\Lambda'_{FD}(\mathcal{Q}_{1}, z_{i}) - \Lambda'_{FD}(\mathcal{Q}_{2}, z_{i}) &= 
\inp*{\frac{\partial \sum_{j} \tilde{a}_{ij}^{self} \; \norm{z_{i}-z_{j}}_{2}^{2}}{\partial z_{i}}}{\frac{\partial \sum_{j}  a_{ij}^{sup} \; \norm{z_{i}-z_{j}}_{2}^{2}}{\partial z_{i}}} \\ 
& - \inp*{\frac{\partial \sum_{j} \tilde{a}_{ij}^{self} \; \norm{f(z_{i})-f(z_{j})}_{2}^{2}}{\partial z_{i}}}{\frac{\partial \sum_{j}  a_{ij}^{sup} \; \norm{f(z_{i})-f(z_{j})}_{2}^{2}}{\partial z_{i}}}, \\
& = \sum_{j,j'} \tilde{a}_{ij}^{self} a_{ij'}^{sup}(z_{i}-z_{j})^{T}(z_{i}-z_{j'}) - \sum_{j,j'} \tilde{a}_{ij}^{self}a_{ij'}^{sup} (f(z_{i})-f(z_{j}))^{T}(f(z_{i})-f(z_{j'})), \\
& \leqslant  \sum_{j,j'} \tilde{a}_{ij}^{self} a_{ij'}^{sup} \norm{z_{i}-z_{j}}_{2} \norm{z_{i}-z_{j'}}_{2} - \sum_{j,j'} \tilde{a}_{ij}^{self} a_{ij'}^{sup} (f(z_{i})-f(z_{j}))^{T}(f(z_{i})-f(z_{j'})), \\
& \leqslant (\tau_{2}^{*})^{2} \; \sum_{j,j'} \tilde{a}_{ij}^{self} a_{ij'}^{sup} \norm{f(z_{i})-f(z_{j})}_{2} \norm{f(z_{i})-f(z_{j'})}_{2} - \sum_{j,j'} \tilde{a}_{ij}^{self} a_{ij'}^{sup} (f(z_{i})-f(z_{j}))^{T}(f(z_{i})-f(z_{j'})), \\
& \leqslant \sum_{j,j'} \tilde{a}_{ij}^{self} a_{ij'}^{sup} \big((\tau_{2}^{*})^{2} \norm{f(z_{i})-f(z_{j})}_{2} \norm{f(z_{i})-f(z_{j'})}_{2} - (f(z_{i})-f(z_{j}))^{T}(f(z_{i})-f(z_{j'}))\big), \\
\end{split}
\end{equation*}

We have
\begin{equation*}
\begin{split}
\tau_{2}^{*} \leqslant \sqrt{\frac{(\bar{Z}_{i}^{'T} a_{i}^{sup})^{T}  (\bar{Z}_{i}^{'T} \tilde{a}_{i}^{self})}{(\zeta_{i}^{'T} \: a_{i}^{sup})(\zeta_{i}^{'T}\tilde{a}_{i}^{self})}} 
\implies (\tau^{*})^{2} & \leqslant \frac{(\bar{Z}_{i}^{'T} a_{i}^{sup})^{T} (\bar{Z}_{i}^{'T} \tilde{a}_{i}^{self})}{(\zeta_{i}^{'T} \: a_{i}^{sup})(\zeta_{i}^{'T} \tilde{a}_{i}^{self})}, \\
& \leqslant \frac{(a_{i}^{sup})^{T} \bar{Z}_{i}^{'}\bar{Z}_{i}^{'T}(\tilde{a}_{i}^{self})}{(a_{i}^{sup})^{T}\zeta_{i}^{'} \zeta_{i}^{'T}(\tilde{a}_{i}^{self})}, \\
& \leqslant \frac{tr\big((a_{i}^{sup})^{T}  \bar{Z}_{i}^{'}\bar{Z}_{i}^{'T} (\tilde{a}_{i}^{self})\big)}{tr\big((a_{i}^{sup})^{'T}\zeta_{i}^{'}\zeta_{i}^{'T}(\tilde{a}_{i}^{self})\big)}, \\
& \leqslant \frac{tr\big((\bar{Z}_{i}^{'}\bar{Z}_{i}^{'T})^{T} a_{i}^{sup} (\tilde{a}_{i}^{self})^{T} \big)}{tr\big((\zeta_{i}^{'} \zeta_{i}^{'T})^{T} a_{i}^{sup} (\tilde{a}_{i}^{self})^{T} \big)}, \\
& \leqslant \frac{\sum\limits_{j,j'} \big((a_{i}^{sup}(\tilde{a}_{i}^{self})^{T}) \circ (\bar{Z}_{i}^{'}\bar{Z}_{i}^{'T})\big)_{jj'}}{\sum\limits_{j,j'}\big((a_{i}^{sup}(\tilde{a}_{i}^{self})^{T}) \circ (\zeta_{i}^{'}\zeta_{i}^{'T})\big)_{jj'}}, \\
& \leqslant \frac{\sum\limits_{j,j'} a_{ij}^{sup}\tilde{a}_{ij'}^{self} (f(z_{i})-f(z_{j}))^{T}(f(z_{i})-f(z_{j'}))}{\sum\limits_{j,j'}a_{ij}^{sup}\tilde{a}_{ij'}^{self} \norm{f(z_{i})-f(z_{j})}_{2} \norm{f(z_{i})-f(z_{j'})}_{2}}. \\
\end{split}
\end{equation*}

Thus
\begin{equation*}
\begin{split}
& (\tau_{2}^{*})^{2} \sum_{j,j'}a_{ij}^{sup}\tilde{a}_{ij'}^{self} \norm{f(z_{i})-f(z_{j})}_{2} \norm{f(z_{i})-f(z_{j'})}_{2} \leqslant  \sum_{j,j'} a_{ij}^{sup}\tilde{a}_{ij'}^{self} (f(z_{i})-f(z_{j}))^{T}(f(z_{i})-f(z_{j'})), \\
& \implies \sum_{j,j'} \tilde{a}_{ij}^{self}a_{ij'}^{sup} \big((\tau_{2}^{*})^{2} \norm{f(z_{i})-f(z_{j})}_{2} \norm{f(z_{i})-f(z_{j'})}_{2}  - (f(z_{i})-f(z_{j}))^{T}(f(z_{i})-f(z_{j'}))\big) \leqslant 0, \\
& \implies \Lambda'_{FD}(\mathcal{Q}_{1}, z_{i}) \leqslant \Lambda'_{FD}(\mathcal{Q}_{2}, z_{i}). 
\end{split}
\end{equation*}
\end{proof}

\section{Proof of Theorem 4} \label{Appendix_H}

\newtheorem*{T4}{Theorem~\ref{theorem_4}}
\begin{T4}
Given two models $\mathcal{Q}_{1}$ and $\mathcal{Q}_{2}$, which optimize the same objective function as described by Equation \ref{Q_51}. $\mathcal{Q}_{1}$ has a single fully-connected encoding layer characterized by the function $f_{1}(X)=ReLU(XW)$, where $W \in \mathbb{R}^{d \times d'}$ represents the learning weights of this layer. $\mathcal{Q}_{2}$ has a single graph convolutional layer characterized by the function $f_{2}(X)=ReLU(\Tilde{A}^{self}XW)$.

\begin{equation}  \label{Q_51}
L_{\mathcal{Q}_{1}} = L_{\mathcal{Q}_{2}} = L_{clus}(Z(\theta)) + \gamma L_{bce}(\hat{A}(Z(\theta)), \,A^{self}).
\end{equation}

\noindent Under Assumption \ref{assumption_1} and Assumption \ref{assumption_2}, we have:
$$\text{If} \;\;\; \mathcal{P}(f_{1}(x_{i})) \geqslant 0 \;\;\; \text{then} \;\;\;  \Lambda'_{FD}(\mathcal{Q}_{2}, x_{i}) \leqslant \Lambda'_{FD}(\mathcal{Q}_{1}, x_{i}).$$
\end{T4} 

\newpage

\begin{proof}
Let $h$ be an aggregation function such that $h^{sup}(x_{i}) = \sum_{j}a_{ij}^{sup} x_{j}$, and $h^{self}(x_{i}) = \sum_{j}\tilde{a}_{ij}^{self} x_{j}$. Let the functions $\mathcal{D}_{1}$ and $\mathcal{D}_{2}$ be distance metrics such that $\mathcal{D}_{1}^{sup}(x_{i})= \frac{1}{2}\sum_{j}a_{ij}^{sup} \norm{x_{i}-x_{j}}_{2}^{2}$, $\mathcal{D}_{1}^{self}(x_{i})= \frac{1}{2}\sum_{j}\tilde{a}_{ij}^{self} \norm{x_{i}-x_{j}}_{2}^{2}$, and $\mathcal{D}_{2}(x_{i})= \frac{1}{2}\sum_{j,j'}\tilde{a}_{ij}^{self}a_{ij}^{sup} \norm{x_{j}-x_{j'}}_{2}^{2}$. We give three lemmas before proving Theorem \ref{theorem_4}.

\begin{lemma} \label{lemma_1} 
$$\forall a, b \in \mathbb{R}^{d} \;\;\; \text{if} \;\;\; Sign(a)=Sign(b) \;\;\; \text{then} \;\;\; Sign(a)=Sign(a+b).$$
\end{lemma}
\begin{proof} 

\begin{equation*}
\begin{split}
\forall m \in \left [|1, d | \right ] \; \; \; Sign(a_{m} + b_{m}) &= \frac{a_{m} + b_{m}}{\abs{a_{m} + b_{m}}}, \\
&= \frac{a_{m}}{\abs{a_{m}}} \; \frac{\abs{a_{m}}}{\abs{a_{m} + b_{m}}} + \frac{b_{m}}{\abs{b_{m}}} \; \frac{\abs{b_{m}}}{\abs{a_{m} + b_{m}}}, \\
&= Sign(a_{m}) \; \frac{\abs{a_{m}}}{\abs{a_{m} + b_{m}}} + Sign(b_{m}) \; \frac{\abs{b_{m}}}{\abs{a_{m} + b_{m}}}, \\
&= Sign(a_{m}) \; \frac{\abs{a_{m}}}{\abs{a_{m} + b_{m}}} + Sign(a_{m}) \; \frac{\abs{b_{m}}}{\abs{a_{m} + b_{m}}}, \; \; \; (Sign(a) = Sign(b) \implies Sign(a_{m}) = Sign(b_{m})), \\
&= Sign(a_{m}) \; \big( \frac{\abs{a_{m}} + \abs{b_{m}}}{\abs{a_{m} + b_{m}}} \big).
\end{split}
\end{equation*}


$$ \text{If} \;\; a_{m} \geqslant 0 \;\; \text{then} \;\; \frac{\abs{a_{m}} + \abs{b_{m}}}{\abs{a_{m} + b_{m}}} = \frac{a_{m} + b_{m}}{a_{m} + b_{m}} = 1.$$
$$ \;\;\;\; \text{Else if} \;\; a_{m} \leqslant 0 \;\; \text{then} \;\; \frac{\abs{a_{m}} + \abs{b_{m}}}{\abs{a_{m} + b_{m}}} = \frac{-a_{m} -b_{m}}{-(a_{m} + b_{m})} = 1.$$
$$\implies \big( \frac{\abs{a_{m}} + \abs{b_{m}}}{\abs{a_{m} + b_{m}}} \big) = 1, $$
$$\implies Sign(a_{m} + b_{m}) = Sign(a_{m}), $$
$$\implies Sign(a + b) = Sign(a). $$
\end{proof}

\begin{lemma} \label{lemma_2} 
$$\forall i \in \left [|1, N | \right ] \; \; \; f_{1}(h^{self}(x_{i})) = h^{self}(f_{1}(x_{i})).$$
\end{lemma}

\begin{proof}

\vspace{3mm}
\noindent Based on Assumption \ref{assumption_2}, we have
$$ \forall j  \in \left [|1, N | \right ] \; \; \; \text{such that} \; \; \tilde{a}_{ij}^{self} \neq 0, \; \; \; Sign(W^{T} x_{i}) = Sign(W^{T} x_{j}).$$

\vspace{3mm}
\noindent Applying Lemma \ref{lemma_1}, we obtain
\begin{equation*}
\begin{split}
Sign(W^{T} x_{i}) &= Sign(W^{T} x_{j}), \\
&= Sign(\sum_{j} \tilde{a}_{ij}^{self} W^{T} x_{j}),   \;\;\;\; (\tilde{a}_{ij}^{self} \geq 0 \; \text{thus it will not affect the sign})\\
&= Sign(W^{T} \sum_{j} \tilde{a}_{ij}^{self} x_{j}).
\end{split}
\end{equation*}

\vspace{3mm}
\noindent On one hand, we have
\begin{equation*}
\begin{split}
f_{1}(h^{self}(x_{i})) &= f_{1}(\sum_{j} \tilde{a}_{ij}^{self}x_{j}), \\
&= Diag(Sign(W^{T} \sum_{j} \tilde{a}_{ij}^{self} x_{j})) \: W^{T} \: \sum_{j} \tilde{a}_{ij}^{self} x_{j}, \\
&= Diag(Sign(W^{T} x_{i})) \; W^{T} \; \sum_{j} \tilde{a}_{ij}^{self} x_{j}. \\
\end{split}
\end{equation*}

\vspace{3mm}
\noindent On the other hand, we have
\begin{equation*}
\begin{split}
h^{self}(f_{1}(x_{i})) &=  \sum_{j} \tilde{a}_{ij}^{self} \; f_{1}(x_{j}), \\
&= \sum_{j} \tilde{a}_{ij}^{self} \; Diag(Sign(W^{T}x_{j})) \; W^{T} \; x_{j}, \\
&= \sum_{j} \tilde{a}_{ij}^{self} \; Diag(Sign(W^{T}x_{i})) \; W^{T} \; x_{j},  \;\;\;\; (\text{based on Assumption \ref{assumption_2}}) \\
&= Diag(Sign(W^{T} x_{i})) \; W^{T} \; \sum_{j} \tilde{a}_{ij}^{self} x_{j}.
\end{split}
\end{equation*}

\vspace{3mm}
\noindent  We conclude that 
$$f_{1}(h^{self}(x_{i})) = h^{self}(f_{1}(x_{i})).$$
\end{proof}

\begin{lemma} \label{lemma_3} 
$$\forall i \in \left [|1, N | \right ] \; \; \; x_{i} \approx h^{self}(x_{i}).$$
\end{lemma}

\begin{proof}

\noindent Based on Assumption \ref{assumption_1},
\begin{equation*}
\begin{split}
\forall j \in \left [|1, N | \right ], \; \; \text{such that} \; \; \tilde{a}_{ij}^{self} \neq 0, \; \; \; x_{i} = x_{j} + \epsilon_{ij} &\implies \; \; \; \sum_{j} \tilde{a}_{ij}^{self} x_{i} = \sum_{j} \tilde{a}_{ij}^{self} x_{j} + \sum_{j} \tilde{a}_{ij}^{self} \epsilon_{ij}, \\
&\implies \; \; \;  x_{i} = h^{self}(x_{i}) + \sum_{j} \tilde{a}_{ij}^{self} \epsilon_{ij}. \\
\end{split}
\end{equation*}

\noindent Given $j_{max} = \text{arg} \max_{j}(\epsilon_{ij})$ and $j_{min} = \text{arg} \min_{j}(\epsilon_{ij})$, we obtain    $h^{self}(x_{i}) + \epsilon_{ij_{min}} \leqslant x_{i} \leqslant h^{self}(x_{i}) + \epsilon_{ij_{max}}.$

\vspace{2mm}
\noindent Since $\epsilon_{ij_{min}} \approx \boldsymbol{0}$ and  $\epsilon_{ij_{max}} \approx \boldsymbol{0}$, then we conclude that $x_{i} \approx h^{self}(x_{i})$.
\end{proof}








\begin{equation*}
\begin{split}
\Lambda'_{FD}(\mathcal{Q}_{2}, x_{i}) - \Lambda'_{FD}(\mathcal{Q}_{1}, x_{i}) &=  \inp*{\frac{\partial \sum_{j} \tilde{a}_{ij}^{self} \; \norm{f_{2}(x_{i})-f_{2}(x_{j})}_{2}^{2}}{\partial x_{i}}}{\frac{\partial \sum_{j} a_{ij}^{sup} \; \norm{f_{2}(x_{i})-f_{2}(x_{j})}_{2}^{2}}{\partial x_{i}}} \\ 
& -  \inp*{\frac{\partial \sum_{j} \tilde{a}_{ij}^{self} \; \norm{f_{1}(x_{i})-f_{1}(x_{j})}_{2}^{2}}{\partial x_{i}}}{\frac{\partial \sum_{j} a_{ij}^{sup} \; \norm{f_{1}(x_{i})-f_{1}(x_{j})}_{2}^{2}}{\partial x_{i}}}, \\
& = \sum_{j,j'} \tilde{a}_{ij}^{self} a_{ij'}^{sup} (f_{2}(x_{i})-f_{2}(x_{j}))^{T}(f_{2}(x_{i})-f_{2}(x_{j'})) \\
& - \sum_{j,j'} \tilde{a}_{ij}^{self} a_{ij'}^{sup} (f_{1}(x_{i})-f_{1}(x_{j}))^{T}(f_{1}(x_{i})-f_{1}(x_{j'})), \\
& = \frac{1}{2}\sum_{j} (\tilde{a}_{ij}^{self} + a_{ij}^{sup}) \norm{f_{2}(x_{i}) - f_{2}(x_{j})}_{2}^{2} -\frac{1}{2}\sum_{j, j'} \tilde{a}_{ij}^{self} a_{ij}^{sup} \norm{f_{2}(x_{j}) - f_{2}(x_{j'})}_{2}^{2} \\
& - \frac{1}{2}\sum_{j} (\tilde{a}_{ij}^{self} + a_{ij}^{sup}) \norm{f_{1}(x_{i}) - f_{1}(x_{j})}_{2}^{2}+\frac{1}{2}\sum_{j, j'} \tilde{a}_{ij}^{self} a_{ij}^{sup} \norm{f_{1}(x_{j}) - f_{1}(x_{j'})}_{2}^{2}, \\
& = \frac{1}{2}\sum_{j} (\tilde{a}_{ij}^{self} + a_{ij}^{sup}) \norm{f_{1}(h^{self}(x_{i})) - f_{1}(h^{self}(x_{j}))}_{2}^{2}  \\
& - \frac{1}{2}\sum_{j, j'} \tilde{a}_{ij}^{self} a_{ij}^{sup} \norm{f_{1}(h^{self}(x_{j})) - f_{1}(h^{self}(x_{j'}))}_{2}^{2} \\
& - \frac{1}{2}\sum_{j} (\tilde{a}_{ij}^{self} + a_{ij}^{sup}) \norm{f_{1}(x_{i}) - f_{1}(x_{j})}_{2}^{2}+\frac{1}{2}\sum_{j, j'} \tilde{a}_{ij}^{self} a_{ij}^{sup} \norm{f_{1}(x_{j}) - f_{1}(x_{j'})}_{2}^{2}, \\
& = \mathcal{D}_{1}^{self}(f_{1}(h^{self}(x_{i}))) - \mathcal{D}_{1}^{self}(f_{1}(x_{i})) + \mathcal{D}_{1}^{sup}(f_{1}(h^{self}(x_{i}))) - \mathcal{D}_{1}^{sup}(f_{1}(x_{i})) \\
& + \mathcal{D}_{2}(f_{1}(x_{i})) - \mathcal{D}_{2}(f_{1}(h^{self}(x_{i}))).  \\
\end{split}
\end{equation*}


\vspace{5mm}
\noindent We know that $f_{1}$ is a Lipschitz function, thus there exists $\tau_{1}$ such that:
$$\norm{f_{1}(x_{i}) - f_{1}(h_{self}(x_{i}))}_{2} \leqslant \tau_{1} \norm{x_{i} - h^{self}(x_{i})}_{2}.$$

\vspace{5mm}
\noindent Consequently, if $\norm{x_{i} - h^{self}(x_{i})}_{2} \rightarrow 0$ then $\norm{f_{1}(x_{i}) - f_{1}(h^{self}(x_{i}))}_{2} \rightarrow 0$. 

\vspace{5mm}
\noindent Based on Lemma \ref{lemma_3}, we have $x_{i} \approx h^{self}(x_{i})$. Additionally, $\mathcal{D}_{1}^{self}$ and $\mathcal{D}_{2}$ are differentiable functions, we can apply a first-order Taylor expansion with Peano's form of remainder at $f_{1}(x_{i})$:

$$\mathcal{D}_{1}^{self}(f_{1}(h^{self}(x_{i}))) = \mathcal{D}_{1}^{self}(f_{1}(x_{i})) + \big(\triangledown_{f_{1}(x_{i})}\mathcal{D}_{1}^{self}(f_{1}(x_{i}))\big)^{T}\big(f_{1}(h^{self}(x_{i}))-f_{1}(x_{i})\big) + o\big(\norm{f_{1}(h^{self}(x_{i}))-f_{1}(x_{i})}_{2}\big),$$

$$\mathcal{D}_{1}^{sup}(f_{1}(h^{self}(x_{i}))) = \mathcal{D}_{1}^{sup}(f_{1}(x_{i})) + \big(\triangledown_{f_{1}(x_{i})}\mathcal{D}_{1}^{sup}(f_{1}(x_{i}))\big)^{T}\big(f_{1}(h^{self}(x_{i}))-f_{1}(x_{i})\big) + o\big(\norm{f_{1}(h^{self}(x_{i}))-f_{1}(x_{i})}_{2}\big),$$

$$\mathcal{D}_{2}(f_{1}(h^{self}(x_{i}))) = \mathcal{D}_{2}(f_{1}(x_{i})) + \big(\triangledown_{f_{1}(x_{i})}\mathcal{D}_{2}(f_{1}(x_{i}))\big)^{T}\big(f_{1}(h^{self}(x_{i}))-f_{1}(x_{i})\big) + o\big(\norm{f_{1}(h^{self}(x_{i}))-f_{1}(x_{i})}_{2}\big).$$

\noindent Hence

\begin{equation*}
\begin{split}
\Lambda'_{FD}(\mathcal{Q}_{2}, x_{i}) - \Lambda'_{FD}(\mathcal{Q}_{1}, x_{i}) &= \big(\triangledown_{f_{1}(x_{i})}\mathcal{D}_{1}^{self}(f_{1}(x_{i}))\big)^{T}\big(f_{1}(h^{self}(x_{i}))-f_{1}(x_{i})\big) + \big(\triangledown_{f_{1}(x_{i})}\mathcal{D}_{1}^{sup}(f_{1}(x_{i}))\big)^{T}\big(f_{1}(h^{self}(x_{i}))-f_{1}(x_{i})\big) \\
&- \big(\triangledown_{f_{1}(x_{i})}\mathcal{D}_{2}(f_{1}(x_{i}))\big)^{T}\big(f_{1}(h^{self}(x_{i}))-f_{1}(x_{i})\big) + o\big(\norm{f_{1}(h^{self}(x_{i}))-f_{1}(x_{i})}_{2}\big). \\
\end{split}
\end{equation*}

\noindent And since 

$$\triangledown_{f_{1}(x_{i})}\mathcal{D}_{1}^{self}(f_{1}(x_{i})) = f_{1}(x_{i}) - h^{self}(f_{1}(x_{i})),$$
$$\triangledown_{f_{1}(x_{i})}\mathcal{D}_{1}^{sup}(f_{1}(x_{i})) = f_{1}(x_{i}) - h^{sup}(f_{1}(x_{i})),$$
$$\triangledown_{f_{1}(x_{i})}\mathcal{D}_{2}(f_{1}(x_{i})) = 0,$$

\noindent Then

\begin{equation*}
\begin{split}
\Lambda'_{FD}(\mathcal{Q}_{2}, x_{i}) - \Lambda'_{FD}(\mathcal{Q}_{1}, x_{i}) &= \big(f_{1}(x_{i}) - h^{self}(f_{1}(x_{i})) \big)^{T}\big(f_{1}(h^{self}(x_{i}))-f_{1}(x_{i})\big) \\ & + \big(f_{1}(x_{i}) - h^{sup}(f_{1}(x_{i}))\big)^{T}\big(f_{1}(h^{self}(x_{i}))-f_{1}(x_{i})\big) + o\big(\norm{f_{1}(h^{self}(x_{i}))-f_{1}(x_{i})}_{2}\big). \\
\end{split}
\end{equation*}

\noindent Based on Lemma \ref{lemma_2}, we can write:

\begin{equation*}
\begin{split}
\Lambda'_{FD}(\mathcal{Q}_{2}, x_{i}) - \Lambda'_{FD}(\mathcal{Q}_{1}, x_{i}) &= \big(f_{1}(x_{i}) - h^{self}(f(x_{i})) \big)^{T}\big(h^{self}(f_{1}(x_{i}))-f_{1}(x_{i})\big) \\ & + \big(f_{1}(x_{i}) - h^{sup}(f_{1}(x_{i}))\big)^{T}\big(h^{self}(f_{1}(x_{i}))-f_{1}(x_{i})\big) + o\big(\norm{h^{self}(f_{1}(x_{i}))-f_{1}(x_{i})}_{2}\big), \\
&= - \norm{f_{1}(x_{i}) - h^{self}(f_{1}(x_{i}))}_{2}^{2} - \big(f_{1}(x_{i}) - h^{sup}(f_{1}(x_{i}))\big)^{T}\big(f_{1}(x_{i}) - h^{self}(f_{1}(x_{i}))\big) \\
&+ o\big(\norm{f_{1}(x_{i}) - h^{self}(f_{1}(x_{i}))}_{2}\big).\\
\end{split}
\end{equation*}

\noindent By applying the law of Cosines to compute $\big(f_{1}(x_{i}) - h^{sup}(f_{1}(x_{i}))\big)^{T}\big(f_{1}(x_{i}) - h^{self}(f_{1}(x_{i}))\big)$, we obtain:

\begin{equation*}
\begin{split}
\Lambda'_{FD}(\mathcal{Q}_{2}, x_{i}) - \Lambda'_{FD}(\mathcal{Q}_{1}, x_{i}) &= - \norm{f_{1}(x_{i}) - h^{self}(f_{1}(x_{i}))}_{2}^{2} - \frac{1}{2} \norm{f_{1}(x_{i}) - h^{sup}(f_{1}(x_{i}))}_{2}^{2} - \frac{1}{2} \norm{f_{1}(x_{i}) - h^{self}(f_{1}(x_{i}))}_{2}^{2} \\
&+ \frac{1}{2} \norm{h^{self}(f_{1}(x_{i})) - h^{sup}(f_{1}(x_{i}))}_{2}^{2} + o\big(\norm{f_{1}(x_{i}) - h^{self}(f_{1}(x_{i}))}_{2}\big),\\
&=  - \frac{3}{2} \norm{f_{1}(x_{i}) - h^{self}(f_{1}(x_{i}))}_{2}^{2} - \frac{1}{2} \norm{f_{1}(x_{i}) - h^{sup}(f_{1}(x_{i}))}_{2}^{2} \\
&+ \frac{1}{2} \norm{h^{self}(f_{1}(x_{i})) - h^{sup}(f_{1}(x_{i}))}_{2}^{2} + o\big(\norm{f_{1}(x_{i}) - h^{self}(f_{1}(x_{i}))}_{2}\big).  
\end{split}
\end{equation*}

\noindent We have 

\begin{equation*}
\begin{split}
\mathcal{P}(f_{1}(x_{i})) \geqslant 0 &\implies \norm{h^{self}(f_{1}(x_{i})) - h^{sup}(f_{1}(x_{i}))}_{2} \leqslant \norm{f_{1}(x_{i}) - h^{sup}(f_{1}(x_{i}))}_{2}, \\ 
&\implies \frac{1}{2} \norm{h^{self}(f_{1}(x_{i})) - h^{sup}(f_{1}(x_{i}))}_{2}^{2} \leqslant \frac{1}{2} \norm{f_{1}(x_{i}) - h^{sup}(f_{1}(x_{i}))}_{2}^{2}.\\
\end{split}
\end{equation*}

\vspace{5mm}

$$\left \{ \begin{matrix}
- \frac{3}{2} \norm{f_{1}(x_{i}) - h^{self}(f_{1}(x_{i}))}_{2}^{2} \leqslant 0 \\ 
\frac{1}{2} \norm{h^{self}(f_{1}(x_{i})) - h^{sup}(f_{1}(x_{i}))}_{2}^{2} - \frac{1}{2} \norm{f_{1}(x_{i}) - h^{sup}(f_{1}(x_{i}))}_{2}^{2} \leqslant 0 \\
\end{matrix} \right. \implies \Lambda'_{FD}(\mathcal{Q}_{2}, x_{i}) \leqslant \Lambda'_{FD}(\mathcal{Q}_{1}, x_{i}).$$


\end{proof}

\section{Proof of Theorem 5} \label{Appendix_I}

\newtheorem*{T5}{Theorem~\ref{theorem_5}}
\begin{T5} 
Given two models $\mathcal{Q}_{1}$ and $\mathcal{Q}_{2}$, which optimize the same objective function as described by Equation (\ref{Q_61}). $\mathcal{Q}_{1}$ has a single graph convolutional layer characterized by the function $f_{1}(X)=ReLU(\Tilde{A}^{self}XW_{1})$, where $W_{1}$ represents the learning weights of this layer. $\mathcal{Q}_{2}$ has two graph convolutional layers characterized by the function $f_{2}(X)=ReLU(\Tilde{A}^{self} \; ReLU(\Tilde{A}^{self}XW_{1}) \; W_{2})$, where $W_{2}$ represents the learning weights of the second layer. We suppose that the Lipschitz constant $\tau_{1}^{*}$ of the second graph convolutional layer is less or equal to 1.

\begin{equation}  \label{Q_61}
L_{\mathcal{Q}_{1}} = L_{\mathcal{Q}_{2}} = L_{clus}(Z(\theta)) + \gamma L_{bce}(\hat{A}(Z(\theta)), \,A^{self}).
\end{equation}

\noindent Under Assumption \ref{assumption_1} and Assumption \ref{assumption_2}, we have:
$$\text{If} \;\;\; \mathcal{P}(f_{1}(x_{i})) \geqslant 0 \;\;\; \text{then} \;\;\; \Lambda'_{FD}(\mathcal{Q}_{2}, x_{i}) \leqslant \Lambda'_{FD}(\mathcal{Q}_{1}, x_{i}).$$
\end{T5}

\begin{proof}
Similar to Theorem \ref{theorem_4}, let $h$ be an aggregation function such that $h^{sup}(x_{i}) = \sum_{j} \tilde{a}_{ij}^{sup} x_{j}$, and $h^{self}(x_{i}) = \sum_{j} \tilde{a}_{ij}^{self} x_{j}$. Let the functions $\mathcal{D}_{1}$ and $\mathcal{D}_{2}$ be distance metrics such that $\mathcal{D}_{1}^{sup}(x_{i})= \frac{1}{2}\sum_{j}a_{ij}^{sup} \norm{x_{i}-x_{j}}_{2}^{2}$, $\mathcal{D}_{1}^{self}(x_{i})= \frac{1}{2}\sum_{j} \tilde{a}_{ij}^{self} \norm{x_{i}-x_{j}}_{2}^{2}$, and $\mathcal{D}_{2}(x_{i})= \frac{1}{2}\sum_{j,j'} \tilde{a}_{ij}^{self} a_{ij}^{sup} \norm{x_{j}-x_{j'}}_{2}^{2}$.

\begin{equation*}
\begin{split}
\Lambda'_{FD}(\mathcal{Q}_{2}, x_{i}) - \Lambda'_{FD}(\mathcal{Q}_{1}, x_{i}) &=  \inp*{\frac{\partial \sum_{j} \tilde{a}_{ij}^{self} \; \norm{f_{2}(x_{i})-f_{2}(x_{j})}_{2}^{2}}{\partial x_{i}}}{\frac{\partial \sum_{j}  a_{ij}^{sup} \; \norm{f_{2}(x_{i})-f_{2}(x_{j})}_{2}^{2}}{\partial x_{i}}} \\ 
& -  \inp*{\frac{\partial \sum_{j}  \tilde{a}_{ij}^{self} \; \norm{f_{1}(x_{i})-f_{1}(x_{j})}_{2}^{2}}{\partial x_{i}}}{\frac{\partial \sum_{j}  a_{ij}^{sup} \; \norm{f_{1}(x_{i})-f_{1}(x_{j})}_{2}^{2}}{\partial x_{i}}}, \\
& = \sum_{j,j'} \tilde{a}_{ij}^{self}a_{ij'}^{sup}(f_{2}(x_{i})-f_{2}(x_{j}))^{T}(f_{2}(x_{i})-f_{2}(x_{j'})) - \sum_{j,j'} \tilde{a}_{ij}^{self} a_{ij'}^{sup}(f_{1}(x_{i})-f_{1}(x_{j}))^{T}(f_{1}(x_{i})-f_{1}(x_{j'})), \\
& = \frac{1}{2}\sum_{j} (\tilde{a}_{ij}^{self} + a_{ij}^{sup}) \norm{f_{2}(x_{i}) - f_{2}(x_{j})}_{2}^{2} -\frac{1}{2}\sum_{j, j'} \tilde{a}_{ij}^{self}a_{ij}^{sup} \norm{f_{2}(x_{j}) - f_{2}(x_{j'})}_{2}^{2} \\
& - \frac{1}{2}\sum_{j} (\tilde{a}_{ij}^{self} + a_{ij}^{sup}) \norm{f_{1}(x_{i}) - f_{1}(x_{j})}_{2}^{2}+\frac{1}{2}\sum_{j, j'} \tilde{a}_{ij}^{self} a_{ij}^{sup} \norm{f_{1}(x_{j}) - f_{1}(x_{j'})}_{2}^{2}, \\
&= \mathcal{D}_{1}^{self}(f_{2}(x_{i})) - \mathcal{D}_{1}^{self}(f_{1}(x_{i})) + \mathcal{D}_{1}^{sup}(f_{2}(x_{i})) - \mathcal{D}_{1}^{sup}(f_{1}(x_{i})) + \mathcal{D}_{2}(f_{1}(x_{i}))  - \mathcal{D}_{2}(f_{2}(x_{i})).  \\
\end{split}
\end{equation*}

\vspace{5mm}
\noindent We add the following null expression to the equation of $\Lambda'_{FD}(\mathcal{Q}_{2}, x_{i}) - \Lambda'_{FD}(\mathcal{Q}_{1}, x_{i})$
$$\mathcal{D}_{1}^{self}(h^{self}(f_{1}(x_{i}))) - \mathcal{D}_{1}^{self}(h^{self}(f_{1}(x_{i}))) + \mathcal{D}_{1}^{sup}(h^{self}(f_{1}(x_{i}))) - \mathcal{D}_{1}^{sup}(h^{self}(f_{1}(x_{i}))) + \mathcal{D}_{2}(h^{self}(f_{1}(x_{i}))) - \mathcal{D}_{2}(h^{self}(f_{1}(x_{i}))) = 0.$$

\vspace{5mm}
\noindent We obtain

\begin{equation*}
\begin{split}
\Lambda'_{FD}(\mathcal{Q}_{2}, x_{i}) - \Lambda'_{FD}(\mathcal{Q}_{1}, x_{i}) &= \mathcal{D}_{1}^{self}(f_{2}(x_{i})) - \mathcal{D}_{1}^{self}(h^{self}(f_{1}(x_{i}))) + \mathcal{D}_{1}^{self}(h^{self}(f_{1}(x_{i}))) - \mathcal{D}_{1}^{self}(f_{1}(x_{i})) \\ 
& + \mathcal{D}_{1}^{sup}(f_{2}(x_{i})) - {D}_{1}^{sup}(h^{self}(f_{1}(x_{i}))) + \mathcal{D}_{1}^{sup}(h^{self}(f_{1}(x_{i}))) - \mathcal{D}_{1}^{sup}(f_{1}(x_{i})) \\
& - \big(\mathcal{D}_{2}(f_{2}(x_{i}))-\mathcal{D}_{2}(h^{self}(f_{1}(x_{i})))\big) - \big(\mathcal{D}_{2}(h^{self}(f_{1}(x_{i}))) - \mathcal{D}_{2}(f_{1}(x_{i}))\big).  \\
\end{split}
\end{equation*}

\vspace{5mm}
\noindent Let $\mathcal{J}$ be
\begin{equation*}
\begin{split}
\mathcal{J} = \mathcal{D}_{1}^{self}(h^{self}(f_{1}(x_{i}))) - \mathcal{D}_{1}^{self}(f_{1}(x_{i})) + \mathcal{D}_{1}^{sup}(h^{self}(f_{1}(x_{i}))) - \mathcal{D}_{1}^{sup}(f_{1}(x_{i})) + \mathcal{D}_{2}(f_{1}(x_{i})) - \mathcal{D}_{2}(h^{self}(f_{1}(x_{i}))).  \\
\end{split}
\end{equation*}

\vspace{5mm}
\noindent Hence

\begin{equation*}
\begin{split}
\Lambda'_{FD}(\mathcal{Q}_{2}, x_{i}) - \Lambda'_{FD}(\mathcal{Q}_{1}, x_{i}) &= \mathcal{D}_{1}^{self}(f_{2}(x_{i})) - \mathcal{D}_{1}^{self}(h^{self}(f_{1}(x_{i}))) \\
& + \mathcal{D}_{1}^{sup}(f_{2}(x_{i})) - {D}_{1}^{sup}(h^{self}(f_{1}(x_{i}))) \\
& - \big(\mathcal{D}_{2}(f_{2}(x_{i}))-\mathcal{D}_{2}(h^{self}(f_{1}(x_{i})))\big) + \mathcal{J}.  \\
\end{split}
\end{equation*}

\vspace{5mm}
\noindent Based on Lemma \ref{lemma_2}
\begin{equation*}
\begin{split}
\mathcal{J} = \mathcal{D}_{1}^{self}(f_{1}(h^{self}(x_{i}))) - \mathcal{D}_{1}^{self}(f_{1}(x_{i})) + \mathcal{D}_{1}^{sup}(f_{1}(h^{self}(x_{i}))) - \mathcal{D}_{1}^{sup}(f_{1}(x_{i})) + \mathcal{D}_{2}(f_{1}(x_{i})) - \mathcal{D}_{2}(f_{1}(h^{self}(x_{i}))).  \\
\end{split}
\end{equation*}

\vspace{5mm}
\noindent Then, based on Theorem \ref{theorem_4}, we can see that $\mathcal{J} \leqslant 0$.  

$$\mathcal{D}_{1}^{self}(f_{2}(x_{i})) - \mathcal{D}_{1}^{self}(h^{self}(f_{1}(x_{i}))) = \mathcal{D}_{1}^{self}(ReLU(W_{2}^{T} \: h_{self}(f_{1}(x_{i})))) - \mathcal{D}_{1}^{self}(h^{self}(f_{1}(x_{i}))),$$
$$\mathcal{D}_{1}^{sup}(f_{2}(x_{i})) - \mathcal{D}_{1}^{sup}(h^{self}(f_{1}(x_{i}))) = \mathcal{D}_{1}^{sup}(ReLU(W_{2}^{T} \: h_{self}(f_{1}(x_{i})))) - \mathcal{D}_{1}^{sup}(h^{self}(f_{1}(x_{i}))).$$

\vspace{5mm}
\noindent The second graph convolutional layer is a Lipschitz function and its Lipschitz constant $\tau_{1}$ is less or equal to 1. Hence
$$\mathcal{D}_{1}^{self}(ReLU(W_{2}^{T} \: h^{self}(f_{1}(x_{i})))) \leqslant \mathcal{D}_{1}^{self}(h^{self}(f_{1}(x_{i}))) \; \implies  \;  \mathcal{D}_{1}^{self}(f_{2}(x_{i})) - \mathcal{D}_{1}^{self}(h^{self}(f_{1}(x_{i}))) \leqslant 0,$$
$$\mathcal{D}_{1}^{sup}(ReLU(W_{2}^{T} \: h^{self}(f_{1}(x_{i})))) \leqslant \mathcal{D}_{1}^{sup}(h^{self}(f_{1}(x_{i}))) \; \implies \; \mathcal{D}_{1}^{sup}(f_{2}(x_{i})) - \mathcal{D}_{1}^{sup}(h^{self}(f_{1}(x_{i}))) \leqslant 0.$$

\vspace{5mm}
\noindent We know that $f_{2}$ and $h^{self}(f_{1})$ are Lipschitz functions. Consequently, for all $j$ and $j'$ indices, if $\norm{x_{j} - x_{j'}}_{2} \rightarrow 0$ then \\ $\norm{f_{2}(x_{j}) - f_{2}(x_{j'})}_{2} \rightarrow 0$ and $\norm{h^{self}(f_{1}(x_{j})) - h^{self}(f_{2}(x_{j'}))}_{2} \rightarrow 0$. 

\vspace{5mm}
\noindent Based on Assumption \ref{assumption_1} 
\begin{equation*}
\begin{split}
&\forall j \in \left [|1, N | \right ], \; \; \text{such that} \; \; \tilde{a}_{ij}^{self} \neq 0, \; \; \; x_{i} \approx x_{j}, \\
&\implies j, j' \in \left [|1, N | \right ], \; \; \text{such that} \; \; \tilde{a}_{ij}^{self} \neq 0 \; \; \text{and} \; \; \tilde{a}_{ij'}^{self} \neq 0, \; \; \; x_{j} \approx x_{j'}, \\
&\implies j, j' \in \left [|1, N | \right ], \; \; \text{such that} \; \; \tilde{a}_{ij}^{self} \neq 0 \; \; \text{and} \; \; \tilde{a}_{ij'}^{self} \neq 0, \; \; \; \norm{x_{j} - x_{j'}}_{2} \approx 0, \\
&\implies j, j' \in \left [|1, N | \right ], \; \; \text{such that} \; \; \tilde{a}_{ij}^{self} \neq 0 \; \; \text{and} \; \; \tilde{a}_{ij'}^{self} \neq 0, \; \; \; \norm{f_{2}(x_{j}) - f_{2}(x_{j'})}_{2} \approx 0 \; \; \text{and} \; \; \norm{h^{self}(f_{1}(x_{j})) - h^{self}(f_{2}(x_{j'}))}_{2} \approx 0 , \\
& \implies j, j' \in \left [|1, N | \right ], \; \; \text{such that} \; \; \tilde{a}_{ij}^{self} \neq 0 \; \; \text{and} \; \; \tilde{a}_{ij'}^{self} \neq 0, \; \; \;  \mathcal{D}_{2}(f_{2}(x_{i})) \approx 0 \; \; \text{and} \; \; \mathcal{D}_{2}(h^{self}(f_{1}(x_{i})))\approx 0.
\end{split}
\end{equation*}

\vspace{5mm}
\noindent So, globally, we have

$$ \left\{\begin{matrix}
\mathcal{J} \leqslant 0 \\ 
\mathcal{D}_{1}^{sup}(f_{2}(x_{i})) - \mathcal{D}_{1}^{sup}(h^{self}(f_{1}(x_{i}))) \leqslant 0\\ 
\mathcal{D}_{1}^{self}(f_{2}(x_{i})) - \mathcal{D}_{1}^{self}(h^{self}(f_{1}(x_{i}))) \leqslant 0\\
\end{matrix}\right. \implies \Lambda'_{FD}(\mathcal{Q}_{2}, x_{i}) - \Lambda'_{FD}(\mathcal{Q}_{1}, x_{i}) \leqslant 0 \implies \Lambda'_{FD}(\mathcal{Q}_{2}, x_{i}) \leqslant \Lambda'_{FD}(\mathcal{Q}_{1}, x_{i}).$$

\end{proof}

\end{document}